\documentclass[10pt]{article}

\DeclareMathSizes{9}{8}{7}{5}
\newlength{\widebarargwidth}
\newlength{\widebarargheight}
\newlength{\widebarargdepth}

  \makeatletter
\long\def\@makecaption#1#2{
        \vskip 0.8ex
        \setbox\@tempboxa\hbox{\small {\bf #1:} #2}
        \parindent 1.5em  
        \dimen0=\hsize
        \advance\dimen0 by -3em
        \ifdim \wd\@tempboxa >\dimen0
                \hbox to \hsize{
                        \parindent 0em
                        \hfil 
                        \parbox{\dimen0}{\def\baselinestretch{0.96}\small
                                {\bf #1.} #2
                                } 
                        \hfil}
        \else \hbox to \hsize{\hfil \box\@tempboxa \hfil}
        \fi
        }
\makeatother

\date{}
\usepackage[utf8]{inputenc} 
\usepackage[style = alphabetic,citestyle=alphabetic,maxbibnames=99,backend=biber,sorting=nyt,natbib=true,backref=true]{biblatex}
\DefineBibliographyStrings{english}{%
  backrefpage = {Cited on page},
  backrefpages = {Cited on pages},
}

\usepackage[T1]{fontenc}    
\usepackage{url}            
\usepackage{booktabs}       
\usepackage{amsfonts}       
\usepackage{nicefrac}       
\usepackage{microtype}      

\usepackage{algorithmic}
\usepackage[ruled,vlined]{algorithm2e}

\usepackage[utf8]{inputenc} 
\usepackage[T1]{fontenc}    
\usepackage[usenames,dvipsnames]{xcolor}
\usepackage{multirow}
\usepackage{mathtools}
\usepackage{amsfonts}
\usepackage{amsthm}
\usepackage{amsmath}
\usepackage{amssymb}
\usepackage{scalerel}
\usepackage{nccmath}
\usepackage{lipsum}
\usepackage{verbatim}
\usepackage{epstopdf}
\usepackage{tikz}
\usepackage{xspace}
\usepackage{enumitem}
\usepackage{wrapfig}
\usepackage{subcaption}

\usepackage[font=small]{caption}
\setlist[itemize]{align=parleft,left=5pt..1.5em}

\usepackage[colorlinks,linkcolor = RawSienna, urlcolor  = RedViolet, citecolor = RoyalBlue, anchorcolor = ForestGreen,bookmarks=False]{hyperref}

\newcommand{\alg}{\textsc{Tart}\xspace}

\newcommand{\real}{\ensuremath{\mathbb{R}}}
\newcommand{\En}{\ensuremath{\mathbb{E}}}

\newcommand{\defn}{:\,=}
\newtheorem{theorem}{Theorem}

\newtheorem{assumption}{Assumption}

\newcommand{\1}{\ensuremath{{\sf (i)}}}
\newcommand{\2}{\ensuremath{{\sf (ii)}}}

\newcommand\inner[2]{\langle #1, #2 \rangle}

\newcommand{\red}[1]{\textcolor{red}{#1}}
\renewcommand{\red}{}

\newcommand{\gap}{\Delta}

\newcommand{\gaprep}{\gap_{{\sf rep}}}
\newcommand{\gapres}{\gap_{{\sf reas}}}
\newcommand{\gapperf}{\gap_{{\sf perf}}}

\newcommand{\normal}{\mathcal{N}}
\newcommand{\sigmoid}{\sigma}
\newcommand{\param}{w}
\newcommand{\x}{x}
\newcommand{\y}{y}

\newcommand{\dx}{d}
\newcommand{\weight}{\alpha}
\newcommand{\loss}{\ell}
\newcommand{\klr}{k}

\newcommand{\T}{\mathcal{T}}
\newcommand{\Tres}{T_S}
\newcommand{\nsyn}{n_{{\sf syn}}}
\newcommand{\Psyn}{P_{{\sf syn}}}
\newcommand{\Pnl}{P_{{\sf NL}}}
\newcommand{\err}{{\sf err}}

\newcommand{\Tr}{T}
\newcommand{\paramT}{\theta}
\newcommand{\lossce}{\loss_{{\sf CE}}}
\newcommand{\seq}{s}
\newcommand{\samp}{S}
\newcommand{\joint}{\gamma}

\newcommand{\gptsmall}{\textsc{GPT-Neo (125M)}\xspace}
\newcommand{\gptmedium}{\textsc{GPT-Neo (1.3B)}\xspace}
\newcommand{\gptlarge}{\textsc{GPT-Neo (2.7B)}\xspace}

\newcommand{\gptj}{\textsc{GPT-J (6B)}\xspace}
\newcommand{\pythiasmall}{\textsc{Pythia (160M)}\xspace}
\newcommand{\pythiamedium}{\textsc{Pythia (1.4B)}\xspace}
\newcommand{\pythialarge}{\textsc{Pythia (2.8B)}\xspace}

\newcommand{\bloomsmall}{\textsc{Bloom (560M)}\xspace}
\newcommand{\bloommedium}{\textsc{Bloom (1.7B)}\xspace}
\newcommand{\bloomlarge}{\textsc{Bloom (3B)}\xspace}
\newcommand{\whisper}{\textsc{Whisper-large}\xspace}
\newcommand{\vit}{\textsc{ViT-large-patch16-224-in21k}\xspace}

\newcommand{\bloomhuge}{\textsc{Bloom (176B)}\xspace}
\newcommand{\opthuge}{\textsc{OPT (175B)}\xspace}

\newcommand{\gpt}{\textsc{GPT-3 (175B)}\xspace}
\newcommand{\neo}{\textsc{GPT-Neo}\xspace}
\newcommand{\pythia}{\textsc{Pythia}\xspace}
\newcommand{\bloom}{\textsc{Bloom}\xspace}

\newcommand{\acclr}{\text{Acc}_{{\sf LR}}}
\newcommand{\accft}{\text{Acc}_{{\sf FT}}}
\newcommand{\accicl}{\text{Acc}_{{\sf ICL}}}

\newcommand{\kt}{k}
\newcommand{\dt}{d}
\newcommand{\PT}{\Theta}
\newcommand{\Pd}{P}
\usepackage[verbose=true,letterpaper]{geometry}

\AtBeginDocument{
  \newgeometry{
    textheight=9in,
    textwidth=5.96in,
    top=1in,
    bottom=1.5in,
    headheight=12pt,
    headsep=25pt,
    footskip=30pt,
    footnotesep=13pt,
  }
}
\title{\textbf{\alg: A plug-and-play Transformer module for\\ 
task-agnostic reasoning}}

\author{%
  Kush Bhatia$^{\dagger}$\thanks{Equal Contribution}
  \and 
  Avanika Narayan$^{\dagger *}$
  \and 
  Christopher De Sa$^{\ddagger}$
  \and Christopher R\'e$^{\dagger}$
}

\begin{document}
\maketitle
\vspace{-10mm}
\begin{center}
\text{  $^\dagger$ Department of Computer Science, Stanford University}\\
  \text{$^\ddagger$ Department of Computer Science, Cornell University}\\
  \vspace{2mm}
   \texttt{\{kushb, avanika, chrismre\}@cs.stanford.edu, cdesa@cs.cornell.edu} 
\end{center}
\vspace{5mm}
\begin{abstract}
Large language models (LLMs) exhibit in-context learning abilities which enable the same model to perform several tasks without any task-specific training. In contrast, traditional adaptation approaches, such as fine-tuning, modify the underlying models for \emph{each} specific task. In-context learning, however, consistently underperforms task-specific tuning approaches \emph{even} when presented with the same examples. While most existing approaches (e.g., prompt engineering) focus on the LLM's learned representations to patch this performance gap, our analysis actually reveal that LLM representations contain sufficient information to make good predictions. As such, we focus on the LLM's reasoning abilities and demonstrate that this performance gap exists due to their inability to perform simple probabilistic reasoning tasks. 
This raises an intriguing question: Are LLMs actually capable of learning how to reason in a task-agnostic manner? We answer this in the affirmative and propose \alg which generically improves an LLM's reasoning abilities using a synthetically trained Transformer-based reasoning module. \alg trains this reasoning module in a task-agnostic manner using \emph{only synthetic} logistic regression tasks and composes it with an arbitrary real-world pre-trained model without any additional training. With a single inference module, \alg improves performance across different model families (\neo, \pythia, \bloom), model sizes (100M - 6B), tasks (14 NLP binary classification tasks), and even across different modalities (audio and vision). Additionally, on the RAFT Benchmark, \alg improves \gptsmall's performance such that it outperforms \bloomhuge, and is within $4\%$ of \gpt.\footnote{Our code and model is available at \url{https://github.com/HazyResearch/TART}}
\end{abstract}
\section{Introduction}\label{sec:intro}
Large language models (LLMs) 
show in-context learning capabilities which enable them to perform a task given only a few examples, without updating the model parameters~\citep{brown2020language, bommasani2021opportunities}. This task-agnostic capability allows for a single model to be applied to a wide range of tasks~\citep{agrawal2022large, wei2022emergent, narayandata}. In contrast, traditional task adaptation approaches, such as fine-tuning, update the model parameters for each specific task.

Despite being task-agnostic, in-context learning is seldom the practitioner's method of choice since it consistently underperforms task-specific adaptation approaches~\citep{lester2021power, brown2020language}. Most existing works attribute this performance gap to the limited context window of LLMs which can only accommodate a few task examples~\citep{kocon2023chatgpt, huyenchip:2023, liu2022few}. However, we show that this gap between in-context learning and fine-tuning approaches exists \emph{even} when presented with the same task examples. 


This observation raises the question whether this performance gap is a generic limitation of task-agnostic methods for adaptation or is it specific to in-context learning? Specifically, can we design adaptation approaches which satisfy the following desiderata:
\begin{itemize}
    \item \emph{Task-agnostic}: The same model generalizes across several different tasks.\vspace{-1mm}
    \item \emph{Quality}: Achieves accuracy competitive with task-specific methods across these different tasks.\vspace{-1mm}
    \item \emph{Data-scalable}: Learning ability improves with increasing number of task examples.
\end{itemize}

We first investigate why this quality gap exists. We decompose an LLM's in-context learning capability into two abilities: learning good \emph{representations} for the task and performing probabilistic inference, or \emph{reasoning}, over these  representations~\citep{jaynes2003probability}. Is the gap because the representations do not contain sufficient information or because the LLMs are unable to reason over them? 
%
\begin{figure}
    \centering
    \includegraphics[width = \textwidth]{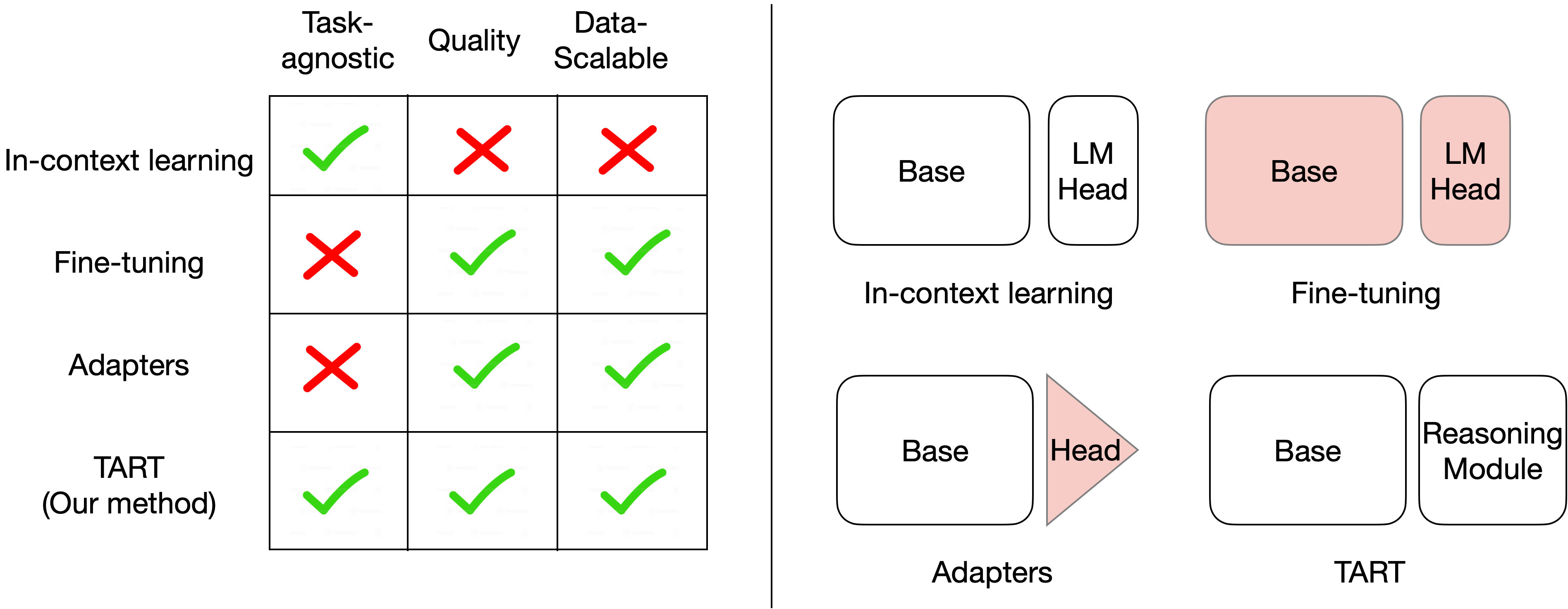}
    \caption{\textbf{Taxonomy of task adaptation strategies.} (Left) Comparison of different adaptation strategies across three desiderata: task-agnostic, quality, scalability. (Right) Parameter updates across adaptations strategies, colored regions represent parameter changes as a result of the adaptation strategy.}
    \label{fig:algs}
\end{figure}
We explore this hypothesis experimentally in Section~\ref{sec:strats} by measuring both the reasoning and the representation gaps across a variety of LLM families (\neo~\citep{gpt-neo}, \pythia~\citep{biderman2023pythia}, \bloom~\citep{scao2022bloom}) over a suite of binary classification tasks. We conclude that LLMs possess good representations, and the majority of the quality gap (up to $79$\%) can be attributed to their insufficient reasoning ability. We further find that fine-tuning improves the base model on both these axes, but primarily improve the task specific reasoning ability which accounts for $72$\% of the gained performance.

Rather surprisingly, most existing techniques for improving the performance gap, such as prompt engineering or active example selection, focus entirely on the LLM's learned representations. In contrast, our work explores the orthogonal direction of improving the LLM's reasoning abilities. 
%
As a first step, we fine-tune LLMs using synthetically generated probabilistic inference tasks to improve their reasoning capabilities. While this approach provides an improvement over the model's base in-context learning performance (up to $19$\%, see Figure~\ref{fig:ft-nl-app} in App.~\ref{app:ft-nl}), this approach requires one to fine-tune each LLM individually. Taking a step further, we consider the possibility of whether one can improve the reasoning capabilities in a manner that is agnostic to \emph{both} tasks and models.



We show that it is indeed possible to improve the reasoning capabilities in a completely agnostic manner. We propose \alg  which improves upon an LLM's reasoning abilities using a synthetically trained reasoning module (see Figure~\ref{fig:main}). \alg trains a Transformer-based reasoning module using only synthetically generated logistic regression tasks independent of the downstream task or the base LLM. This inference module can be composed, \emph{without any additional training}, with the embeddings of an LLM to improve upon its reasoning abilities. 
Notably, \alg satisfies the desired objectives:
\begin{itemize}
    \item \emph{Task-agnostic}: \alg's inference module is only trained once using synthetic data.\vspace{-1mm}
    \item \emph{Quality}: Outperforms base LLM on all tasks and closes gap to task specific fine-tuning methods.\vspace{-1mm}
    \item \emph{Data-scalable}: Can accommodate 10x more examples than in-context learning.
\end{itemize}
\alg is \emph{task, model, and domain} agnostic. Using a single inference module trained on synthetic data, we exhibit that \alg not only generalizes across three model families (\neo, \pythia, \bloom) over 14 NLP classification tasks, but even across different domains (vision and speech; see Figure~\ref{fig:modalities}).
%
%
In terms of quality, we show that \alg's performance is \red{18.4\%}  better than in-context learning, \red{3.4\%} better than task-specific adapters, and is within \red{3.1\%} of full task-specific fine-tuning across a suite of NLP tasks. On the RAFT Benchmark~\citep{alexraft}, \alg improves \gptsmall's performance such that it outperforms \bloomhuge, and is within $4\%$ of \gpt. 
\alg is data-scalable and overcomes the limited context length bottleneck of in-context learning. While each example spans multiple tokens in an LLM, often spanning hundreds of tokens, \alg's reasoning module encodes each example using only two tokens -- one for the context and the other for the label. 
This data-scalability can lead to improvements of up to \red{$6.8\%$} (see Figure~\ref{fig:context-window}).

From a theoretical standpoint, we show that the generalization abilities of \alg depends mainly on the distribution shift between the natural text embedding distribution produced by the LLM and the synthetic data distribution, measured in terms of the Wasserstein-1 metric (Theorem~\ref{thm:gen}).

To summarize, our main contributions are as follows:\vspace{-1mm}
\begin{itemize}
\item Study why in-context learning does not perform as well as task-specific fine-tuning despite having access to the same information, via a representation-reasoning decomposition.\vspace{-1mm}
\item Propose a new task-agnostic method, \alg, which bridges the performance gap to task-specific methods and is trained using only synthetic data.\vspace{-1mm}
\item Demonstrate that \alg works across different NLP tasks for a range of model families. The same inference module generalizes to vision and speech domains as well.
\end{itemize}

\begin{figure}[t!]
    \centering
    \includegraphics[width=1\textwidth]{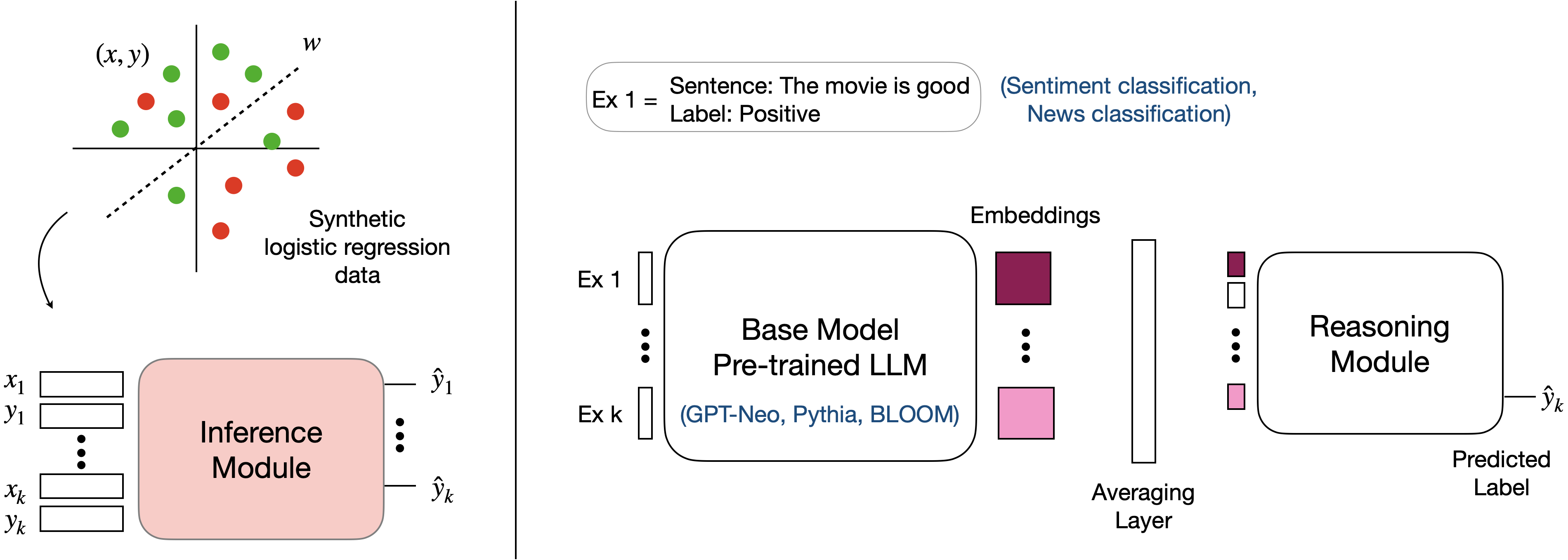}
    \caption{\textbf{\alg}. (Left) Inference module training procedure: The inference module is trained on sequences of synthetically generated logistic regression tasks. (Right) End-to-end framework: \alg composes a pre-trained LLM with the inference module. \alg uses the LLM to embed the input text. These embeddings, along with the train labels, are passed as a sequence to the inference module which generates a final prediction.}
    \label{fig:main}
\end{figure}
\section{Related work}
\label{sec:rel-work}





Prompt engineering focuses on improving the in-context task adaptation abilities of LLMs by modifying prompts. A line of work improves performance by carefully designing the natural language task specifications ~\cite{arora2022ask, wei2022chain} while others improve performance by optimizing the examples chosen for the prompt~\cite{diao2023active, liu2022makes}, encouraging the models to sequentially reason~\cite{kojima2022large, wei2022chain, zelikmanstar} and aggregating prompts~\citep{wang2022self, wang2022rationale}. Unfortunately, prompt-based task adaptation is noisy~\citep{lu2022fantastically}. 
Alternatively, prompt tuning improves the in-context abilities of models by training a small amounts of learnable vectors~\citep{li2021prefix,lester2021power, liu-etal-2022-p} for specific tasks. While these methods have been shown to improve in-context learning performance, they require task-specific fine-tuning and are not task-agnostic. 

Recent works seek to understand the in-context learning property of LLMs by presenting mechanistic interpretations of in-context learning~\citep{von2022transformers}, performing exploratory analysis of in-context learning behaviors~\citep{wei2023larger}, and explaining it as implicit Bayesian inference~\cite{xie2021explanation}. Existing literature demonstrates that LLMs can learn simple function classes in-context~\citep{garg2022can} and propose that LLMs are performing gradient descent when learning tasks in-context~\citep{von2022transformers}. 
Complementary to these, our work provides insights on the mechanisms of in-context learning and its deficiencies. 
 %
 Furthermore, task transfer strategies adapt LLMs to a pre-specified target task. Strategies range from  parameter efficient finetuning (PEFT)~\citep{houlsby2019parameter, zhang2023llamaadapter} to Low-Rank adaptation (LoRA)~\citep{hu2022lora} which introduces trainable rank decomposition matrices into each layer to combining linear probing and fine-tuning~\cite{kumar2022fine}. While they have good performance, these methods require training models on a task-by-task basis in contrast to \alg. 

\section{Task adaptation strategies: Taxonomy and evaluation}\label{sec:strats}
We begin by describing the problem of adapting pre-trained language models for a collection of downstream tasks while being task-agnostic, competent in performance, and data-scalable. Given these criteria, we evaluate existing task adaptation approaches and propose a representation-reasoning decomposition to understand their relative performances.  

\subsection{Problem statement and evaluation criteria}
Our focus is on methods for adapting pre-trained large language models (LLMs) for downstream tasks. Specifically, given an LLM and limited labeled data for a task, how does one adapt the model to the task? 
When evaluating a task adaptation strategy, we care about the following properties: 

\paragraph{Task-agnostic.} Given the general capabilities of pre-trained LLMs, we strive to utilize the same model across different tasks without requiring any task-specific training. With the increase in model sizes, the cost of deploying task-specific models increase both during training (expensive hyper-parameter search) as well as during inference (deploying several models). In general, task-agnostic methods will scale better with increasing model sizes by side-stepping both these costs.

\paragraph{Performance quality.} We would like the adaptation approaches to be competitive in performance when compared with task-specific approaches across a wide range of tasks. For the binary classification tasks, the method should have accuracy comparable with task-specific approaches. 

\paragraph{Data-scalable.} The task adaptation method should be scalable with the number of labeled task examples. In particular, the method should be capable of learning from large datasets, and continually improve its performance quality.

\subsection{Taxonomy of task adaptation strategies}
\label{sec:taxonomy}
We can broadly taxonomize the existing task adaptation strategies for LLMs as in-context learning, fine-tuning the model, and training task-specific adapters (see Figure~\ref{fig:algs}).

\paragraph{In-context learning.} In-context learning allows for adapting the model without updating any model parameters, by simply providing a few demonstrations of the task in the LLM prompt. 
In-context learning is completely task-agnostic since the same model can be used across tasks since no weights are updated at inference time. However, its performance is usually not at par when compared with task-specific methods and it does not scale well with data since the number of examples that can be utilized is bottlenecked by the context length of the model.

\paragraph{Fine-tuning.} This traditional class of methods update the model weights to adapt it specifically for the task, typically by performing gradient descent over the labeled dataset. 
Fine-tuning methods are not task-agnostic since they change the underlying model significantly but usually achieve state-of-the-art performance for any given task and are data scalable.

\paragraph{Adapters.} Adapters adapt the underlying LLM to a specific task by composing the LLM base model with an additional set of parameters which are optimized for the task. In contrast to fine-tuning which performs updates to the base model, adapters keep the base model frozen and only update the additional parameters. Performance of adapters is usually competitive with full fine-tuning. 

\subsection{Understanding performance via Representation-Reasoning decomposition}\label{sec:rep-reas}
From the taxonomy of task adaptation approaches, only in-context learning satisfies the task-agnostic property but it consistently underperforms the task-specific tuning approaches. This section investigates why this performance gap exists. We hypothesize that it is either because (a) the representations learned by the LLM are insufficient to learn a good predictor for the specific task, or (b) the LLM lacks the capability to reason over these representations to make good predictions for the task.

To understand whether the representations have sufficient information, we train a task-specific linear classifier using these representations, also known as linear probing, and evaluate its accuracy. Let $\accft$, $\accicl$, and $\acclr$ denote the accuracies obtained by fine-tuning, in-context learning, and by linear probing respectively.  Using this as an intermediate, we decompose the performance gap
{
\begin{equation}\label{eq:rep-decomp}
\gapperf \defn \accft - \accicl = \underbrace{\accft - \acclr}_{\gaprep} +\underbrace{\acclr - \accicl}_{\gapres}
\end{equation}
}
where $\gaprep$ represents the gap in performance which can be attributed to insufficient representation capacity and $\gapres$ is the performance gap due to insufficient reasoning abilities. Using this decomposition, we consider the following hypotheses:\vspace{-2mm}
\begin{enumerate}
    \item[H1.] LLM representations have enough information to perform the task in-context, but they lack the reasoning abilities to perform the task well.
    \item[H2.] Fine-tuning affects both the representations and reasoning but the improvement in reasoning abilities primarily leads to better performance.
    \item[H3.] Fine-tuning and adapters are not task-agnostic because the task-specific training hurts their ability to transfer reasoning.
\end{enumerate}
We now analyze each of the task adaptation approaches through the lens of the above hypotheses. We perform all experiments with three different classes of language models (\neo, \pythia, \bloom) across a collection of 6  binary classification tasks. 
See Appendix~\ref{app:sec-2} for further details.

\begin{figure*}
    \centering
    \begin{subfigure}[b]{0.31\textwidth}
        \centering
        \includegraphics[width=1\textwidth]{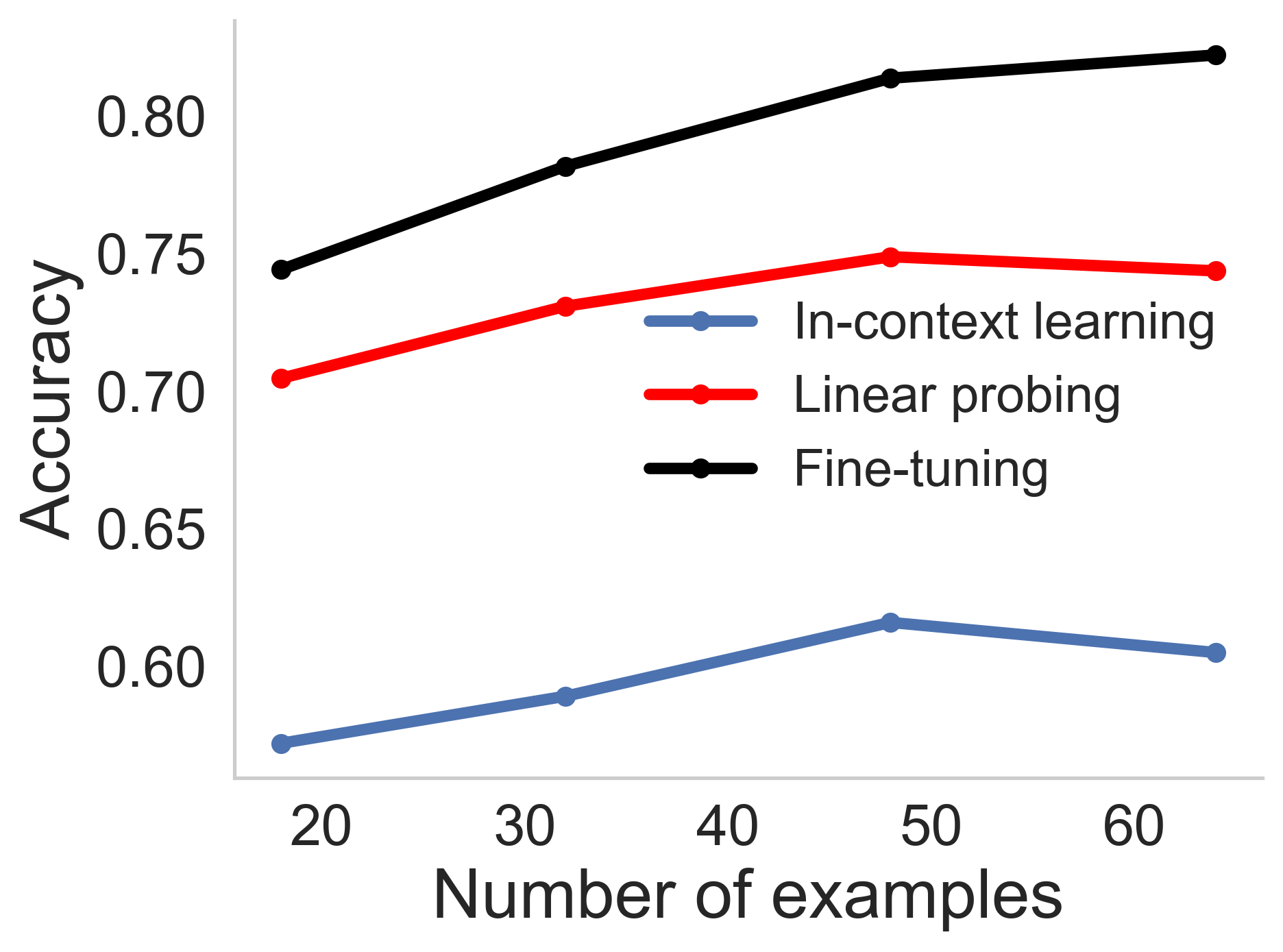}
        \subcaption{}
        \label{fig:icl-micro-main}
    \end{subfigure}
    \begin{subfigure}[b]{0.31\textwidth}
        \centering
        \includegraphics[width=0.6\textwidth]{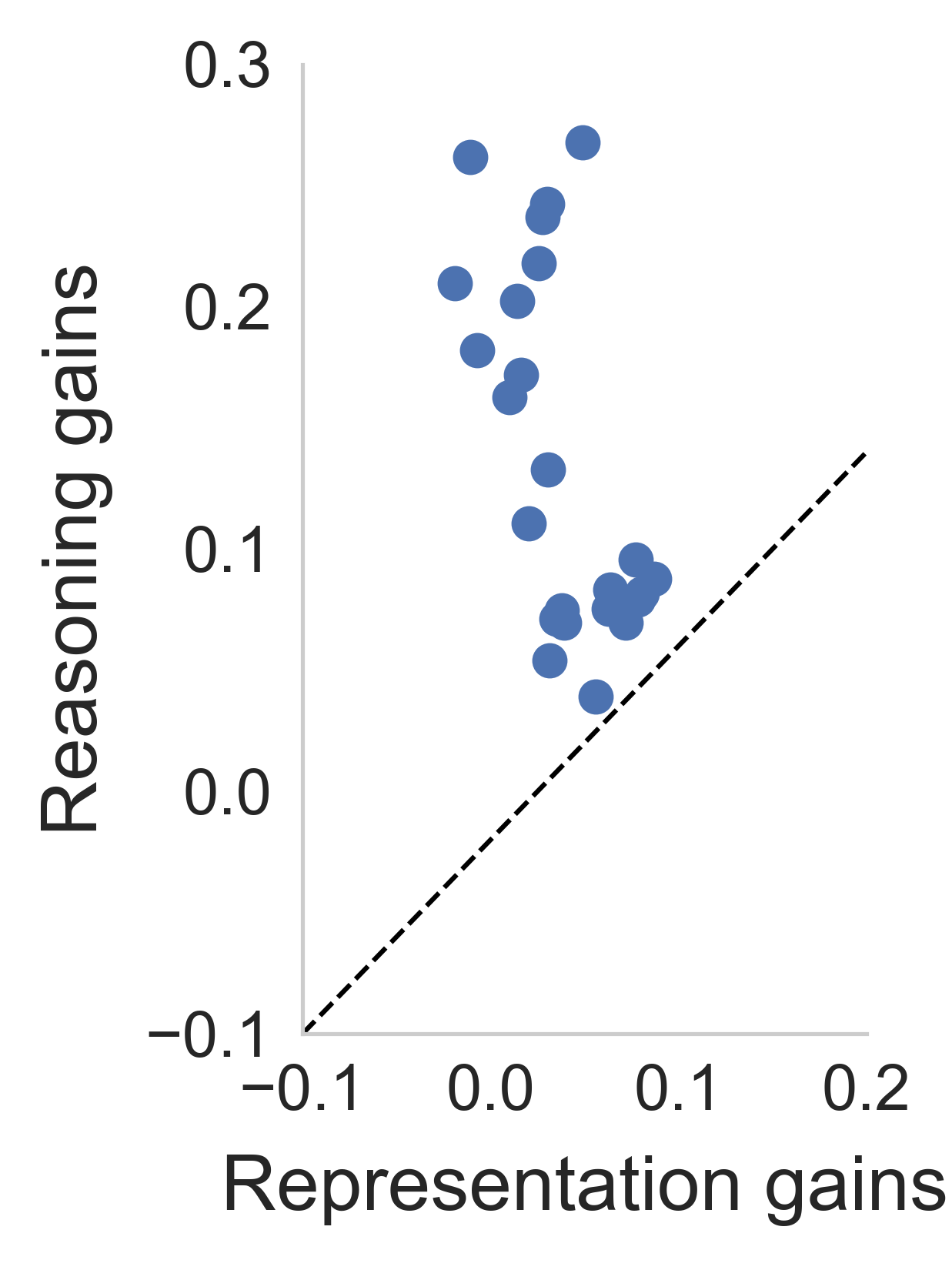}
        \subcaption{}
        \label{fig:rep-reas-main}
    \end{subfigure}
     \begin{subfigure}[b]{0.31\textwidth}
        \centering
        \includegraphics[width=1\textwidth]{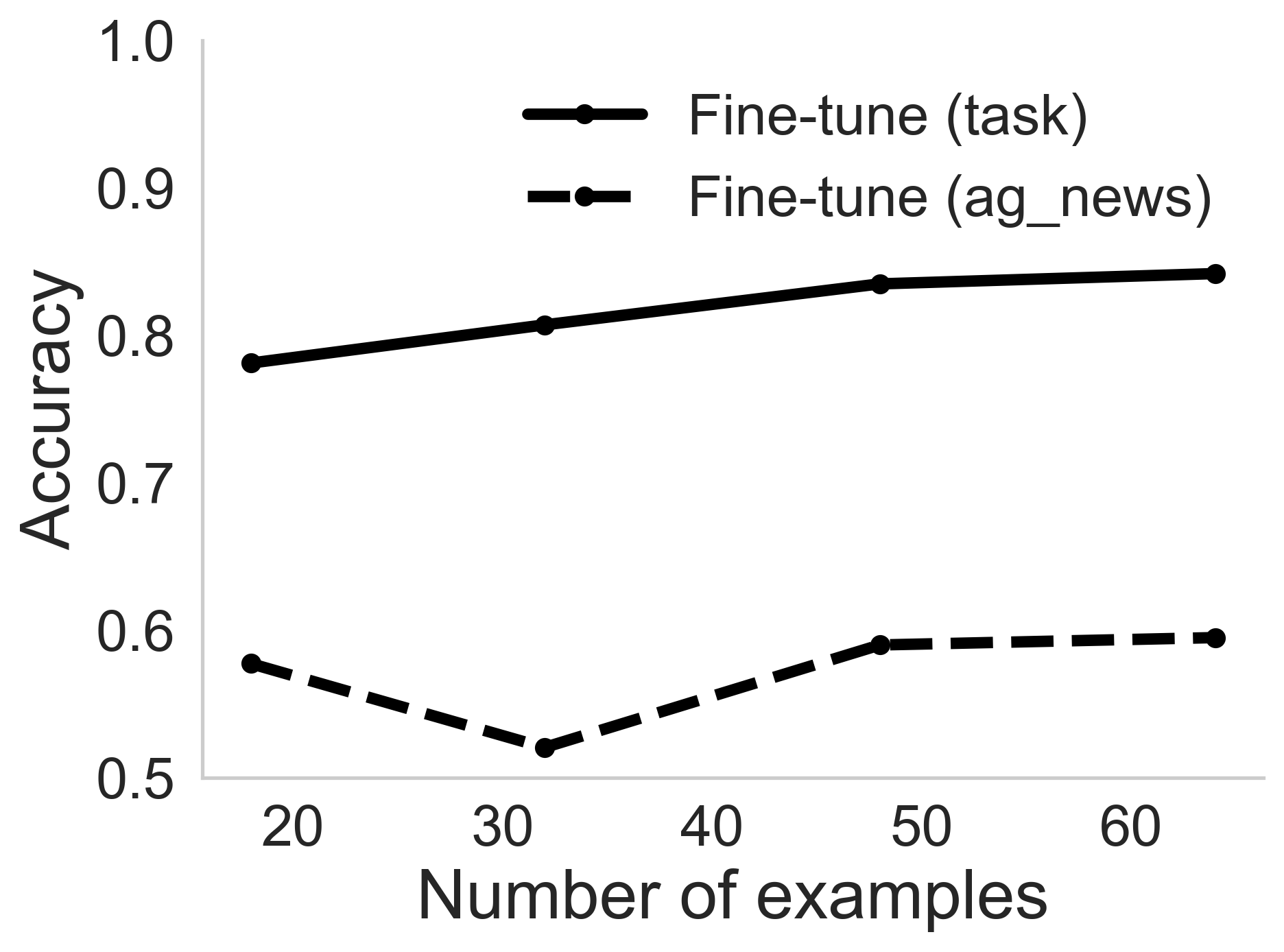}
        \subcaption{}
        \label{fig:agnostic-m2-main}
    \end{subfigure}
    \caption{All results for \gptsmall. (a) Accuracy of in-context learning vs. linear probing on model embeddings: representations have sufficient information. (b) Scatter plot showing the representation and reasoning gains (see eq.~\eqref{eq:rep-decomp}) across different NLP datasets for a fine-tuned model when compared to base in-context learning. Fine-tuning majorly improves task-specific reasoning across datasets. (c) Accuracy (averaged across $6$ datasets) of model fine-tuned on AGNews and tested on a separate task X vs model fine-tuned on task X and tested on the same task X. On average, fine-tuning hurts task-agnosticity can can be up to $25\%$ off from fine-tuning for the specific task.}
\end{figure*}

\paragraph{In-context learning: LLMs lack reasoning abilities.}
We begin by studying the representation and reasoning gaps, as defined in eq.~\eqref{eq:rep-decomp}, for in-context learning. In Figure~\ref{fig:icl-micro-main}, we plot the average accuracy across datasets for in-context learning, task-specific fine-tuning, and linear probing. We see that across models and different numbers of in-context examples, the reasoning gap $\gapres$ accounts for up to $79.11$\% of the performance gap between in-context learning and fine-tuning. This indicates that the LLM representations have sufficient information but  lack the ability to reason over them.

\paragraph{Fine-tuning: Improves task-specific reasoning.}
We next investigate how fine-tuning for a specific task affects the performance of the base model. In Figure~\ref{fig:rep-reas-main}, we show a scatter plot of the gains that can be attributed to improved representations against the reasoning gains. We see that, across models, reasoning improvements accounts for $73.06$\% of the improvements. This indicates that while fine-tuning improves both reasoning and representations of the LLM, the gains are predominantly due to improvements in task-specific reasoning. Furthermore, this task-specific fine-tuning of the LLM hurts its performance on other tasks. In Figure~\ref{fig:agnostic-m2-main}, we show that the accuracy of a model fine-tuned on the AGNews dataset~\citep{zhang2015character}, leads to an average decrease of $25.77$\% on other tasks. Furthermore, this drop in accuracy can be attributed to the drop in task-specific reasoning capabilities---these account for $72.58$\% of the drop (see Appendix~\ref{app:sec-2} for more details).

\paragraph{Adapters: Impairs task-agnosticity via reasoning.}
Task-specific adapters do not change the underlying representation ability of the model. To study their ability to generalize across tasks, we train an adapter for the AGNews dataset and evaluate it on other tasks. In Appendix~\ref{app:sec-2}, we show that the performance drops across tasks by an average of $19.8$\%, indicating that adapters only learn task-specific reasoning abilities.

\section{\alg: Task-Agnostic Reasoning Transformers}\label{sec:fluffy}
The above analysis showed how it is the effective reasoning capabilities of the LLMs which limits its performance when compared with task-specific adaptation approaches. Building on this insight, we propose \alg, which learns a general-purpose reasoning module completely agnostic to the underlying base LLM and when composed with any LLM via its embeddings, generically improves upon its reasoning abilities. \alg is a completely task-agnostic method which works across a suite of tasks without any task-specific training. 

\alg comprises of two components: a generic task-agnostic reasoning module, and embeddings from the base LLM. The reasoning module is trained using only synthetic data (Gaussian logistic regression problems), agnostic of the auto-regressively trained language model, with the objective of learning to perform probabilistic inference (Section~\ref{sec:res-mod}). This learned transformer module is then composed with the base LLM, without any training, by simply aggregating the output embedding and using those as an input along with the class label (Section~\ref{sec:rep-emb}). Together, these components make \alg task-agnostic, boost performance quality by improving reasoning, and make the approach data-scalable by aggregating input embeddings into a single vector.

Intuitively, the Gaussian logistic regression task is a simple probabilistic reasoning task wherein the objective is to regress a given feature vector to a discrete binary label. Teaching an independent module to perform a family of these tasks and composing them with pre-trained language models can be seen as a way to generically improve upon the LLM's reasoning abilities by making them perform such regression better.

\subsection{Reasoning module: Can Transformers learn probabilistic inference?}\label{sec:res-mod}
\alg's reasoning module is a Transformer-based model which is trained to perform probabilistic inference in-context using only synthetically generated data.

\subsubsection{Training the reasoning module}\label{sec:training}
The reasoning module is a Transformer model which is auto-regressively trained on a \emph{family} of logistic regression tasks, with each input sequence corresponding to a different logistic regression problem. We next describe the model architecture and the training procedure. \vspace{-3mm}

\paragraph{Model architecture.} The reasoning module is based on the standard decoder-only Transformer architecture from the GPT-2 family (see Appendix~\ref{app:tr-reas-det} for details). The architecture takes as input a sequence of vectors and is trained to predict the next vector in the sequence. The input sequence consists of $k$ pairs of labeled examples $(\x_1, \y_1), (\x_2, \y_2), \ldots, (\x_k, \y_k)$, with each example $z_i = (\x_i, \y_i)$ using only two input positions of the transformer --  one for the covariates $x$ and the other for the label $y$. This is in contrast to standard LLMs where each example is spread over multiple tokens which limits how many examples can be put in the context. For example, with a context window of 2048, our module can support 1024 examples while the base model can support only 10 examples, assuming each demonstration comprises 200 natural language tokens. \vspace{-3mm}



\paragraph{Training procedure.} This module is trained using gradient descent to minimize the population loss
{
\begin{equation}
\loss(\Tr_\paramT) \defn \En_{\x, \y}\left[\frac{1}{k} \sum_{i=1}^k \lossce (\Tr_\paramT(z_{1:i-1}, x_i), \y_i)\right]\;,
\end{equation}}
where $z_{1:i-1}$ corresponds to the first $i-1$ examples and $\lossce$ is the cross-entropy loss evaluated on the transformer prediction and the true $y_i$. Each training sequence $\seq_t$ used to update the parameters $\param$ comprises a different $\dx$-dimensional logistic regression problem, sampled as
{
\begin{equation}\label{eq:data-gen-syn}
\text{Sequence } \seq_t: \param_t \sim \normal(0, I_\dx), \quad \x_{i,t}\sim \normal(0, I_\dx), \quad \y_{i,t} \sim \sigmoid(\weight\inner{\x_{i,t}}{\param_t}) \qquad \text{for }i \in [\klr]\;,
\end{equation}
}
where $\sigmoid$ represents the sigmoid function and the multiplier $\weight$ determines the noise level of the problem. We train our model with $\dx=16$ and $\klr = 256$. Observe that the loss is only computed on the predicted output of the features $\x$ in the sequence. We describe the model hyper-parameters and the training procedure in more detail in Appendix~\ref{app:tr-reas-det}. 

We trained the reasoning module with input dimension set to 16 with the labels $\y$ encoded in this space using a one-hot encoding by appending the true label with zeros. While most base models produce representations which are much higher dimensional (ranging from 784 to 2048). In order to reduce the dimensionality of these representations, we perform PCA on the output embeddings of the base model, learning the components using only the training points available for that specific task. The test examples are then projected onto these principal components to produce 16 dimensional input representations.

\subsubsection{Properties of reasoning module}
The task-agnostic reasoning module described above is trained to perform well on a family of logistic regression tasks. We study some properties of the reasoning module, in particular how well it learns to perform the task at an instance level and how robust is it to variations in the noise level $\weight$. 

\paragraph{Accuracy of probabilistic inference.} 
For understanding the instance level performance of our reasoning module, we evaluate it on a sample of $64$ different logistic regression problems, sampled according to eq.~\eqref{eq:data-gen-syn}. For each problem, we train task-specific linear classifiers using logistic regression and compare them with our task-agnostic reasoning module. In Figure~\ref{fig:tartlrcomp} we plot the deviation of the predicted probabilities (averaged over the $64$ problems) from the true probabilities for our reasoning module and the task-specific logistic solvers as a function of the number of examples used for predictions. We observe that the error for our reasoning module decreases as a function of the number of in-context examples and is within $2\%$ of the task-specific logistic function.


\paragraph{Robustness to noise level.}
We study the robustness of the learned module to the noise levels, $\weight$, of the logistic regression problem. Recall that we trained our inference module by fixing the noise level $\weight = 10$. At inference time, we vary the noise level to $[0.5, 1, 10, 20]$, where lower values corresponds to noisier problem. The reasoning module generalizes to easier problem without any drop in accuracy but as we make the problem harder $(\weight = [0.5, 1])$, the error increases progressively (see Figure~\ref{fig:noise}). 


\begin{figure}[t!]
    \centering
    \begin{subfigure}[b]{0.3\textwidth}
        \centering
    \includegraphics[width=0.9\textwidth]{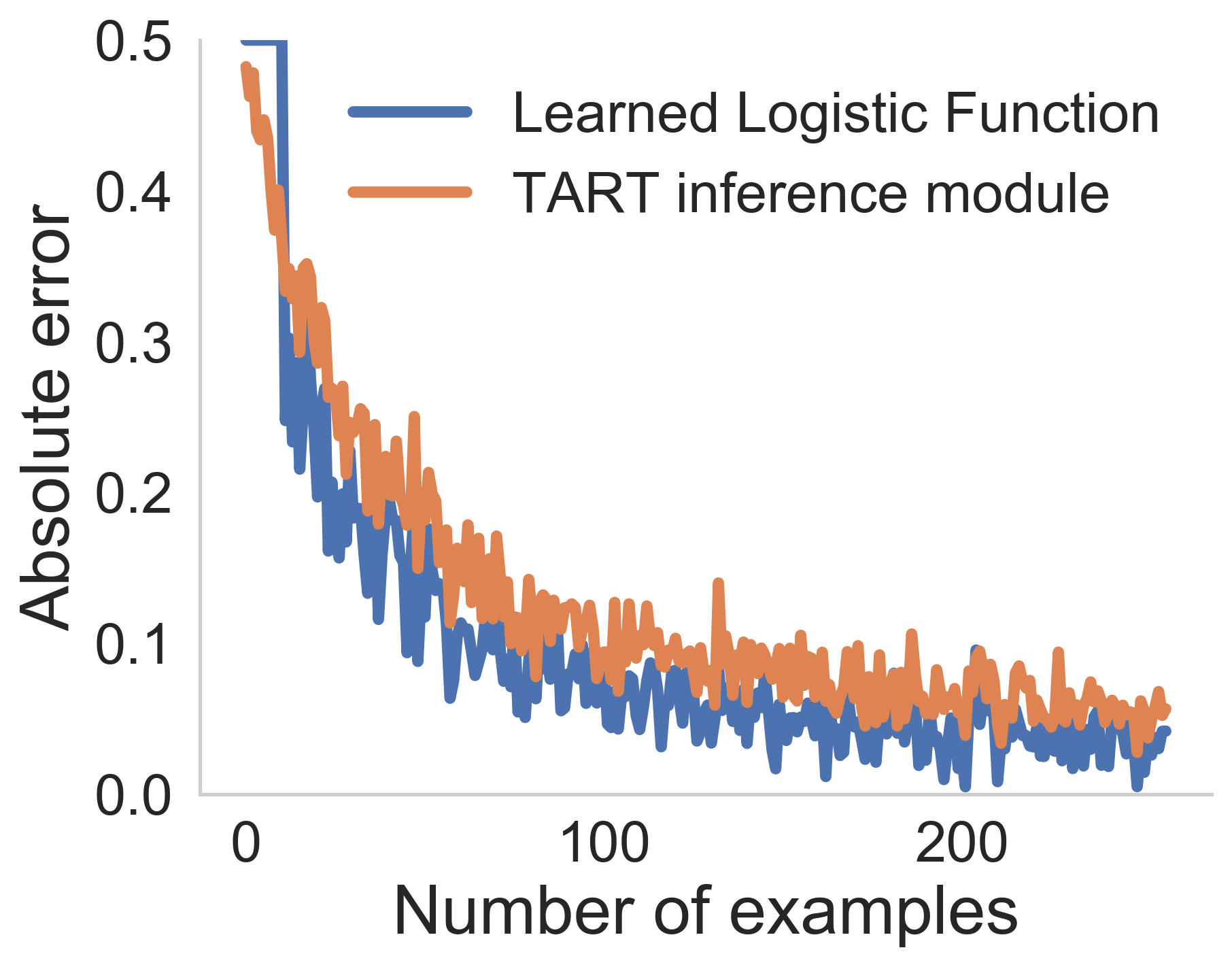}
    \subcaption{}
    \label{fig:tartlrcomp}
    \end{subfigure}
    \begin{subfigure}[b]{0.3\textwidth}
        \centering
    \includegraphics[width=0.9\textwidth]{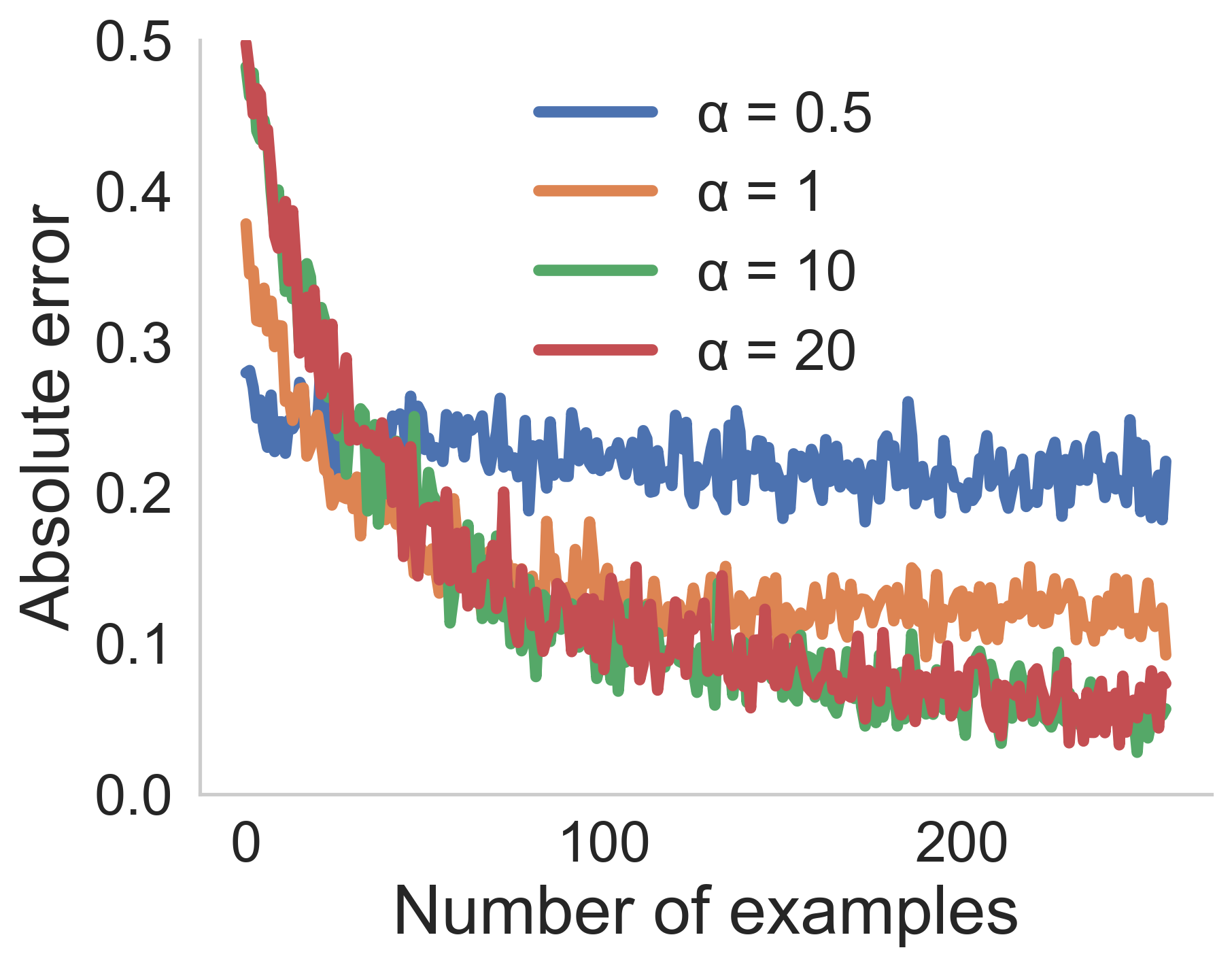}
    \subcaption{}
    \label{fig:noise}
    \end{subfigure}
    \begin{subfigure}[b]{0.3\textwidth}
        \centering
    \includegraphics[width=0.9\textwidth]{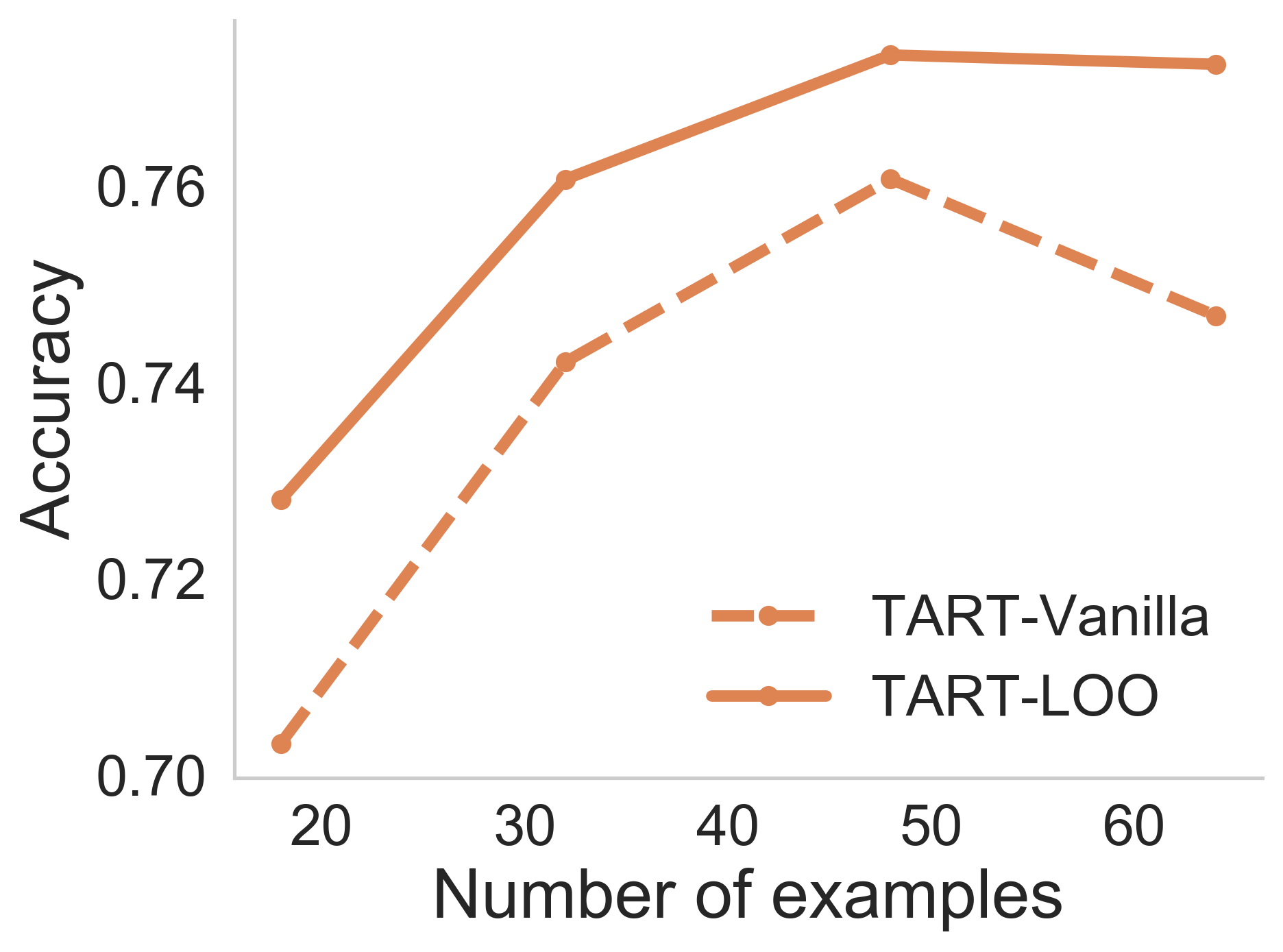}
    \subcaption{}
    \label{fig:loovanilla}
    \end{subfigure}
    \caption{\textbf{Properties of \alg's inference module}. (a) Comparison with learned logistic function: inference module recovers underlying probabilities. (b) Variation in error with different noise levels for model trained on $\weight=10$. (c) Comparison of \alg performance when using LOO embeddings and vanilla embeddings. }
    \label{fig:fluffy}
\end{figure}

\subsection{Role of representations: Which embeddings to take?}\label{sec:rep-emb}
The reasoning module composes with a base LLM through its final layer embeddings. A natural way to produce these embeddings is to place all the train examples in-context and then average the embedding vectors corresponding to the particular example (see Figure~\ref{fig:emb-vanilla}). At inference time, we append the test example to the training set, and average the embeddings corresponding to this example. We call these vanilla embeddings. Our experiments reveal that these embeddings seem to saturate (or even hurt performance) beyond a certain number of in-context examples (see Figure~\ref{fig:loovanilla}). One  reason can be that the causal nature of the model causes these embeddings to have asymmetric information---the embeddings of each example is influenced by its preceding examples. 



To counter this asymmetry, we propose \emph{leave-one-out} (LOO)  embeddings where the embeddings for each training point is formed by placing all the other train examples before it in the prompt such that all the embedding are formed with the same information content (see Figure~\ref{fig:emb-loo}). In Figure~\ref{fig:loovanilla}, changing the embedding style from vanilla to LOO consistently improves performance across models and tasks. The LOO-embeddings help \alg be data-scalable by enabling it to embed a much larger number of points than the context window can support. To do so, we use only a subset of the train examples as the in-context prompt. The reasoning module, by its architecture design, can already accommodate many more examples than supported by the context window of the base LLM.


\begin{figure*}
    \centering
    \begin{subfigure}[b]{0.4\textwidth}
        \centering
        \includegraphics[width=1\textwidth]{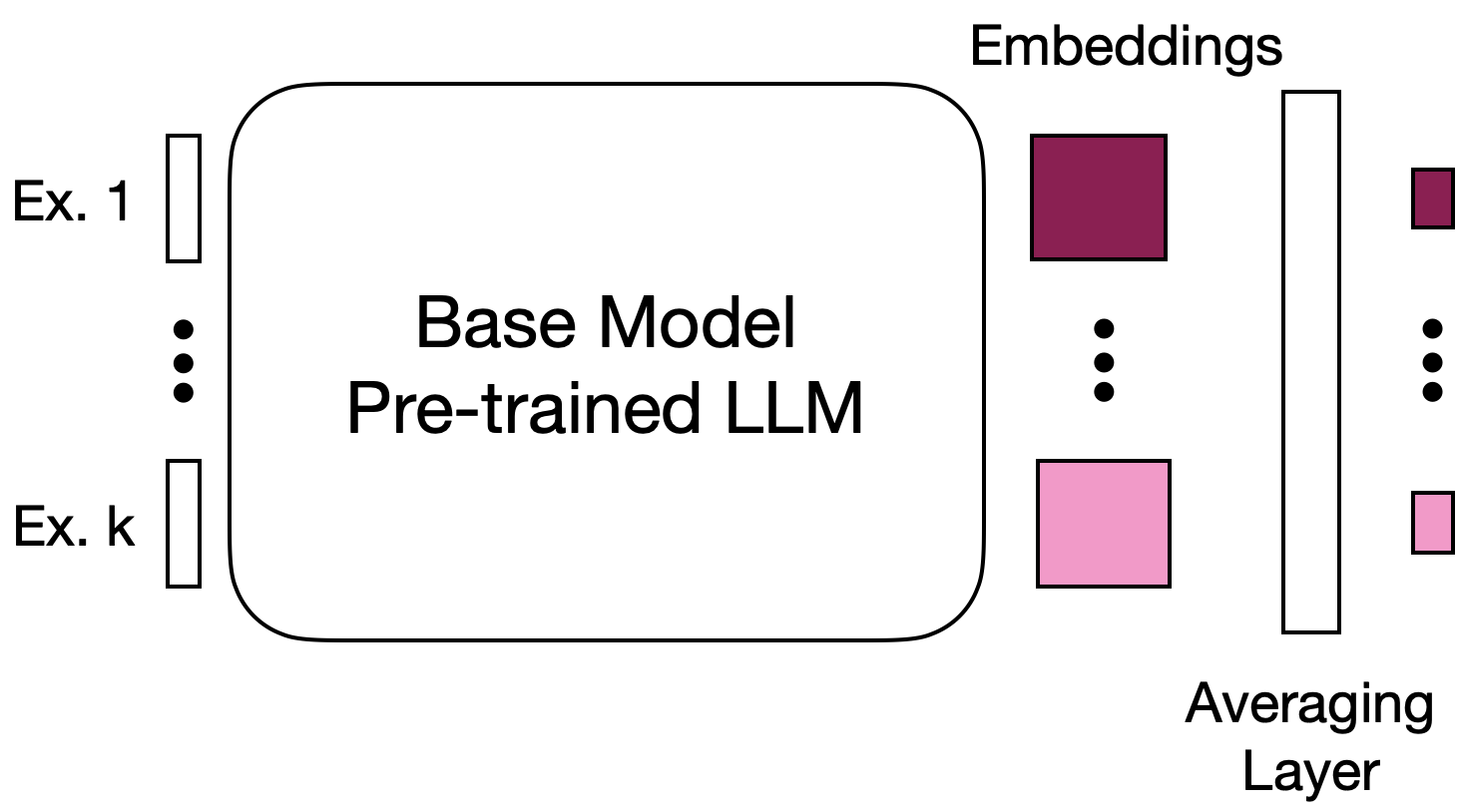}
        \subcaption{Vanilla Embeddings}
        \label{fig:emb-vanilla}
    \end{subfigure}
    \hspace{12mm}
    \begin{subfigure}[b]{0.48\textwidth}
        \centering
        \includegraphics[width=1\textwidth]{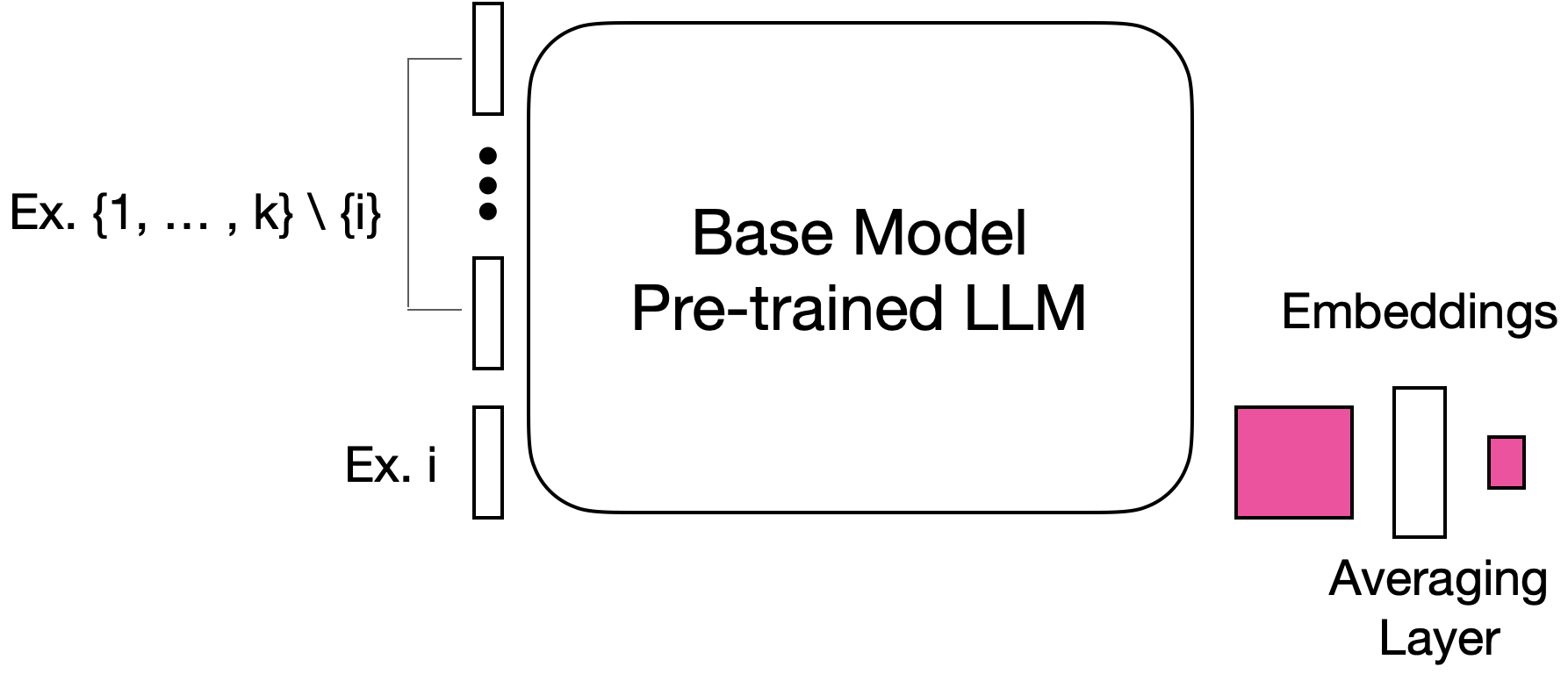}
        \subcaption{LOO Embeddings}
        \label{fig:emb-loo}
    \end{subfigure}
    \caption{\textbf{\alg Embedding Protocols}. (a) For the vanilla embeddings, the test example is appended to the training set and the sequence is passed to the base model. The representation for each train example in this sequence is taken as the average embedding across all its tokens. (b) For the LOO embeddings, we generate embeddings for each train example separately by placing all the other train examples before it in the prompt and averaging the embeddings over the final example's tokens. The figure shows how to compute the embedding for the $i^{\text{th}}$ training example.}
\end{figure*}

\subsection{Theoretical analysis: Generalization of \alg to language tasks}\label{sec:theory}
We study the generalization properties of the proposed task-agnostic method \alg. Note that that the inference module is trained completely on synthetic data while at evaluation time, our input is the embeddings from a natural language task. In Theorem~\ref{thm:gen} we show that its performance on the natural language task depends on the distribution shift from the synthetic to the true distribution (see Appendix~\ref{app:theory} for a formal statement and proof).

\begin{theorem}[Informal]\label{thm:gen}
Let $\T$ represent the class of transformer models and $\Tres \in \T$ denote the trained reasoning module on set $S$ of synthetic regression with $\nsyn$ sequences sampled from distribution $\Psyn$ in eq.~\eqref{eq:data-gen-syn}. The error of the transformer $\Tres$ when evaluated on a distribution $\Pnl$ over natural language sequences is
\begin{equation}
\err_{\Pnl} \lesssim W_1(\Pnl, \Psyn) + \sqrt{\frac{\text{Comp}(\T)}{\nsyn}} + \hat{\err}_{\Psyn}(\Tres)\;,
\end{equation}
where $W_1$ denotes the Wasserstein-1 metric, $\text{Comp}(\T)$ represents the complexity of class $\T$, and $\hat{\err}$ represents the error on the empirical distribution. 
\end{theorem}

A few comments are in order: The first term  represents the distribution shift error between the true natural language task and the synthetic task. The second term corresponds to the generalization error on the logistic regression task, which can be made arbitrarily small since it scales with $\nsyn$, the number of synthetic datapoints which can be generated without any cost. The third term is the optimization error indicating how well has the reasoning module $\Tres$ fit to the synthetic training set.


\section{Experimental evaluation}\label{sec:exp}
We evaluate \alg on a wide range of binary classification tasks across three domains: language, vision and audio. We demonstrate that \alg improves base in-context performance and closes the gap with standard task-specific strategies. We also conduct ablations to demonstrate that \alg scales with model size and can support 10x more samples than in-context learning.

\subsection{Experimental setup}

\paragraph{Datasets.} We briefly describe the datasets used, with details available in Appendix~\ref{app:exp-setup}. We consider $14$ different binary classification tasks ranging from sentiment classification, news article categorization to spam detection. The evaluation datasets include: SST~\citep{socher2013recursive}, Rotten Tomatoes~\citep{Pang+Lee+Vaithyanathan:02a}, SMS Spam~\citep{Almeida2011SpamFiltering}, IMDB~\citep{maas-EtAl:2011:ACL-HLT2011}, Civil Comments~\citep{DBLP:journals/corr/abs-1903-04561}, AGNews~\citep{zhang2015character}, DBPedia~\citep{zhang2015character}, and the Youtube dataset~\citep{zhang2021wrench}. Since AGNews and DBPedia14 are multi-class datasets, we construct 4 binary classification tasks from each dataset respectively.
For each dataset, we truncate the input text to be at most 100 characters to enable us to fit sufficient number of samples in-context.\vspace{-3mm}



\paragraph{Model families.} We evaluate our method across three different families of models: \neo~\citep{gpt-neo}, \pythia~\citep{biderman2023pythia}, and \bloom~\citep{scao2022bloom}. For our evaluations across 14 datasets, we use \gptsmall, \pythiasmall and \bloomsmall. For ablations on larger models, we evaluate models with~1B parameters across each of the model families (i.e., \gptmedium, \pythiamedium and \bloommedium) and models with~3B parameters (i.e., \gptlarge, \pythialarge and \bloomlarge). We additionally evaluate on \gptj~\citep{gpt-j}.  \vspace{-3mm}

\paragraph{Baselines.} We evaluate our models against all types of task-adaptation strategies described in Section~\ref{sec:taxonomy}: 1) in-context learning, 2) full fine-tuning, 3) last layer fine-tuning, 4) LM head fine-tuning, and 5) adapters. For each baseline, we perform an extensive hyper-parameter search over number of epochs and learning rate for each dataset in order to optimize performance (see Appendix~\ref{app:exp-setup} for hyperparameter details). For \alg, we chose a base default set of parameters and use the \emph{same} inference module with the exact same weights for all the experiments in this section. 

\subsection{Natual language benchmark evaluations}
For this section, all reported accuracies are averaged over $5$ independent random seeds. A complete set of results with standard deviations can be found in Appendix~\ref{app:nl-eval}. \vspace{-3mm}

\paragraph{Performance with respect to baselines.} 
As shown in Appendix~\ref{app:nl-eval}, averaged across all tasks and model families, \alg improves upon the base in-context learning performance by an average of $18.4$ points, improves upon adapter heads by $3.4$ points, and is within $3.1$ points of full fine-tuning. We also observe that \alg consistently outperforms the task specific strategies of LM head fine-tuning and last layer fine-tuning.\vspace{-3mm}

\begin{table}[t!]
\centering
\begin{tabular}{lrrrrr}
\toprule
        Model & \textbf{\alg} & \textbf{\gptj} & \textbf{\opthuge} & \textbf{\bloomhuge} & \textbf{\gpt} \\
\midrule
              Accuracy  & 0.634 & 0.608 & 0.637 & 0.595 & 0.673 \\
\bottomrule
\end{tabular}
\vspace{2mm}
\caption{\textbf{RAFT (HELM) Binary Classification Performance (Average Accuracy)}. \alg is used with \gptsmall model which is $1000$x smaller than the corresponding 175B parameter models. \alg outperforms  \bloomhuge  and is competitive with \opthuge and \gpt.}
\label{tab:raft}
\end{table}

\paragraph{Performance on RAFT Benchmark} 
We evaluate \alg on all binary classification tasks in the RAFT Benchmark, following the protocol used in HELM~\citep{liang2022holistic}. When applied with \gptsmall, \alg outperforms  \bloomhuge, and is within $4\%$ points of \gpt, both of which are $1000$x larger in size. See Table~\ref{tab:raft} for exact accuracies.


\begin{figure*}[t!]
    \centering
    \begin{subfigure}[b]{0.31\textwidth}
        \centering
\includegraphics[width=1\textwidth]{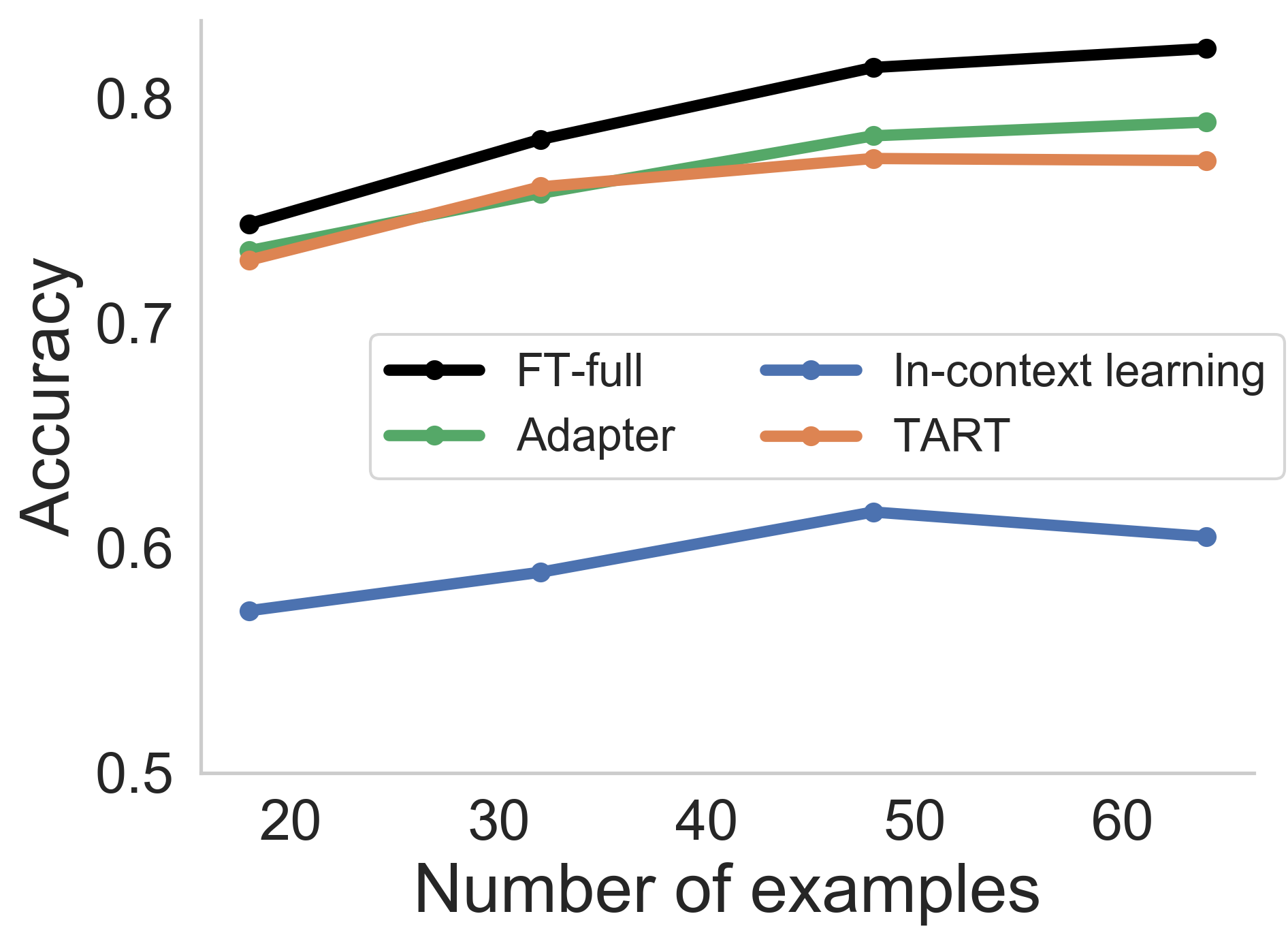}
    \subcaption{}
    \label{fig:context-length}
    \end{subfigure}
    \begin{subfigure}[b]{0.31\textwidth}
        \centering
\includegraphics[width=1\textwidth]{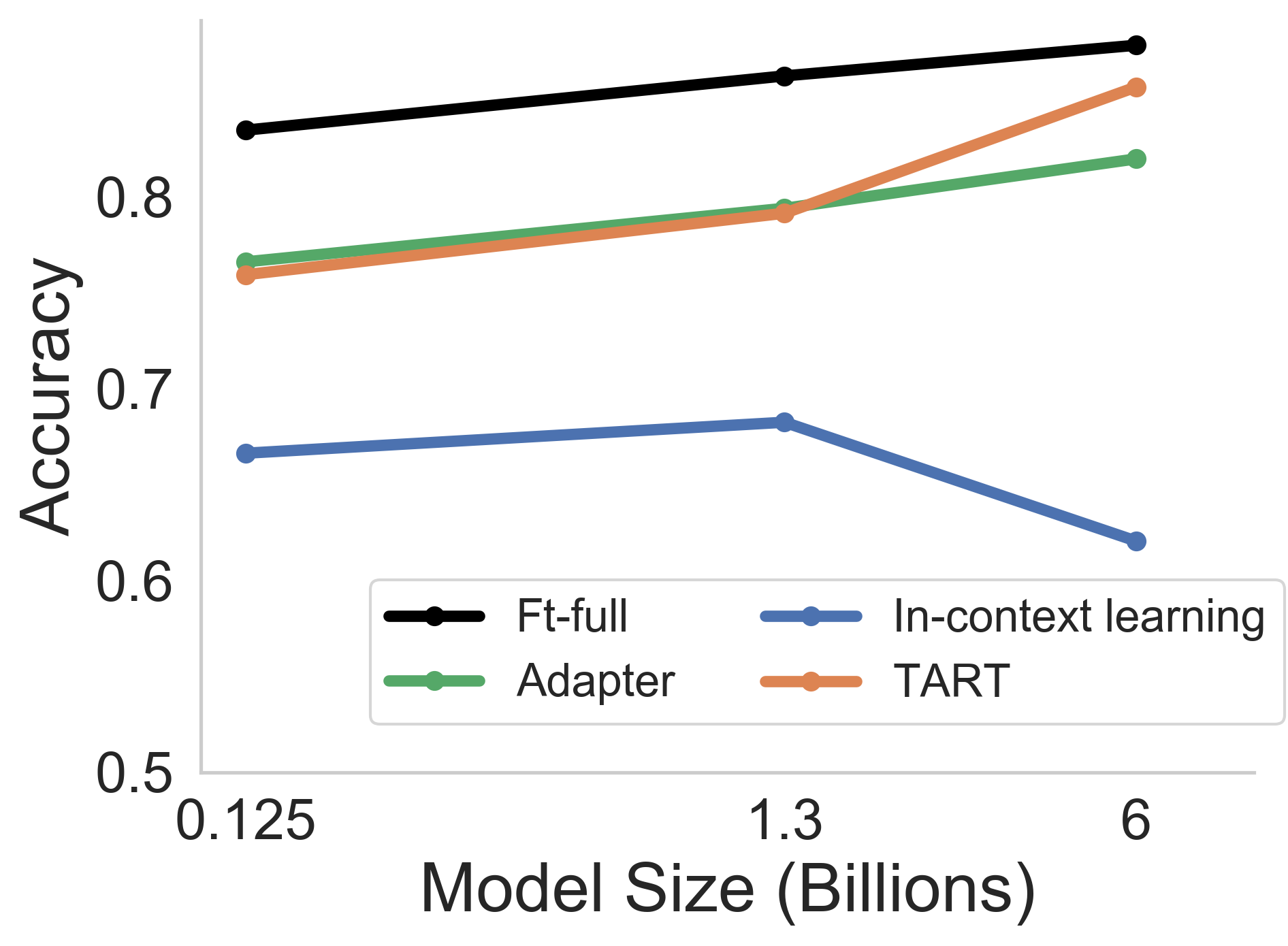}
    \subcaption{}
    \label{fig:model-scale}
    \end{subfigure}
    \begin{subfigure}[b]{0.31\textwidth}
        \centering
\includegraphics[width=1\textwidth]{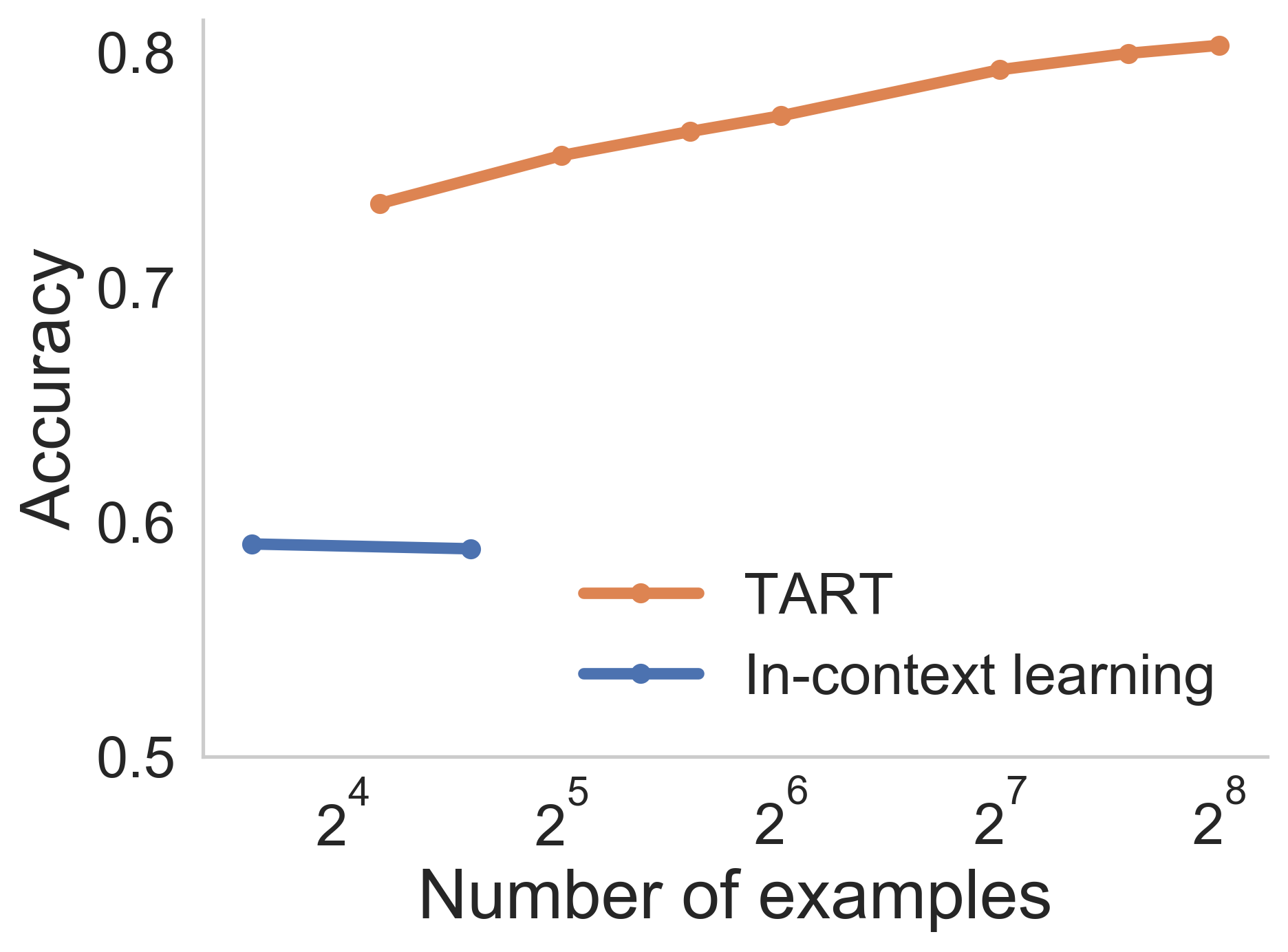}
    \subcaption{}
    \label{fig:context-window}
    \end{subfigure}
    \caption{\textbf{Effects of scale}. (a) Effect of number of in-context examples on performance for different task adaptation strategies. (b) Effect of model size on the performance of different task adaptation strategies. (c) Beyond context length limitations, performance comparison with respect to number of in-context examples.}
    \label{fig:sec-exp}
\end{figure*}

\vspace{-3mm}
\paragraph{Performance with number of in-context examples.} 
Our results demonstrate that performance of \alg scales with number of in-context examples (see Figure~\ref{fig:context-length}). Across $14$ tasks and $3$ model families, when scaling from $18$ to $64$ examples, \alg improves performance by an average of $4.8$\%. Correspondingly, full fine-tuning improves performance by $9.0$\%.\vspace{-3mm}



\paragraph{Scaling with base model size.} We analyze how different task-adaptation strategies scale with respect to model size using the GPT-Neo family: \gptsmall, \gptmedium and \gptj. 
Figure~\ref{fig:model-scale} shows that when scaling from 100M to 6B parameters, performance of task-specific methods and \alg increases as a function scale. For \alg, the performance increases by $9.8\%$ while using the same inference module across model sizes. Furthermore, the difference in performance between \alg and fine-tuning baseline reduces from $7.5\%$ to $2.2\%$ from the 100M scale to 6B scale. 
\vspace{-3mm}
\paragraph{Beyond context length.} 
We evaluate the data-scaling properties for both in-context learning and \alg (Figure~\ref{fig:context-window}). To demonstrate the scaling property, we do not truncate the input text to $100$ characters and utilize the entire text sequences. For \alg, we observe that accuracy continues to improve when scaling from $18$ to $256$ in-context examples with $6.8$\% lift in performance. In comparison, ICL, which is bottlenecked by context length, supports $10$x less samples, with the context window saturating at $24$ examples only and lags \alg by an average of $19.1$\%.




\subsection{Extensions to other modalities}
\begin{figure*}[t!]
    \centering
    \begin{subfigure}[b]{0.31\textwidth}
        \centering
    \includegraphics[width=1\textwidth]{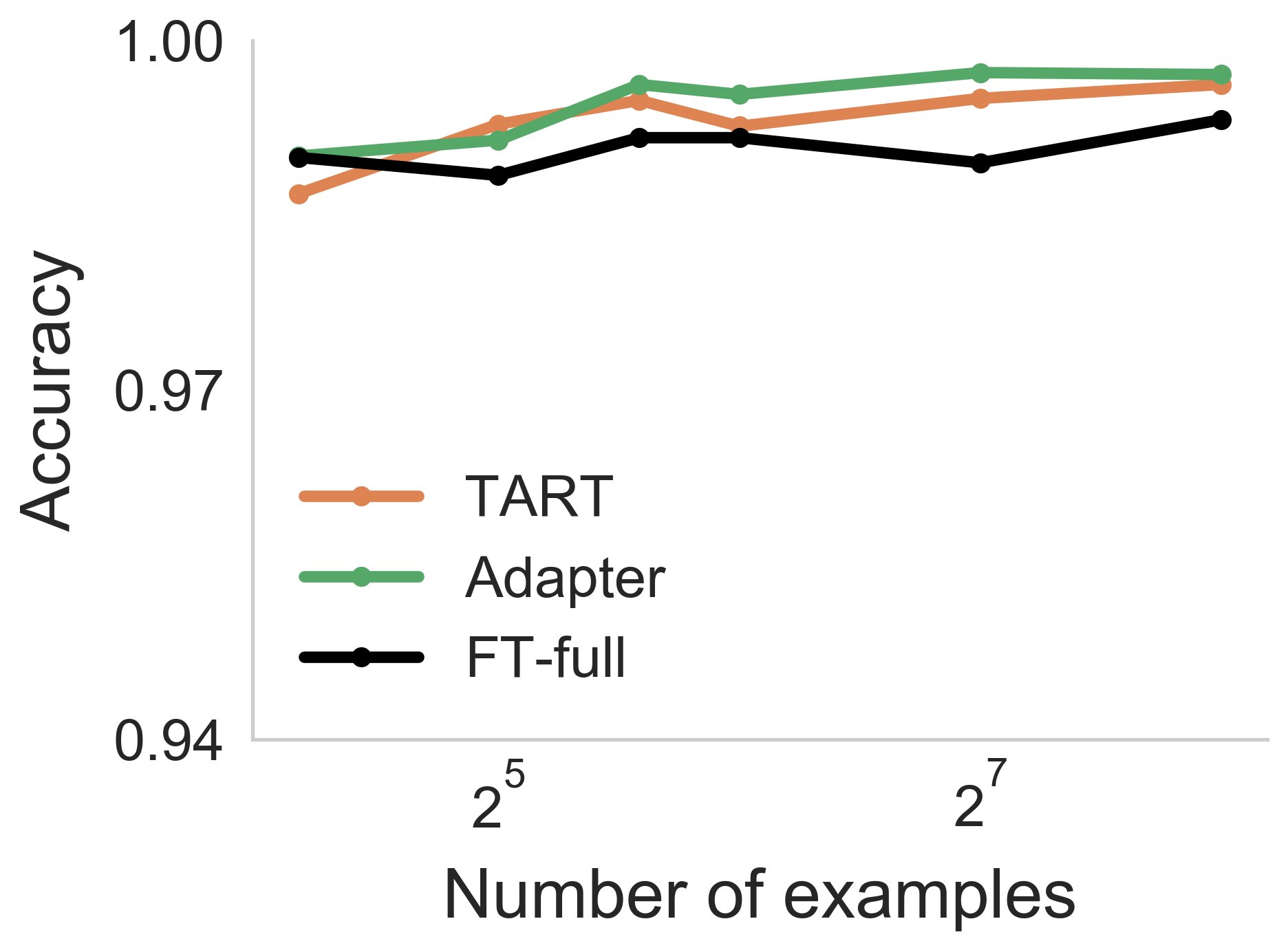}
    \subcaption{MNIST}
    \label{fig:mnist}
    \end{subfigure}
    \begin{subfigure}[b]{0.31\textwidth}
        \centering
    \includegraphics[width=1\textwidth]{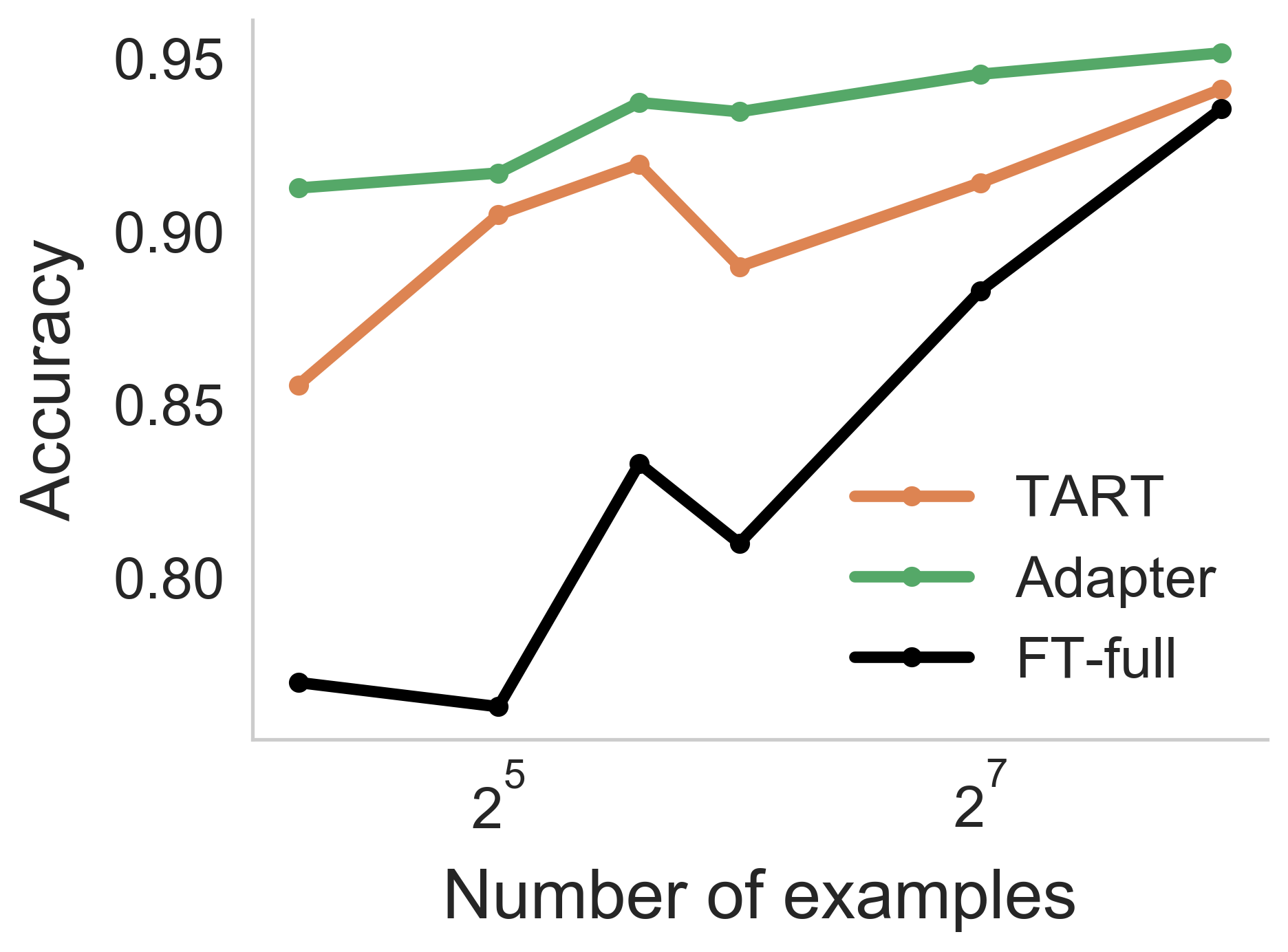}
    \subcaption{CIFAR-10}
    \label{fig:cifar10}
    \end{subfigure}
    \begin{subfigure}[b]{0.31\textwidth}
        \centering
    \includegraphics[width=1\textwidth]{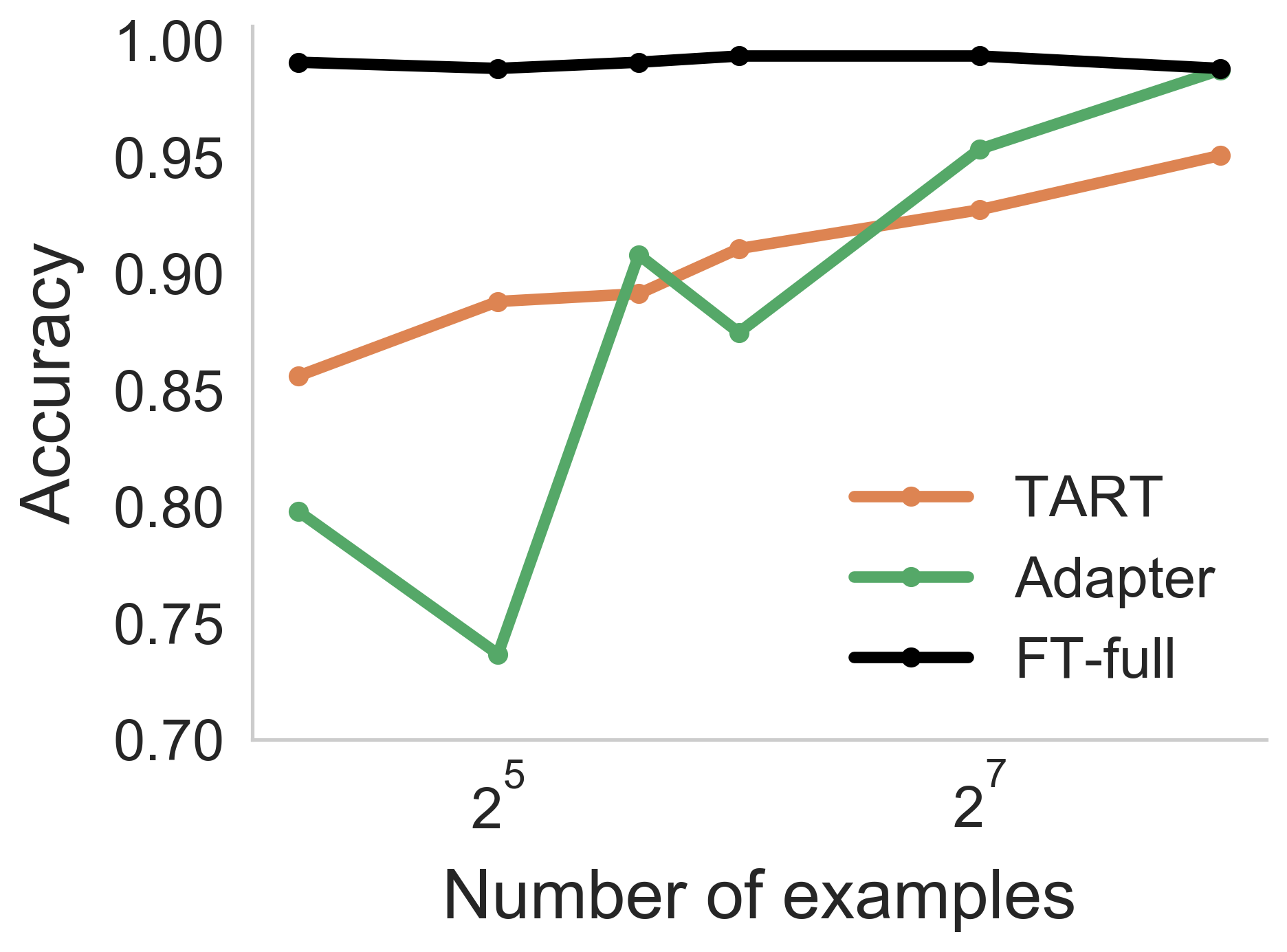}
    \subcaption{Speech Commands}
    \label{fig:speechcommands}
    \end{subfigure}
    \caption{\alg can generalize across domains using the same inference module that was used for language benchmarks: Performance across vision tasks (MNIST, CIFAR-10) and an audio task (Speech Commands).}
    \label{fig:modalities}
\end{figure*}

We demonstrate that \alg is not only agnostic to models and tasks, but also modalities.  We extend \alg to  classification tasks on modalities beyond language: vision and audio. For vision tasks, we use representations from Google's 307M parameter pretrained Vision Transformer (ViT) model~\cite{wu2020visual}: \vit. For audio tasks, we use representations from OpenAI's 1.5B parameter pretrained Whisper model~\citep{radford2022whisper}: \whisper. In applying \alg to the representations from these models, we provide a way for performing in-context learning in modalities beyond text. We refer the reader to Appendix~\ref{app:modalities} for further details on the experiment setup.

\vspace{-3mm}
\paragraph{Vision application.} We evaluate the performance of \alg on binary classification versions of CIFAR-10~\cite{Krizhevsky09learningmultiple} (classes plane and bird)  and MNIST~\citep{lecun2010mnist} (classes 0 and 8). As shown in Figure~\ref{fig:mnist} and~\ref{fig:cifar10}, performance of \alg is competitive with task-specific adaptation approaches. \vspace{-3mm}

\paragraph{Audio application.} We evaluate \alg on a binary classification version of the Speech Commands dataset~\citep{speechcommandsv2}, where the task is to classify ``stop'' and ``go'' utterances. As shown in Figure~\ref{fig:speechcommands}, performance of \alg is competitive with task-adaptation approaches.

\section{Discussion}\label{sec:conc}
We look at the problem of task-agnostic learning with LLMs.
%
%
We show that LLMs lack the ability to perform simple reasoning over their learned representations and introduce \alg, a task, model and domain agnostic method for improving their reasoning abilities.  
%
In this work, we focus on binary classification tasks, showing that synthetic, logistic regression task data can be used to train a generic reasoning module capable of completing this class of tasks. Extensions to multi-class classification tasks are possible either using a one-vs-all approach or by training \alg's reasoning module using multi-class synthetic data.
In future work, we seek to understand whether synthetic tasks exist for training other generic reasoning modules, capable of improving base LLM performance on tasks such as generation or summarization.
%
%
%
%

\subsection*{Acknowledgements}
We are grateful to Simran Arora, Rishi Bommasani, Niladri Chatterji, Arjun Desai, Sabri Eyuboglu, Neha Gupta, Karan Goel, Erik Jones, Ananya Kumar, Cassidy Laidlaw, Megan Leszczynski, Piero Molino, Laurel Orr, Michael Poli, Dimitris Tsipras, Michael Wornow, Ce Zhang, and Michael Zhang for their helpful comments and feedback, and discussions which helped shape this project.

We gratefully acknowledge the support of NIH under No. U54EB020405 (Mobilize), NSF under Nos. CCF1763315 (Beyond Sparsity), CCF1563078 (Volume to Velocity), and 1937301 (RTML); US DEVCOM ARL under No. W911NF-21-2-0251 (Interactive Human-AI Teaming); ONR under No. N000141712266 (Unifying Weak Supervision); ONR N00014-20-1-2480: Understanding and Applying Non-Euclidean Geometry in Machine Learning; N000142012275 (NEPTUNE); NXP, Xilinx, LETI-CEA, Intel, IBM, Microsoft, NEC, Toshiba, TSMC, ARM, Hitachi, BASF, Accenture, Ericsson, Qualcomm, Analog Devices, Google Cloud, Salesforce, Total, the HAI-GCP Cloud Credits for Research program,  the Stanford Data Science Initiative (SDSI), and members of the Stanford DAWN project: Facebook, Google, and VMWare. CDS was supported by a NSF CAREER (award 2046760). 

The U.S. Government is authorized to reproduce and distribute reprints for Governmental purposes notwithstanding any copyright notation thereon.
Any opinions, findings, and conclusions or recommendations expressed in this material are those of the authors and do not necessarily reflect the views, policies, or endorsements, either expressed or implied, of NIH, ONR, or the U.S. Government.

\newpage 
\printbibliography

@article{kocon2023chatgpt,
  title={Chatgpt: Jack of all trades, master of none},
  author={Koco{\'n}, Jan and Cichecki, Igor and Kaszyca, Oliwier and Kochanek, Mateusz and Szyd{\l}o, Dominika and Baran, Joanna and Bielaniewicz, Julita and Gruza, Marcin and Janz, Arkadiusz and Kanclerz, Kamil and others},
  journal={arXiv preprint arXiv:2302.10724},
  year={2023}
}

@article{liang2022holistic,
  title={Holistic evaluation of language models},
  author={Liang, Percy and Bommasani, Rishi and Lee, Tony and Tsipras, Dimitris and Soylu, Dilara and Yasunaga, Michihiro and Zhang, Yian and Narayanan, Deepak and Wu, Yuhuai and Kumar, Ananya and others},
  journal={arXiv preprint arXiv:2211.09110},
  year={2022}
}

@article{poli2023hyena,
  title={Hyena hierarchy: Towards larger convolutional language models},
  author={Poli, Michael and Massaroli, Stefano and Nguyen, Eric and Fu, Daniel Y and Dao, Tri and Baccus, Stephen and Bengio, Yoshua and Ermon, Stefano and R{\'e}, Christopher},
  journal={arXiv preprint arXiv:2302.10866},
  year={2023}
}

@article{brown2020language,
  title={Language models are few-shot learners},
  author={Brown, Tom and Mann, Benjamin and Ryder, Nick and Subbiah, Melanie and Kaplan, Jared D and Dhariwal, Prafulla and Neelakantan, Arvind and Shyam, Pranav and Sastry, Girish and Askell, Amanda and others},
  journal={Advances in neural information processing systems},
  volume={33},
  pages={1877--1901},
  year={2020}
}

@article{bommasani2021opportunities,
  title={On the opportunities and risks of foundation models},
  author={Bommasani, Rishi and Hudson, Drew A and Adeli, Ehsan and Altman, Russ and Arora, Simran and von Arx, Sydney and Bernstein, Michael S and Bohg, Jeannette and Bosselut, Antoine and Brunskill, Emma and others},
  journal={arXiv preprint arXiv:2108.07258},
  year={2021}
}

@article{
wei2022emergent,
title={Emergent Abilities of Large Language Models},
author={Jason Wei and Yi Tay and Rishi Bommasani and Colin Raffel and Barret Zoph and Sebastian Borgeaud and Dani Yogatama and Maarten Bosma and Denny Zhou and Donald Metzler and Ed H. Chi and Tatsunori Hashimoto and Oriol Vinyals and Percy Liang and Jeff Dean and William Fedus},
journal={Transactions on Machine Learning Research},
issn={2835-8856},
year={2022},
%url={https://openreview.net/forum?id=yzkSU5zdwD},
note={Survey Certification}
}

@inproceedings{alexraft,
  title={RAFT: A Real-World Few-Shot Text Classification Benchmark},
  author={Alex, Neel and Lifland, Eli and Tunstall, Lewis and Thakur, Abhishek and Maham, Pegah and Riedel, C Jess and Hine, Emmie and Ashurst, Carolyn and Sedille, Paul and Carlier, Alexis and others},
  booktitle={Thirty-fifth Conference on Neural Information Processing Systems Datasets and Benchmarks Track (Round 2)},
  year={2021}
}

@inproceedings{arora2022ask,
title={Ask Me Anything: A simple strategy for prompting language models},
  author={Arora, Simran and Narayan, Avanika and Chen, Mayee F and Orr, Laurel J and Guha, Neel and Bhatia, Kush and Chami, Ines and Sala, Frederic and R{\'e}, Christopher},
year	= {2023},
booktitle	= {ICLR 2023}
}

@article{wang2022self,
  title={Self-consistency improves chain of thought reasoning in language models},
  author={Wang, Xuezhi and Wei, Jason and Schuurmans, Dale and Le, Quoc and Chi, Ed and Zhou, Denny},
  journal={arXiv preprint arXiv:2203.11171},
  year={2022}
}

@article{wang2022rationale,
  title={Rationale-augmented ensembles in language models},
  author={Wang, Xuezhi and Wei, Jason and Schuurmans, Dale and Le, Quoc and Chi, Ed and Zhou, Denny},
  journal={arXiv preprint arXiv:2207.00747},
  year={2022}
}

@article{diao2023active,
  title={Active Prompting with Chain-of-Thought for Large Language Models},
  author={Diao, Shizhe and Wang, Pengcheng and Lin, Yong and Zhang, Tong},
  journal={arXiv preprint arXiv:2302.12246},
  year={2023}
}

@inproceedings{liu2022makes,
  title={What Makes Good In-Context Examples for GPT-3?},
  author={Liu, Jiachang and Shen, Dinghan and Zhang, Yizhe and Dolan, William B and Carin, Lawrence and Chen, Weizhu},
  booktitle={Proceedings of Deep Learning Inside Out (DeeLIO 2022): The 3rd Workshop on Knowledge Extraction and Integration for Deep Learning Architectures},
  pages={100--114},
  year={2022}
}

@inproceedings{zelikmanstar,
  title={STaR: Bootstrapping Reasoning With Reasoning},
  author={Zelikman, Eric and Wu, Yuhuai and Mu, Jesse and Goodman, Noah},
  booktitle={Advances in Neural Information Processing Systems},
  year={2022}
}

@article{wei2022chain,
  title={Chain of thought prompting elicits reasoning in large language models},
  author={Wei, Jason and Wang, Xuezhi and Schuurmans, Dale and Bosma, Maarten and Chi, Ed and Le, Quoc and Zhou, Denny},
  journal={arXiv preprint arXiv:2201.11903},
  year={2022}
}

@article{garg2022can,
  title={What can transformers learn in-context? a case study of simple function classes},
  author={Garg, Shivam and Tsipras, Dimitris and Liang, Percy S and Valiant, Gregory},
  journal={Advances in Neural Information Processing Systems},
  volume={35},
  pages={30583--30598},
  year={2022}
}

@article{von2022transformers,
  title={Transformers learn in-context by gradient descent},
  author={von Oswald, Johannes and Niklasson, Eyvind and Randazzo, Ettore and Sacramento, Jo{\~a}o and Mordvintsev, Alexander and Zhmoginov, Andrey and Vladymyrov, Max},
  journal={arXiv preprint arXiv:2212.07677},
  year={2022}
}

@inproceedings{socher2013recursive,
  title={Recursive deep models for semantic compositionality over a sentiment treebank},
  author={Socher, Richard and Perelygin, Alex and Wu, Jean and Chuang, Jason and Manning, Christopher D and Ng, Andrew Y and Potts, Christopher},
  booktitle={Proceedings of the 2013 conference on empirical methods in natural language processing},
  pages={1631--1642},
  year={2013}
}

@inproceedings{Pang+Lee+Vaithyanathan:02a,
   author = {Bo Pang and Lillian Lee and Shivakumar Vaithyanathan},
   title = {Thumbs Up? Sentiment Classification Using Machine Learning Techniques},
   year = {2002},
   pages = {79--86},
   booktitle = {Proceedings of EMNLP}
}

@article{zhang2015character,
  title={Character-level convolutional networks for text classification},
  author={Zhang, Xiang and Zhao, Junbo and LeCun, Yann},
  journal={Advances in neural information processing systems},
  volume={28},
  year={2015}
}

@article{biderman2023pythia,
  title={Pythia: A suite for analyzing large language models across training and scaling},
  author={Biderman, Stella and Schoelkopf, Hailey and Anthony, Quentin and Bradley, Herbie and O'Brien, Kyle and Hallahan, Eric and Khan, Mohammad Aflah and Purohit, Shivanshu and Prashanth, USVSN Sai and Raff, Edward and others},
  journal={arXiv preprint arXiv:2304.01373},
  year={2023}
}

@misc{gpt-neo,
  author       = {Black, Sid and
                  Gao, Leo and
                  Wang, Phil and
                  Leahy, Connor and
                  Biderman, Stella},
  title        = {{GPT-Neo: Large Scale Autoregressive Language 
                   Modeling with Mesh-Tensorflow}},
  month        = mar,
  year         = 2021,
  publisher    = {Zenodo},
  %url          = {https://doi.org/10.5281/zenodo.5297715}
}

@inproceedings{lester2021power,
  title={The Power of Scale for Parameter-Efficient Prompt Tuning},
  author={Lester, Brian and Al-Rfou, Rami and Constant, Noah},
  booktitle={Proceedings of the 2021 Conference on Empirical Methods in Natural Language Processing},
  pages={3045--3059},
  year={2021}
}

@inproceedings{liu-etal-2022-p,
    title = "{P}-Tuning: Prompt Tuning Can Be Comparable to Fine-tuning Across Scales and Tasks",
    author = "Liu, Xiao  and
      Ji, Kaixuan  and
      Fu, Yicheng  and
      Tam, Weng  and
      Du, Zhengxiao  and
      Yang, Zhilin  and
      Tang, Jie",
    booktitle = "Proceedings of the 60th Annual Meeting of the Association for Computational Linguistics",
    month = may,
    year = "2022",
    address = "Dublin, Ireland",
    publisher = "Association for Computational Linguistics",
    %url = "https://aclanthology.org/2022.acl-short.8",
}

@inproceedings{li2021prefix,
  title={Prefix-Tuning: Optimizing Continuous Prompts for Generation},
  author={Li, Xiang Lisa and Liang, Percy},
  booktitle={Proceedings of the 59th Annual Meeting of the Association for Computational Linguistics and the 11th International Joint Conference on Natural Language Processing (Volume 1: Long Papers)},
  pages={4582--4597},
  year={2021}
}

@article{scao2022bloom,
  title={Bloom: A 176b-parameter open-access multilingual language model},
  author={Scao, Teven Le and Fan, Angela and Akiki, Christopher and Pavlick, Ellie and Ili{\'c}, Suzana and Hesslow, Daniel and Castagn{\'e}, Roman and Luccioni, Alexandra Sasha and Yvon, Fran{\c{c}}ois and Gall{\'e}, Matthias and others},
  journal={arXiv preprint arXiv:2211.05100},
  year={2022}
}

@inproceedings{
hu2022lora,
title={Lo{RA}: Low-Rank Adaptation of Large Language Models},
author={Edward J Hu and yelong shen and Phillip Wallis and Zeyuan Allen-Zhu and Yuanzhi Li and Shean Wang and Lu Wang and Weizhu Chen},
booktitle={International Conference on Learning Representations},
year={2022},
%url={https://openreview.net/forum?id=nZeVKeeFYf9}
}

@inproceedings{houlsby2019parameter,
  title={Parameter-efficient transfer learning for NLP},
  author={Houlsby, Neil and Giurgiu, Andrei and Jastrzebski, Stanislaw and Morrone, Bruna and De Laroussilhe, Quentin and Gesmundo, Andrea and Attariyan, Mona and Gelly, Sylvain},
  booktitle={International Conference on Machine Learning},
  pages={2790--2799},
  year={2019},
  organization={PMLR}
}

@misc{gpt-j,
  author = {Wang, Ben and Komatsuzaki, Aran},
  title = {{GPT-J-6B: A 6 Billion Parameter Autoregressive Language Model}},
  howpublished = {\url{https://github.com/kingoflolz/mesh-transformer-jax}},
  year = 2021,
  month = May
}

@misc{wu2020visual,
      title={Visual Transformers: Token-based Image Representation and Processing for Computer Vision}, 
      author={Bichen Wu and Chenfeng Xu and Xiaoliang Dai and Alvin Wan and Peizhao Zhang and Zhicheng Yan and Masayoshi Tomizuka and Joseph Gonzalez and Kurt Keutzer and Peter Vajda},
      year={2020},
      archivePrefix={arXiv},
      primaryClass={cs.CV}
}

@article{wei2023larger,
  title={Larger language models do in-context learning differently},
  author={Wei, Jerry and Wei, Jason and Tay, Yi and Tran, Dustin and Webson, Albert and Lu, Yifeng and Chen, Xinyun and Liu, Hanxiao and Huang, Da and Zhou, Denny and others},
  journal={arXiv preprint arXiv:2303.03846},
  year={2023}
}

@article{zhang2023llamaadapter,
  title = {LLaMA-Adapter: Efficient Fine-tuning of Language Models with Zero-init Attention},
  author={Zhang, Renrui and Han, Jiaming and Zhou, Aojun and Hu, Xiangfei and Yan, Shilin and Lu, Pan and Li, Hongsheng and Gao, Peng and Qiao, Yu},
  journal={arXiv preprint arXiv:2303.16199},
  year={2023}
}

@misc{radford2022whisper,
  %url = {https://arxiv.org/abs/2212.04356},
  author = {Radford, Alec and Kim, Jong Wook and Xu, Tao and Brockman, Greg and McLeavey, Christine and Sutskever, Ilya},
  title = {Robust Speech Recognition via Large-Scale Weak Supervision},
  publisher = {arXiv},
  year = {2022},
  copyright = {arXiv.org perpetual, non-exclusive license}
}

@TECHREPORT{Krizhevsky09learningmultiple,
    author = {Alex Krizhevsky},
    title = {Learning multiple layers of features from tiny images},
    institution={University of Toronto},
    year = {2009}
}

@article{lecun2010mnist,
  title={MNIST handwritten digit database},
  author={LeCun, Yann and Cortes, Corinna and Burges, CJ},
  journal={ATT Labs [Online]. Available: http://yann.lecun.com/exdb/mnist},
  volume={2},
  year={2010}
}

@inproceedings{lu2022fantastically,
  title={Fantastically Ordered Prompts and Where to Find Them: Overcoming Few-Shot Prompt Order Sensitivity},
  author={Lu, Yao and Bartolo, Max and Moore, Alastair and Riedel, Sebastian and Stenetorp, Pontus},
  booktitle={Proceedings of the 60th Annual Meeting of the Association for Computational Linguistics (Volume 1: Long Papers)},
  pages={8086--8098},
  year={2022}
}

@inproceedings{kojima2022large,
  title={Large Language Models are Zero-Shot Reasoners},
  author={Kojima, Takeshi and Gu, Shixiang Shane and Reid, Machel and Matsuo, Yutaka and Iwasawa, Yusuke},
  year	= {2022},
  booktitle={ICML 2022 Workshop on Knowledge Retrieval and Language Models}
}

@inproceedings{
zhang2021wrench,
title={{WRENCH}: A Comprehensive Benchmark for Weak Supervision},
author={Jieyu Zhang and Yue Yu and Yinghao Li and Yujing Wang and Yaming Yang and Mao Yang and Alexander Ratner},
booktitle={Thirty-fifth Conference on Neural Information Processing Systems Datasets and Benchmarks Track},
year={2021},
%url={https://openreview.net/forum?id=Q9SKS5k8io}
}

@article{DBLP:journals/corr/abs-1903-04561,
  author    = {Daniel Borkan and
               Lucas Dixon and
               Jeffrey Sorensen and
               Nithum Thain and
               Lucy Vasserman},
  title     = {Nuanced Metrics for Measuring Unintended Bias with Real Data for Text
               Classification},
  journal   = {CoRR},
  volume    = {abs/1903.04561},
  year      = {2019},
  %url       = {http://arxiv.org/abs/1903.04561},
}

@InProceedings{maas-EtAl:2011:ACL-HLT2011,
  author    = {Maas, Andrew L.  and  Daly, Raymond E.  and  Pham, Peter T.  and  Huang, Dan  and  Ng, Andrew Y.  and  Potts, Christopher},
  title     = {Learning Word Vectors for Sentiment Analysis},
  booktitle = {Proceedings of the 49th Annual Meeting of the Association for Computational Linguistics: Human Language Technologies},
  year      = {2011},
  address   = {Portland, Oregon, USA},
  publisher = {Association for Computational Linguistics},
  pages     = {142--150},
  %url       = {http://www.aclweb.org/anthology/P11-1015}
}

@inproceedings{Almeida2011SpamFiltering, title={Contributions to the Study of SMS Spam Filtering: New Collection and Results}, author={Tiago A. Almeida and Jose Maria Gomez Hidalgo and Akebo Yamakami}, year={2011}, booktitle = "Proceedings of the 2011 ACM Symposium on Document Engineering (DOCENG'11)", }

@article{speechcommandsv2,
   author = { {Warden}, P.},
    title = "{Speech Commands: A Dataset for Limited-Vocabulary Speech Recognition}",
  journal = {ArXiv e-prints},
  primaryClass = "cs.CL",
  keywords = {Computer Science - Computation and Language, Computer Science - Human-Computer Interaction},
    year = 2018,
    month = apr,
    %url = {https://arxiv.org/abs/1804.03209},
}

@article{narayandata,
author = {Narayan, Avanika and Chami, Ines and Orr, Laurel and R\'{e}, Christopher},
title = {Can Foundation Models Wrangle Your Data?},
year = {2022},
issue_date = {December 2022},
publisher = {VLDB Endowment},
volume = {16},
number = {4},
issn = {2150-8097},
%url = {https://doi.org/10.14778/3574245.3574258},
journal = {Proc. VLDB Endow.},
numpages = {9}
}

@article{agrawal2022large,
  title={Large language models are zero-shot clinical information extractors},
  author={Agrawal, Monica and Hegselmann, Stefan and Lang, Hunter and Kim, Yoon and Sontag, David},
  journal={arXiv preprint arXiv:2205.12689},
  year={2022}
}

@misc{huyenchip:2023,
  author = {Huyen, Chip},
  title = {Prompting vs. Finetuning vs. Alternatives},
  year = {2023},
  %url = {https://huyenchip.com/2023/04/11/llm-engineering.html#prompting_vs_finetuning_vs_alternatives},
}

@article{liu2022few,
  title={Few-shot parameter-efficient fine-tuning is better and cheaper than in-context learning},
  author={Liu, Haokun and Tam, Derek and Muqeeth, Mohammed and Mohta, Jay and Huang, Tenghao and Bansal, Mohit and Raffel, Colin A},
  journal={Advances in Neural Information Processing Systems},
  volume={35},
  pages={1950--1965},
  year={2022}
}

@book{shalev2014understanding,
  title={Understanding machine learning: From theory to algorithms},
  author={Shalev-Shwartz, Shai and Ben-David, Shai},
  year={2014},
  publisher={Cambridge university press}
}

@article{radford2018improving,
  title={Improving language understanding by generative pre-training},
  author={Radford, Alec and Narasimhan, Karthik and Salimans, Tim and Sutskever, Ilya and others},
  journal={arXiv preprint},
  year={2018}
}

@book{jaynes2003probability,
  title={Probability theory: The logic of science},
  year={2003},
  publisher={Cambridge university press}
}

@article{kumar2022fine,
  title={Fine-tuning can distort pretrained features and underperform out-of-distribution},
  author={Kumar, Ananya and Raghunathan, Aditi and Jones, Robbie and Ma, Tengyu and Liang, Percy},
  journal={arXiv preprint arXiv:2202.10054},
  year={2022}
}

@article{xie2021explanation,
  title={An explanation of in-context learning as implicit bayesian inference},
  author={Xie, Sang Michael and Raghunathan, Aditi and Liang, Percy and Ma, Tengyu},
  journal={arXiv preprint arXiv:2111.02080},
  year={2021}
}

\newpage
\appendix
\section{Fine-tuning model with NL-based Probabilistic Inference Tasks}\label{app:ft-nl}
As highlighted in Section~\ref{sec:intro}, we describe the details for \emph{directly} fine-tuning an LLM on synthetically generated probabilistic inference tasks to improve reasoning capabilities. For the following experiments, we use \gptsmall as the base model.

\begin{figure*}
  \centering
    \begin{subfigure}[b]{0.50\textwidth}
        \centering
\includegraphics[width=0.9\textwidth]{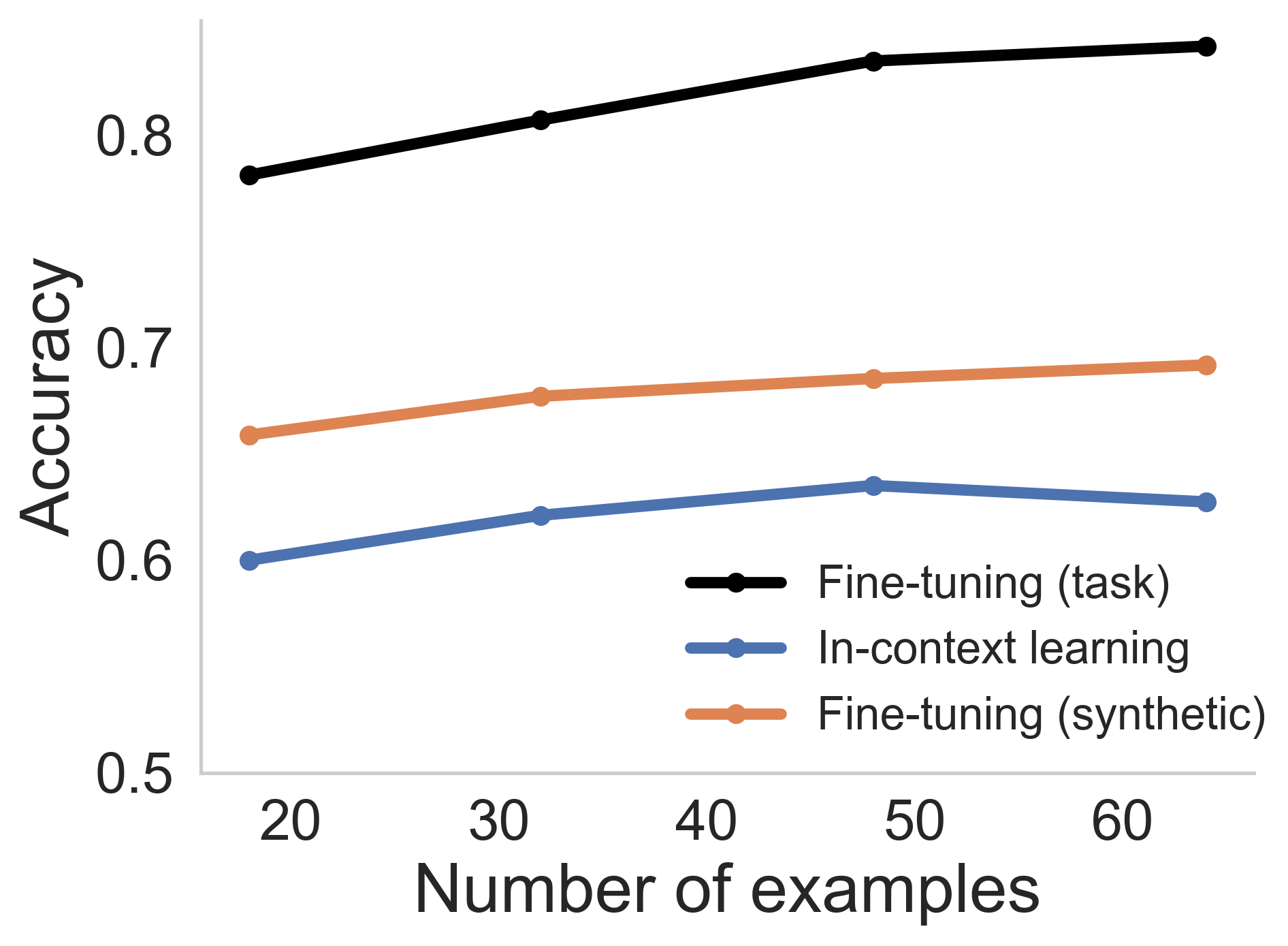}
    \subcaption{}
    \label{fig:ft-nl-k}
    \end{subfigure}
    \begin{subfigure}[b]{0.43\textwidth}
        \centering
\includegraphics[width=1\textwidth]{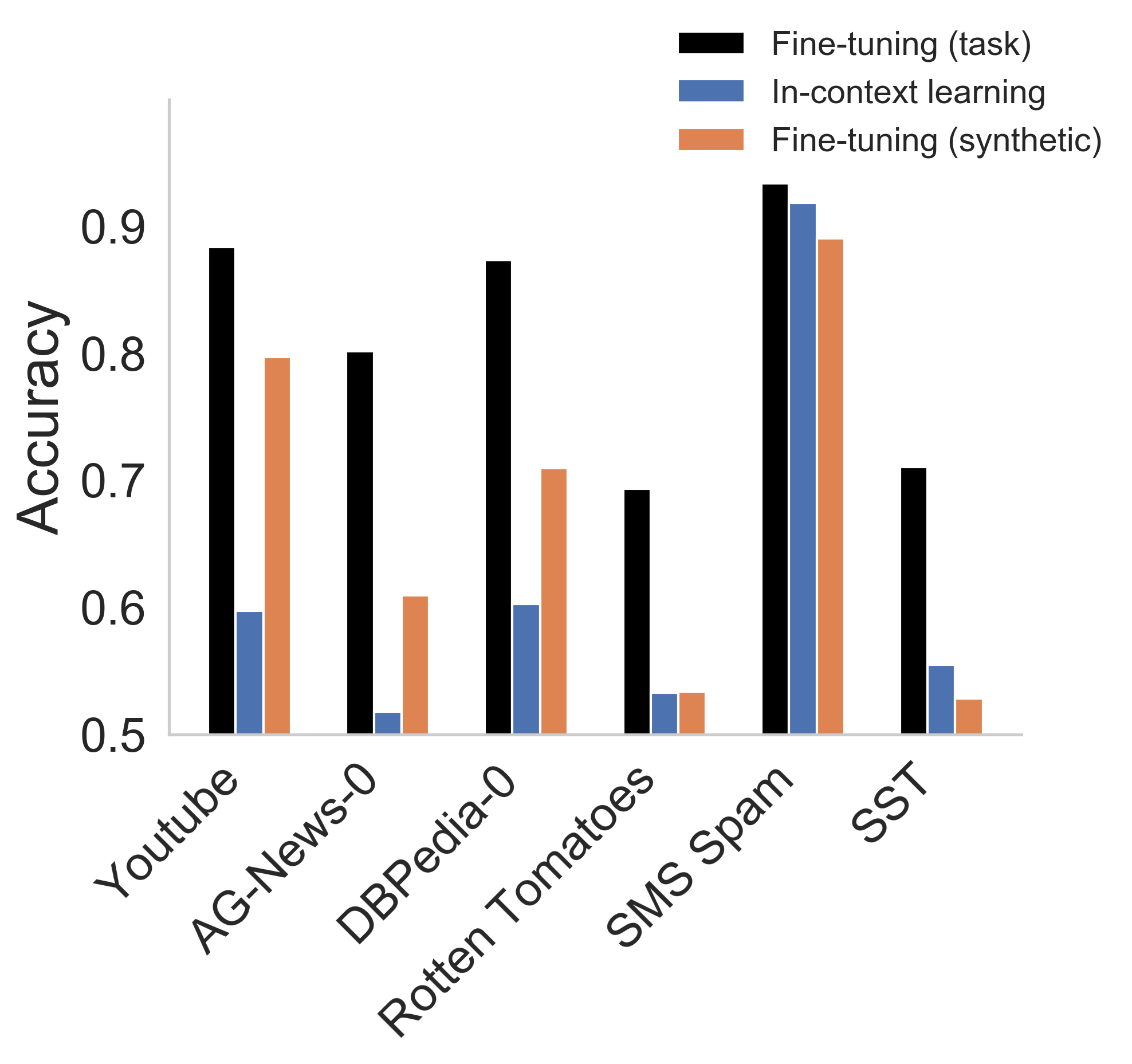}
    \subcaption{}
    \label{fig:ft-nl-dataset}
    \end{subfigure}
\caption{\textbf{Fine-tuning with NL synthetic task}. (Left) Averaged over 6 different tasks, fine-tuning with the NL synthetic task provides a lift over base in-context learning, and scales with number of examples. (Right) Dataset level comparisons between task-specific fine-tuning,  in-context learning and synthetic fine-tuning: synthetic fine-tuning outperforms base in-context learning on 4 out of 6 datasets, but lags task-specific tuning.}
\label{fig:ft-nl-app}
\end{figure*}

\subsection{Training Task}
We fine-tune the base model using a sequence of $\klr$ pairs of synthetically generated labeled natural language examples $(\x, \y)$. Each example $\x$ in the sequence $s = (\x_1, \y_1), \ldots, (\x_k, \y_k)$ consists of a list of strings constructed from a fixed  $V$ size of dimension $d = 30$ .
We use the following fixed vocabulary: [
    ``sports'',
    ``love'',
    ``hate'',
    ``car'',
    ``school'',
    ``family'',
    ``work'',
    ``sleep'',
    ``water'',
    ``tree'',
    ``fox'',
    ``train'',
    ``random'',
    ``movie'',
    ``music'',
    ``book'',
    ``play'',
    ``house'',
    ``spell'',
    ``bar'',
    ``jump'',
    ``park'',
    ``run'',
    ``hill'',
    ``fast'',
    ``slow'',
    ``talk'',
    ``wallet'',
    ``orange'',
    ``apple'',
    ``ball'',
    ``cat''
].

To generate a particular example $\x_i$, we sample each coordinate  $\x_{i,j}$ uniformly from the set $\{-1, +1 \}$. If the sampled value is $+1$, we set the value to be the corresponding word in the vocabulary, that is, $\x_{i,j} = V_j$. Otherwise, the word $\x_{i,j}$ is set to ``null''. For a given sequence $s$, we generate each of the labels $\{y_i\}$ as:
%
\begin{equation}
\param_t \sim \normal(0, I_\dx), \quad \y_{i} \sim \sigmoid(\weight\inner{\x_{i}}{\param}), \text{for }i \in [\klr]\;,
\end{equation}
where we set noise parameter $\weight = 5$. If the sampled output is 0, we set the $y_{i}$ to ``negative'' and ``positive'' otherwise. 

Finally, the inputs are formatted with following template: ``$x_1$ : $y_1$ , $x_2$ : $y_2$ , ... , $x_k$ : $y_k$'' and the model is trained using gradient descent on the loss
\begin{equation}
\loss(\Tr_\paramT) \defn \En_{\x, \y}\left[\frac{1}{k} \sum_{i=1}^k \lossce (\Tr_\paramT(z_{1:i-1}, x_i), \y_i)\right]\;,
\end{equation}
where $z_{1:i-1}$ corresponds to the first $i-1$ examples and $\lossce$ is the cross-entropy loss evaluated on the transformer prediction and the true $y_i$.

More concretely, a sample input sample sequence $s$ to be used for training looks like:
\begin{verbatim}
"sports love null car ... cat: positive, 
null love null car ... null: negative, 
... 
sports null hat null ... cat : positive"
\end{verbatim}

\subsection{Training Parameters}
We train \gptsmall on this synthetic task with a learning rate of 0.0001 and a batch size of 4. For each sequence we sampled a total of $k=60$ examples and trained the model for 10000 steps.

\subsection{Evaluation}
We evaluate on 6 datasets: AG News~\citep{zhang2015character}, DBPedia~\citep{zhang2015character}, SST~\citep{socher2013recursive}, SMS Spam~\citep{Almeida2011SpamFiltering}, Youtube~\citep{zhang2021wrench} and Rotten Tomatoes~\citep{Pang+Lee+Vaithyanathan:02a}. We truncate the input texts to 100 characters to fit more in-context examples. We evaluate over a range of context sizes ($k$=[18, 32, 48, 60]). At evaluation time, we use the same ``sentence : label'' format that was used to train the model. We evaluate over 3 random seeds. In Figure~\ref{fig:ft-nl-app}, we compare the performance of the model fine-tuned on probabilistic inference tasks and the base in-context learning performance. While the performance of the fine-tuned model is better than the base in-context learning capabilities, task-specific fine-tuning still outperforms it by an average of $16.87\%$ (see Figure~\ref{fig:ft-nl-app}).

\section{Details for Representation-Reasoning decomposition evaluations}\label{app:sec-2}
In this section, we provide details for the experimental evaluation and additional results for the representation-reasoning decomposition introduced in Section~\ref{sec:rep-reas}. 

\subsection{Experimental setup}
For these experiments, we evaluate three different language models: \gptsmall, \pythiasmall, and \bloomsmall on a collection of 6 binary classification datasets: AG News~\citep{zhang2015character}, DBPedia~\citep{zhang2015character}, SST~\citep{socher2013recursive}, SMS Spam~\citep{Almeida2011SpamFiltering}, Youtube~\citep{zhang2021wrench} and Rotten Tomatoes~\citep{Pang+Lee+Vaithyanathan:02a}. For each model, we run evaluations for three different random seeds, where the randomness was in the set of datapoints chosen for the training task.  For the hyperparameters, we performed an extensive search for all models across datasets. For details on these hyperparameters and the adapter architecture we evaluate over, see Appendix~\ref{app:exp-setup}. 

To conduct linear probing over the embeddings, we perform logistic regression over the output embeddings of each model and the given labels in the training set using the built-in logistic regression solver from the scikit-learn python library, utilizing the \emph{lbgfs} solver.

\subsection{Detailed results}
For each class of methods in the task-adaptation taxonomy from Section~\ref{sec:taxonomy}, we now describe the details of the experimental evaluation and present additional results.

\begin{figure*}[t!]
    \centering
    \begin{subfigure}[b]{0.31\textwidth}
        \centering
\includegraphics[width=1\textwidth]{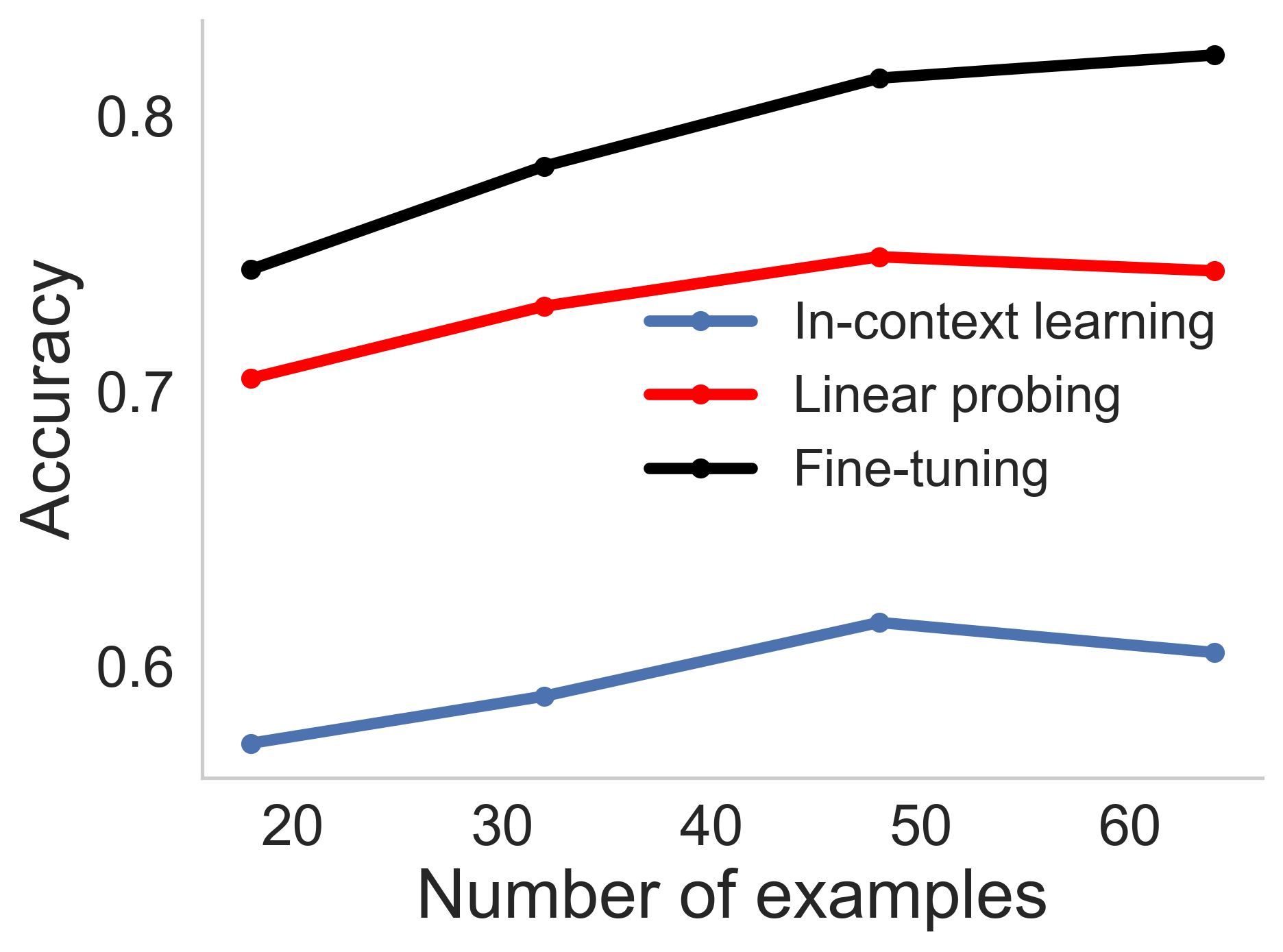}
    \subcaption{\gptsmall}
    \end{subfigure}
    \begin{subfigure}[b]{0.31\textwidth}
        \centering
\includegraphics[width=1\textwidth]{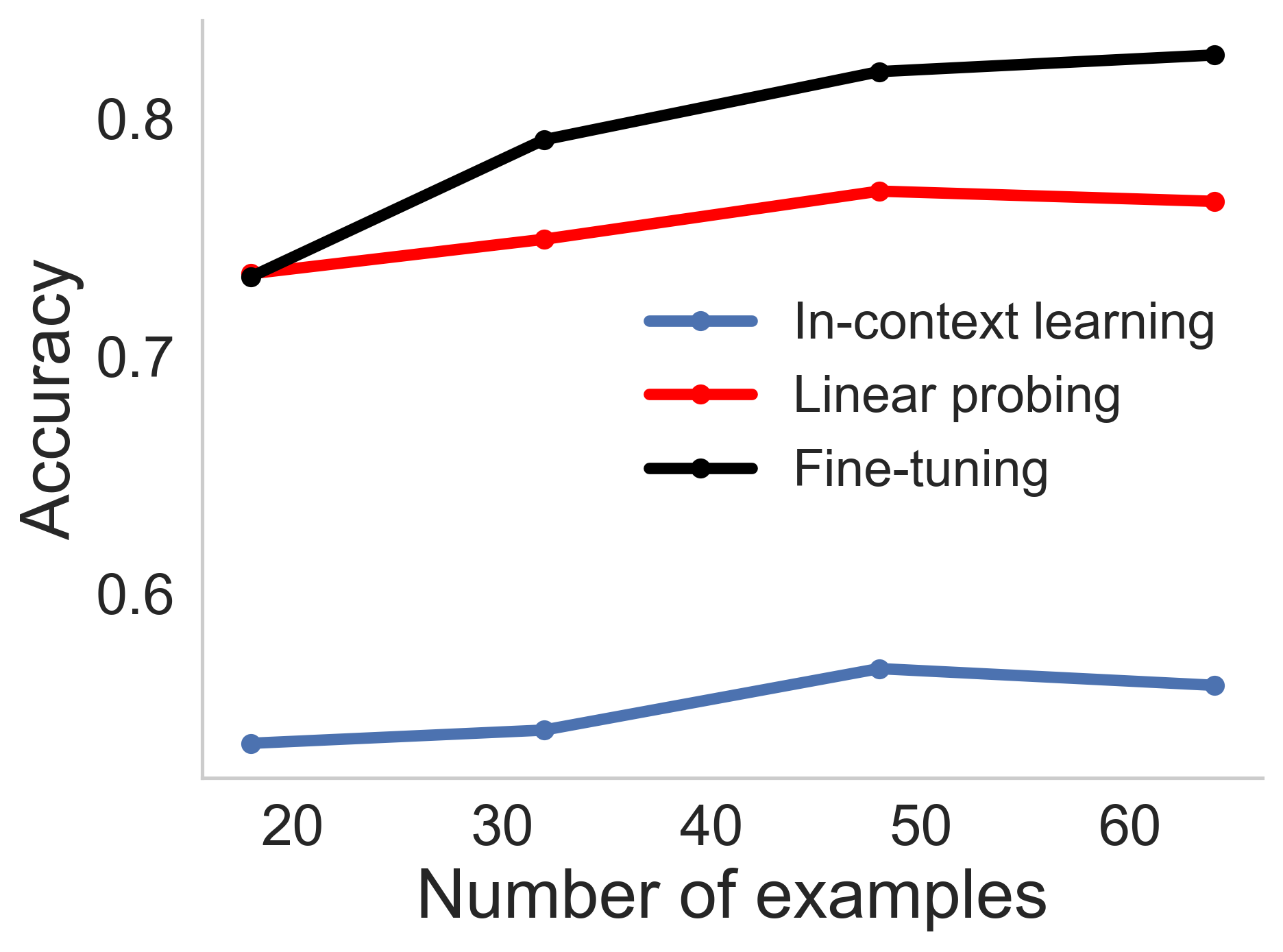}
    \subcaption{\pythiasmall}
    \end{subfigure}
    \begin{subfigure}[b]{0.31\textwidth}
        \centering
\includegraphics[width=1\textwidth]{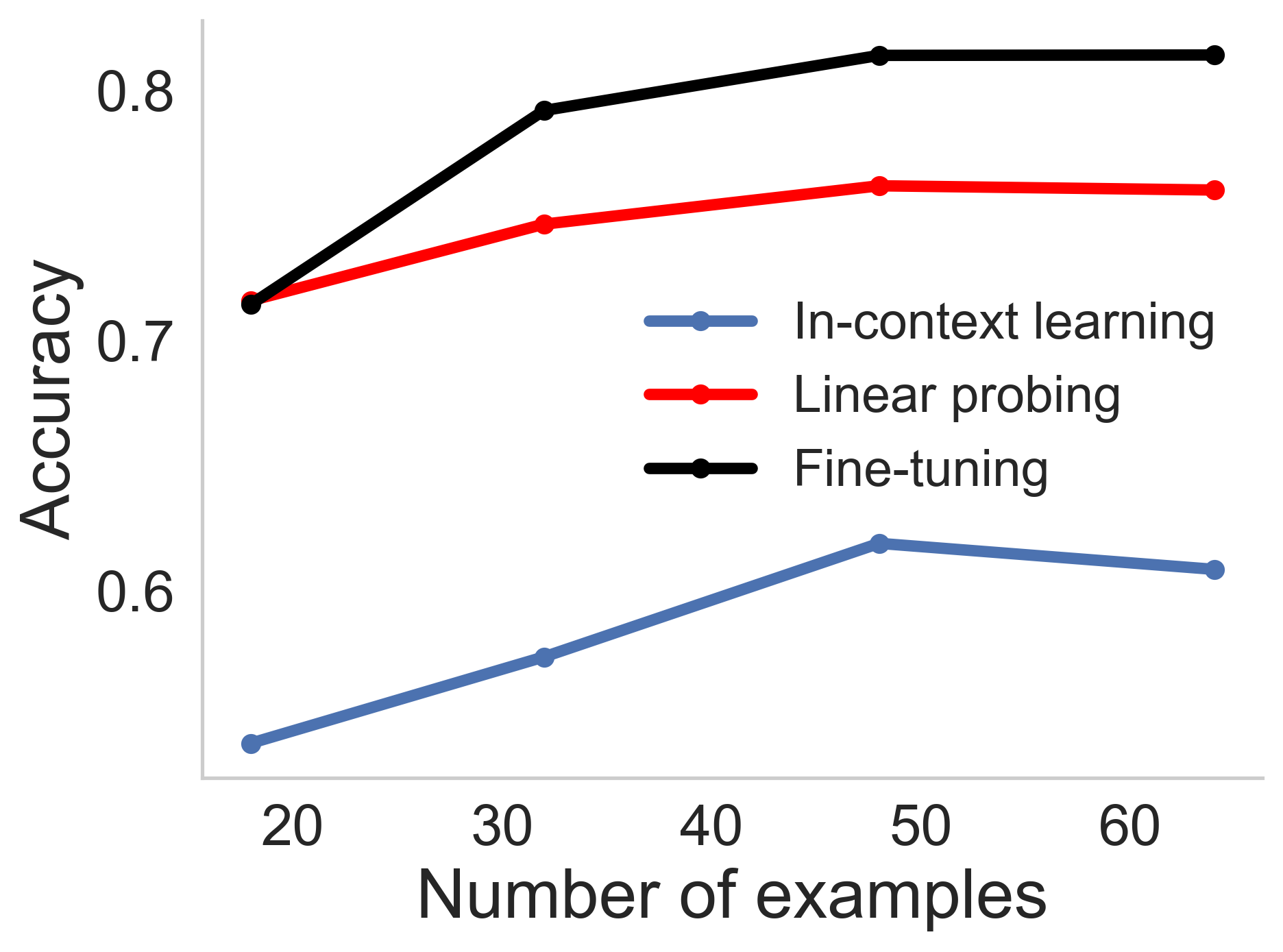}
    \subcaption{\bloomsmall}
    \end{subfigure}
    \caption{\textbf{Comparison of linear probing, in-context learning, and fine-tuning}. Accuracy of in-context learning vs. linear probing on model embeddings across three model families: representations have sufficient information. }
    \label{fig:icl-ft-lr}
\end{figure*}

\begin{figure*}[t!]
    \centering
    \begin{subfigure}[b]{1\textwidth}
        \centering
\includegraphics[width=1\textwidth]{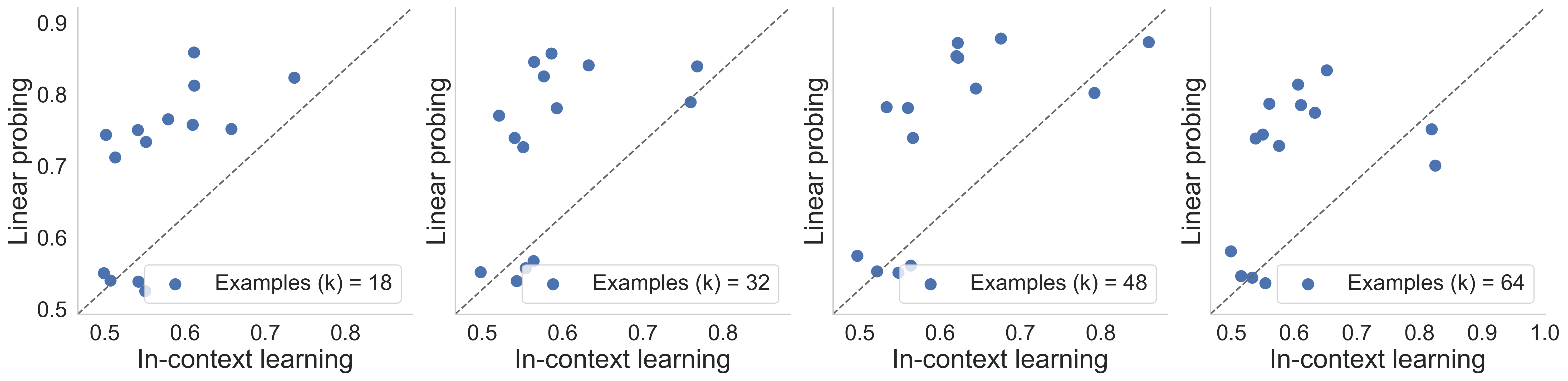}
    \subcaption{\gptsmall}
    \end{subfigure}
    \begin{subfigure}[b]{1\textwidth}
        \centering
\includegraphics[width=1\textwidth]{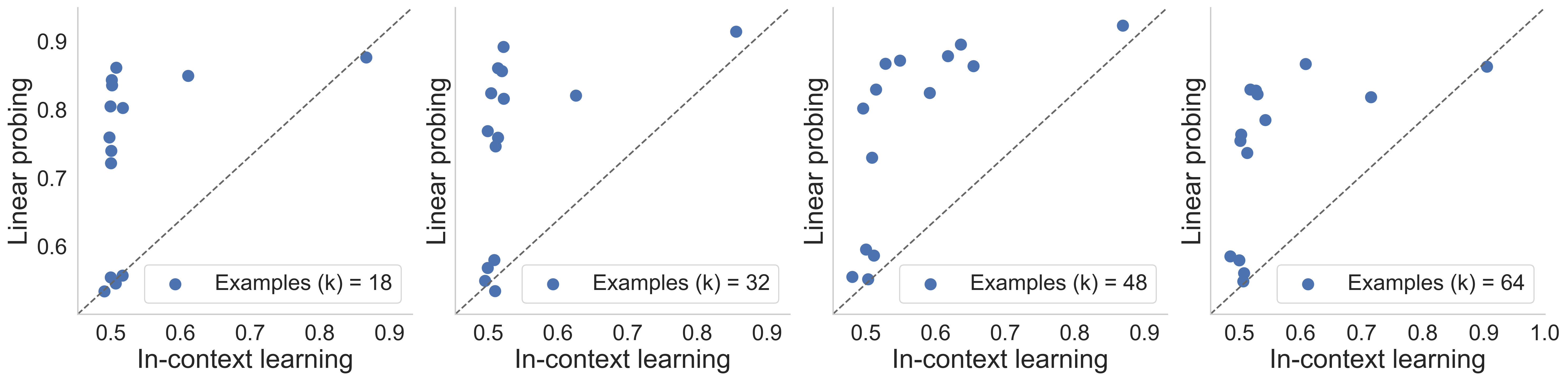}
    \subcaption{\pythiasmall}
    \end{subfigure}
    \begin{subfigure}[b]{1\textwidth}
        \centering
\includegraphics[width=1\textwidth]{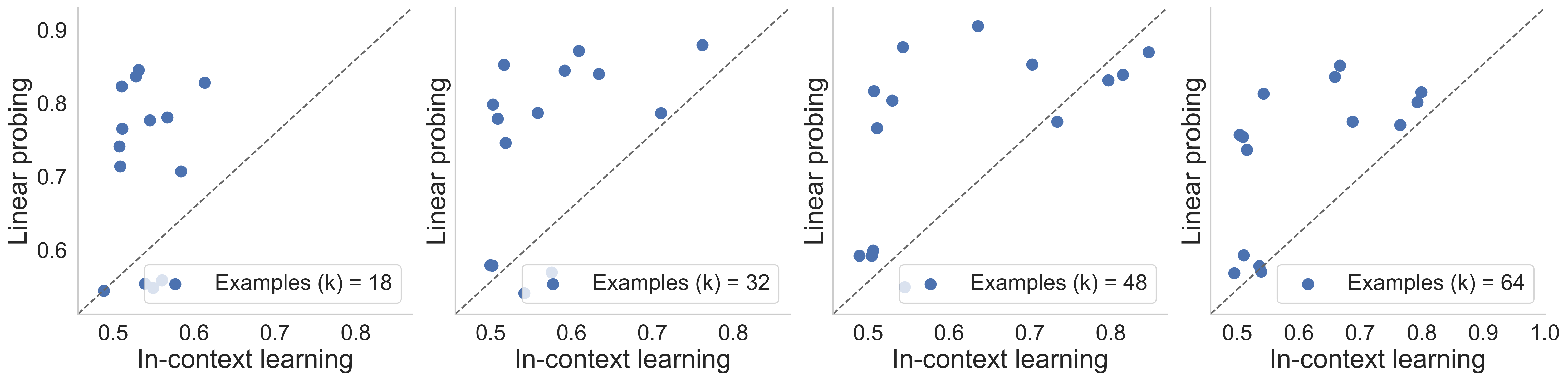}
    \subcaption{\bloomsmall}
    \end{subfigure}
    \caption{\textbf{Linear probing vs. in-context learning}. Scatter plot of accuracy of in-context learning vs. linear probing on model embeddings across model families and different number of in-context examples: linear probing consistently outperforms in-context learning indicating that the learned representations have sufficient information. Each point in the plot represents a dataset.}
    \label{fig:icl-lr}
\end{figure*}

\paragraph{In-context learning.} To understand the representation and reasoning gaps for in-context learning, we evaluated three accuracies: a) using in-context learning with base models, b) fine-tuning the model for the task, and c) linear probing the model specifically for the task. The gap due to representation was taken to be the difference between the fine-tuning and linear probing accuracies while the reasoning gap was the gap between linear probing and in-context learning, as described in eq.~\eqref{eq:rep-decomp}. 

In Figure~\ref{fig:icl-ft-lr}, we show the average accuracies of in-context learning, linear probing, and fine-tuning across the 6 tasks. Linear probing closes the gap between in-context learning and fine-tuning, while being task-specific. In Figure~\ref{fig:icl-lr}, we show a scatter plot of the accuracies of linear probing vs. the accuracies of base in-context learning. Linear probing consistently out performs in-context learning showing that the learned representations across these models have sufficient information to complete the tasks but lack reasoning abilities.

\begin{figure*}[t!]
    \centering
    \begin{subfigure}[b]{0.31\textwidth}
        \centering
\includegraphics[width=0.9\textwidth]{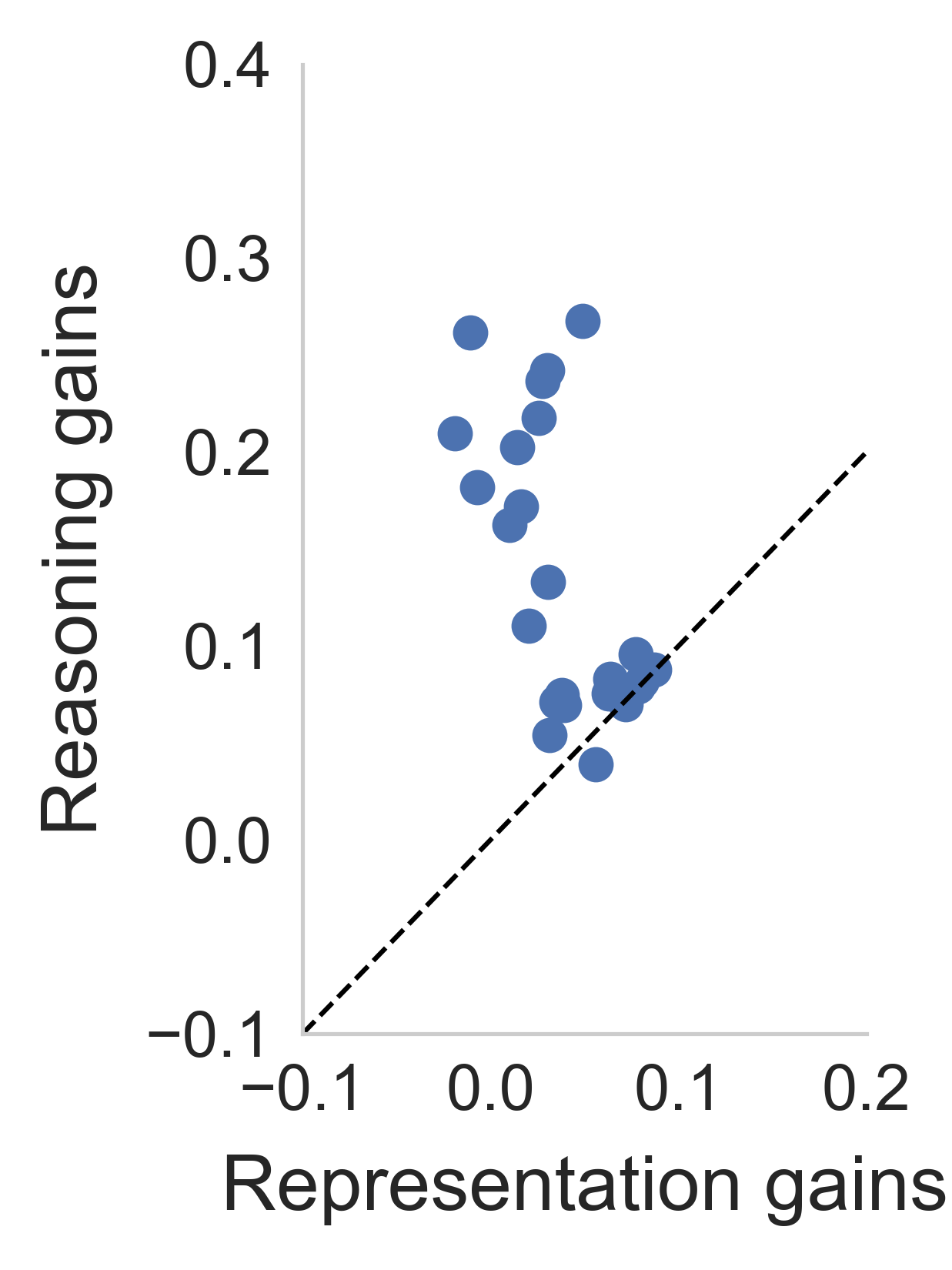}
    \subcaption{\gptsmall}
    \end{subfigure}
    \begin{subfigure}[b]{0.31\textwidth}
        \centering
\includegraphics[width=0.9\textwidth]{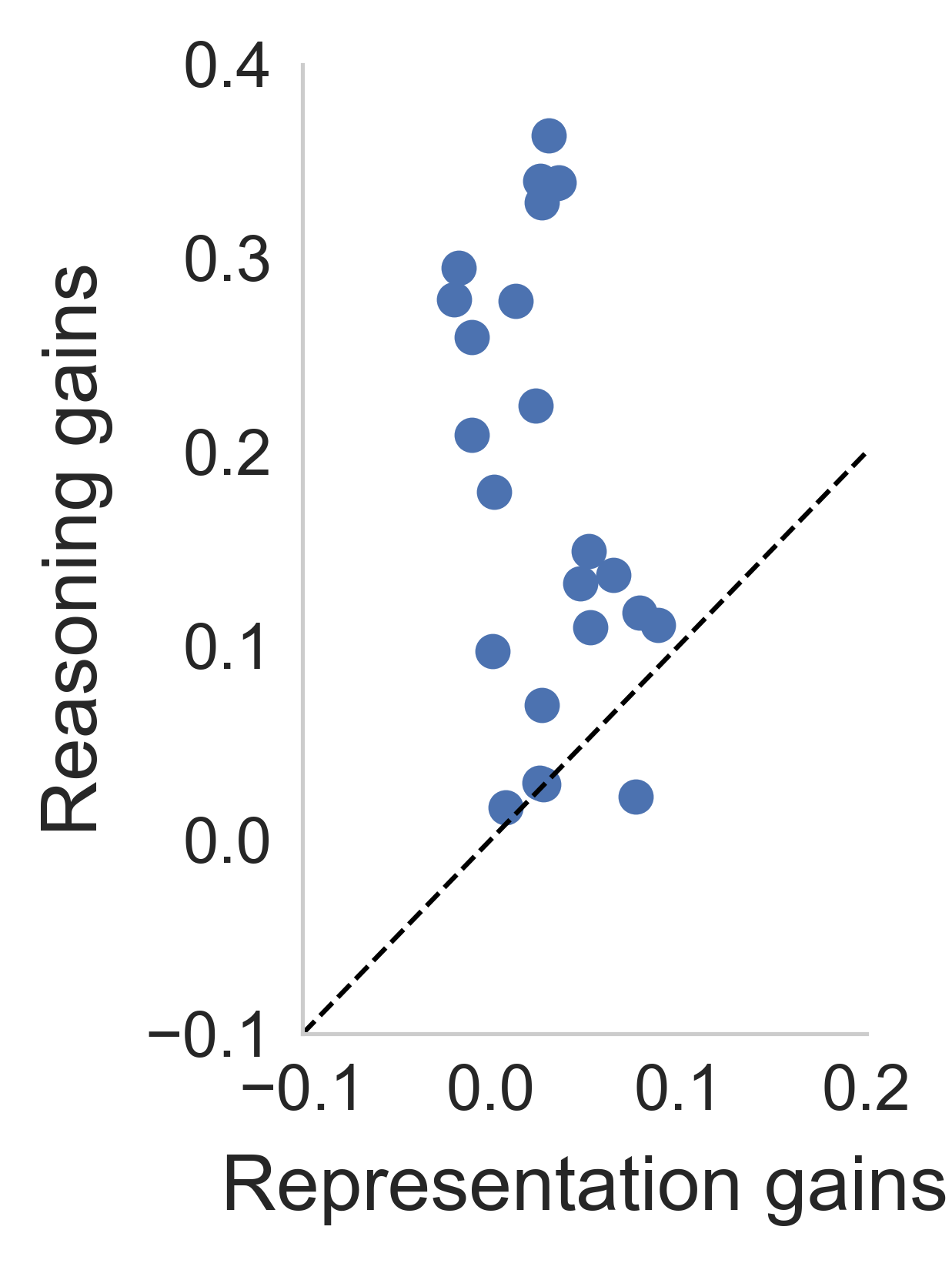}
    \subcaption{\pythiasmall}
    \end{subfigure}
    \begin{subfigure}[b]{0.31\textwidth}
        \centering
\includegraphics[width=0.9\textwidth]{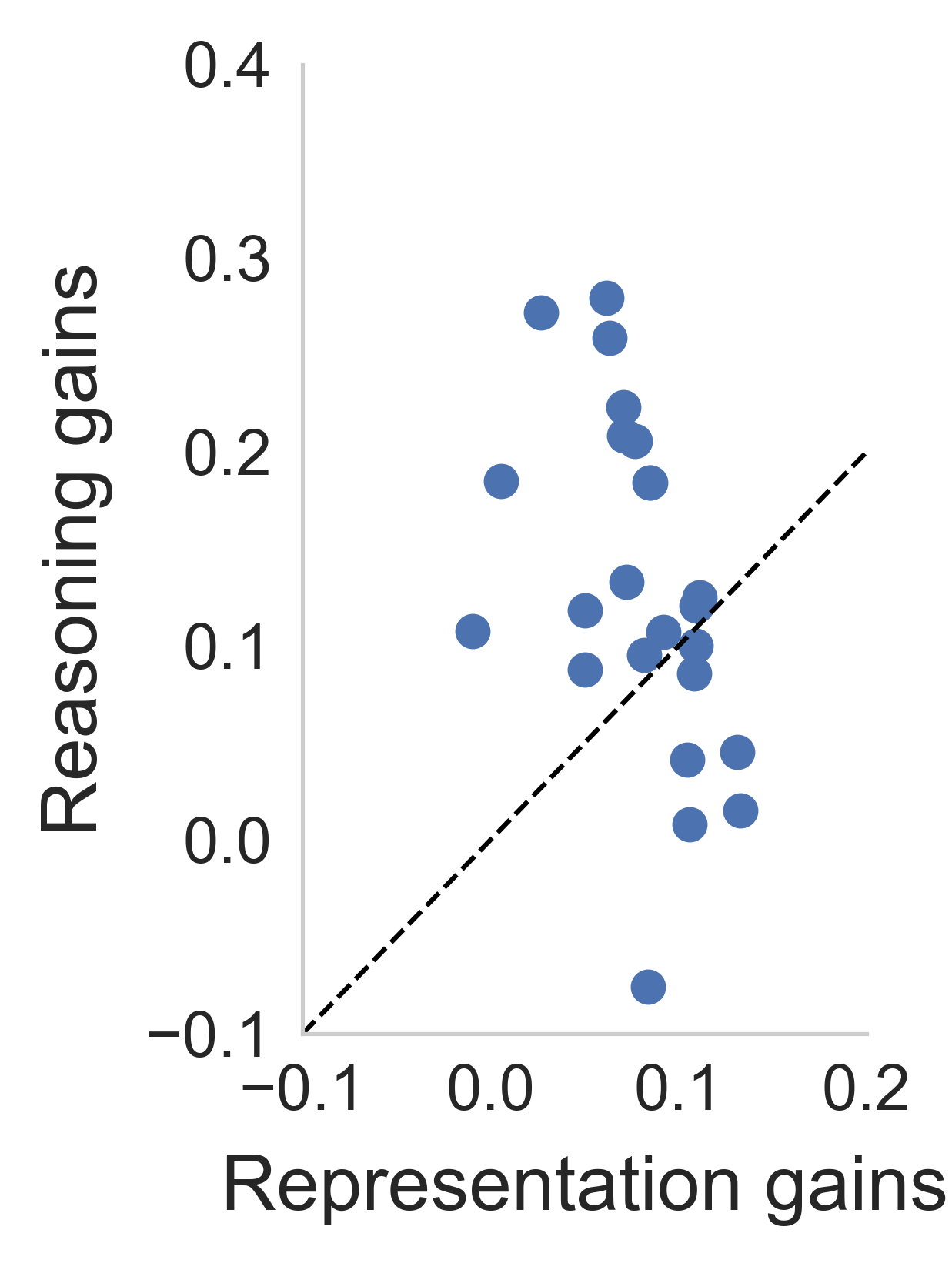}
    \subcaption{\bloomsmall}
    \end{subfigure}
    \caption{\textbf{Effects of fine-tuning on reasoning}. Across datasets (each point in plot represents a dataset) and model families, fine-tuning improves task-specific reasoning which improves it performance over base in-context learning.}
    \label{fig:ft-rep-reas}
\end{figure*}

\begin{figure*}[t!]
    \centering
    \begin{subfigure}[b]{0.31\textwidth}
        \centering
\includegraphics[width=1\textwidth]{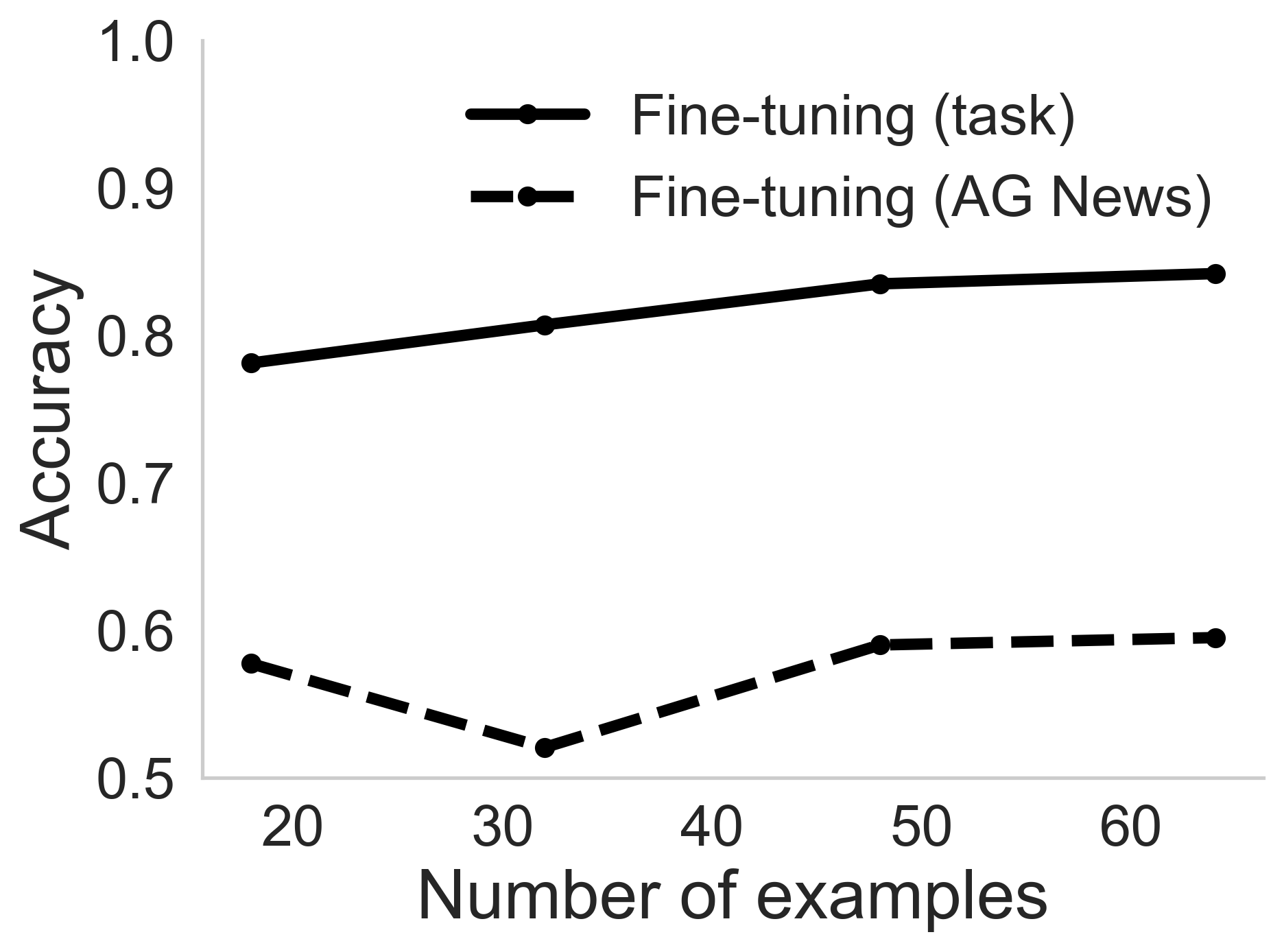}
    \subcaption{\gptsmall}
    \end{subfigure}
    \begin{subfigure}[b]{0.31\textwidth}
        \centering
\includegraphics[width=1\textwidth]{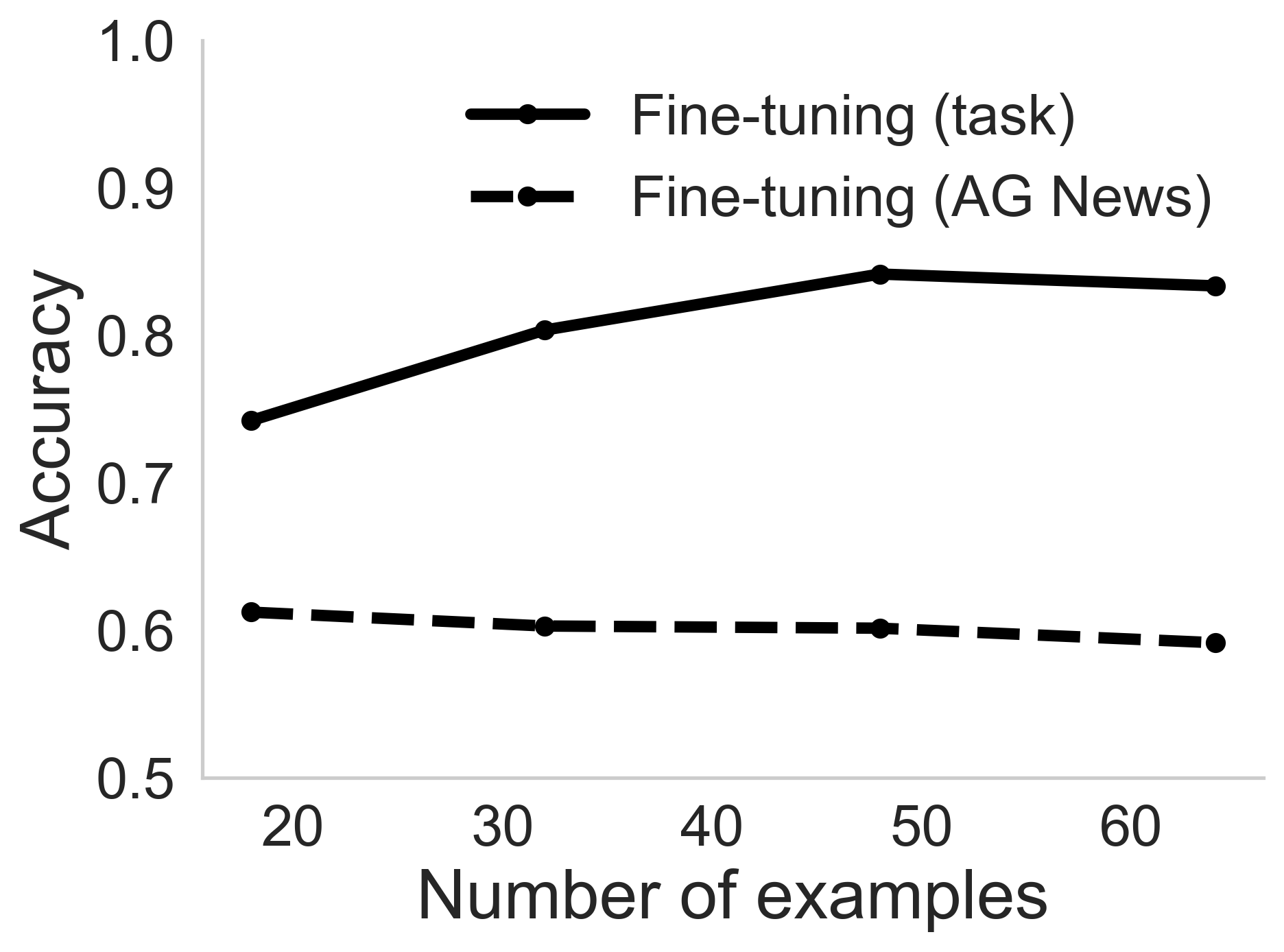}
    \subcaption{\pythiasmall}
    \end{subfigure}
    \begin{subfigure}[b]{0.31\textwidth}
        \centering
\includegraphics[width=1\textwidth]{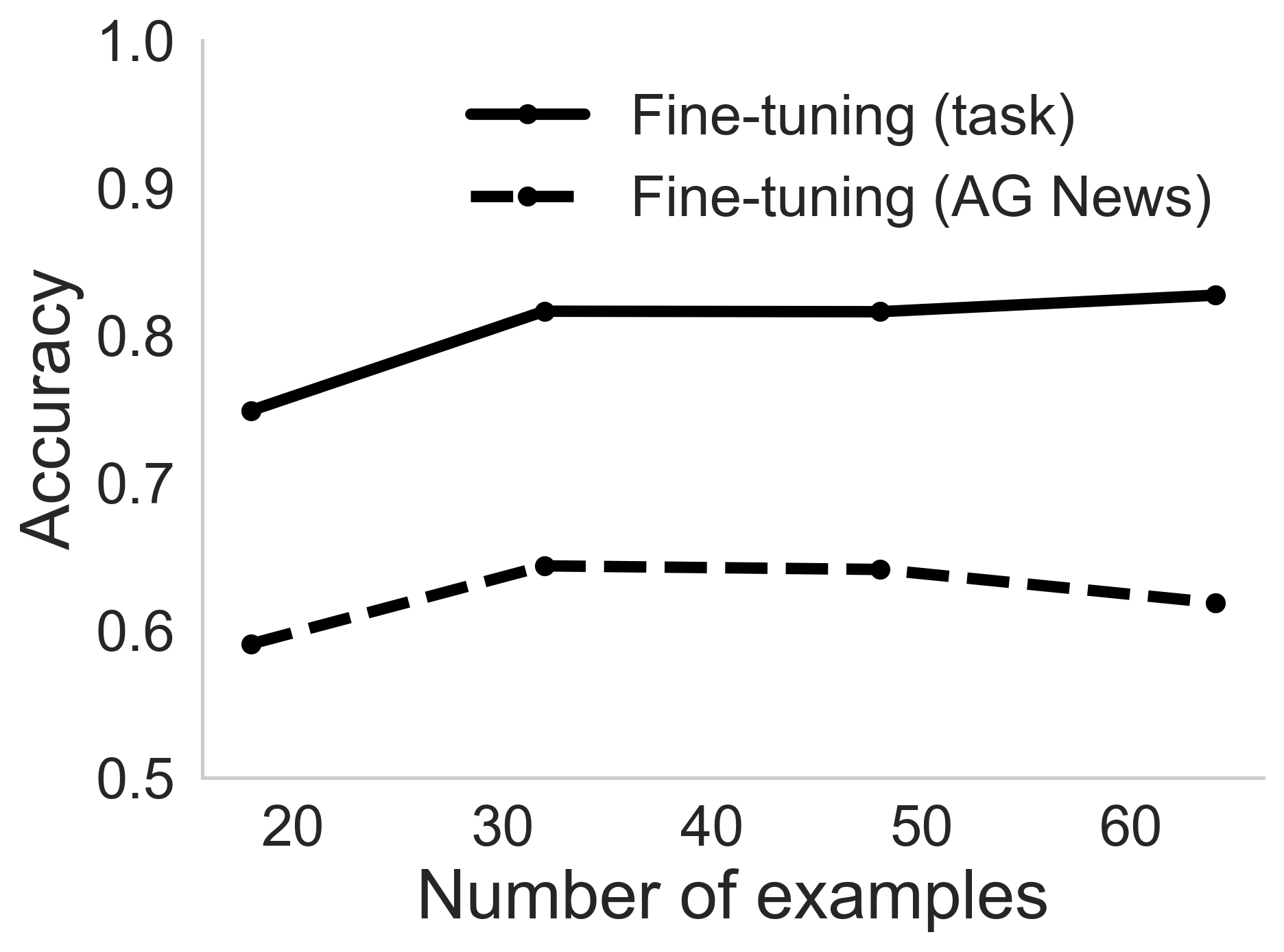}
    \subcaption{\bloomsmall}
    \end{subfigure}
    \caption{\textbf{Effects of fine-tuning on task-agnosticity} Accuracy of task-specific fine-tuned model vs. accuracy of model fine-tuned on AG-News-0 and evaluated on task. Fine-tuning hurts task-agnosticity across all three model families.}
    \label{fig:ft-ag-acc}
\end{figure*}

\begin{figure*}[t!]
    \centering
    \begin{subfigure}[b]{0.31\textwidth}
        \centering
\includegraphics[width=1\textwidth]{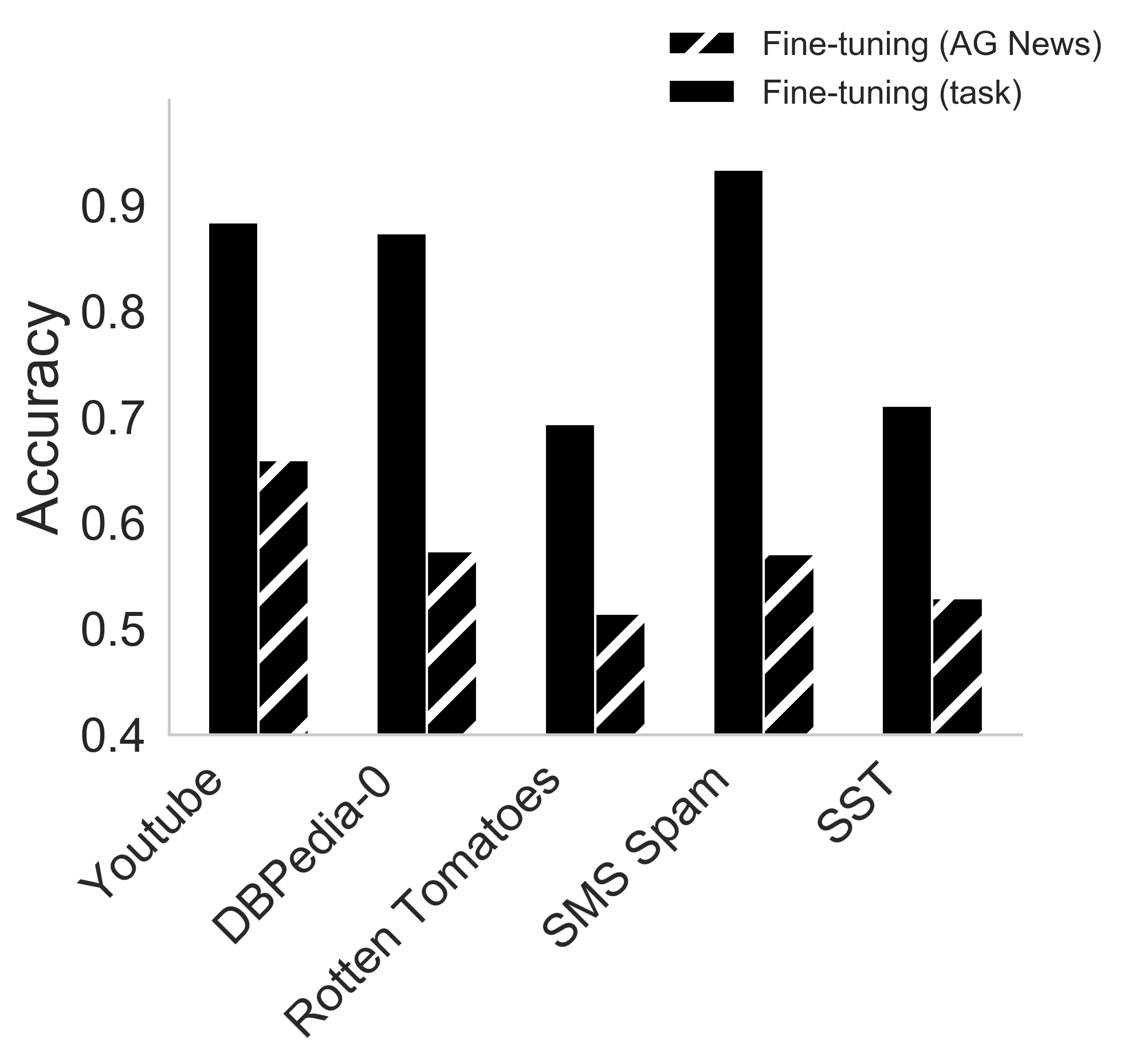}
    \subcaption{\gptsmall}
    \end{subfigure}
    \begin{subfigure}[b]{0.31\textwidth}
        \centering
\includegraphics[width=1\textwidth]{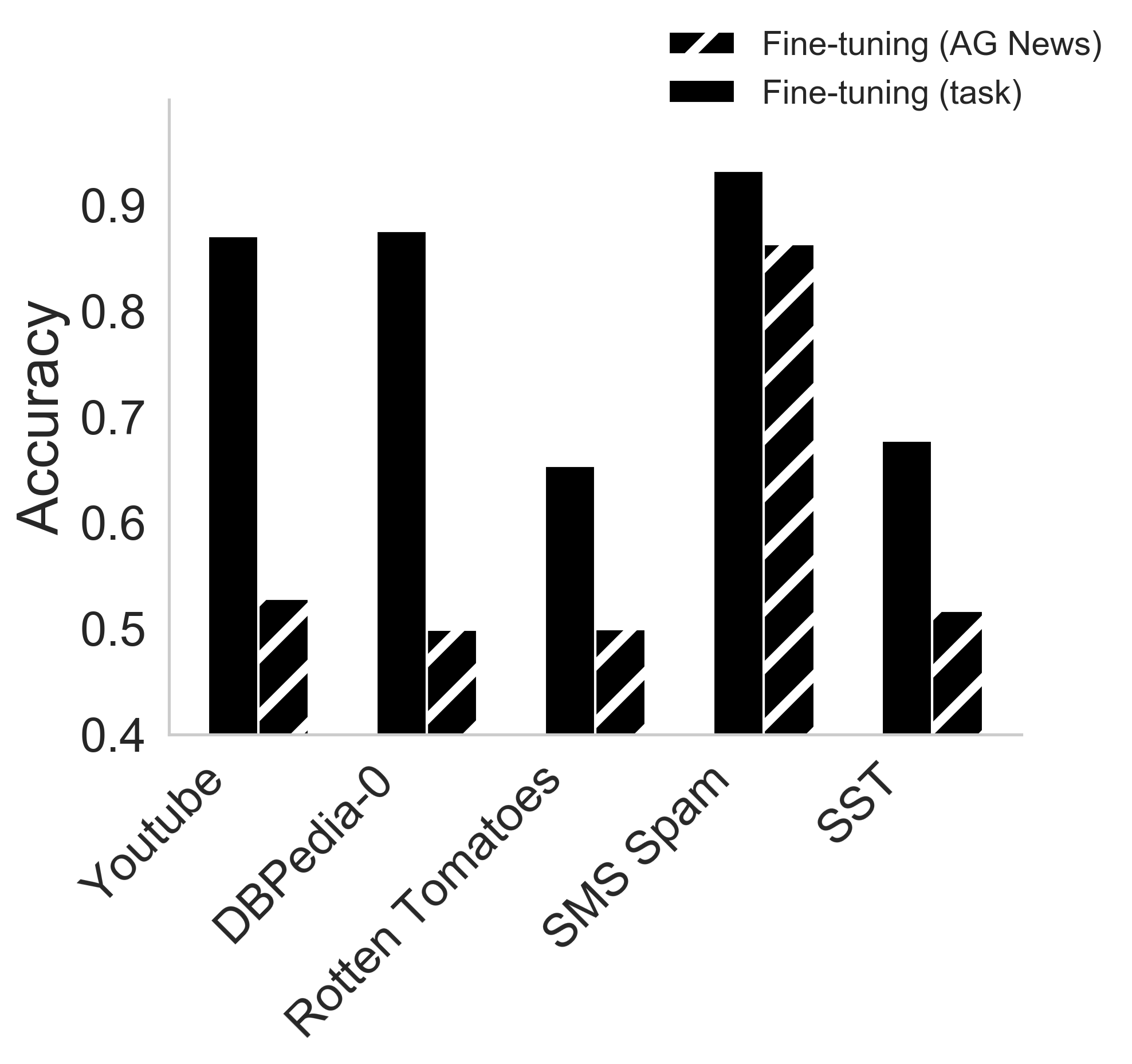}
    \subcaption{\pythiasmall}
    \end{subfigure}
    \begin{subfigure}[b]{0.31\textwidth}
        \centering
\includegraphics[width=1\textwidth]{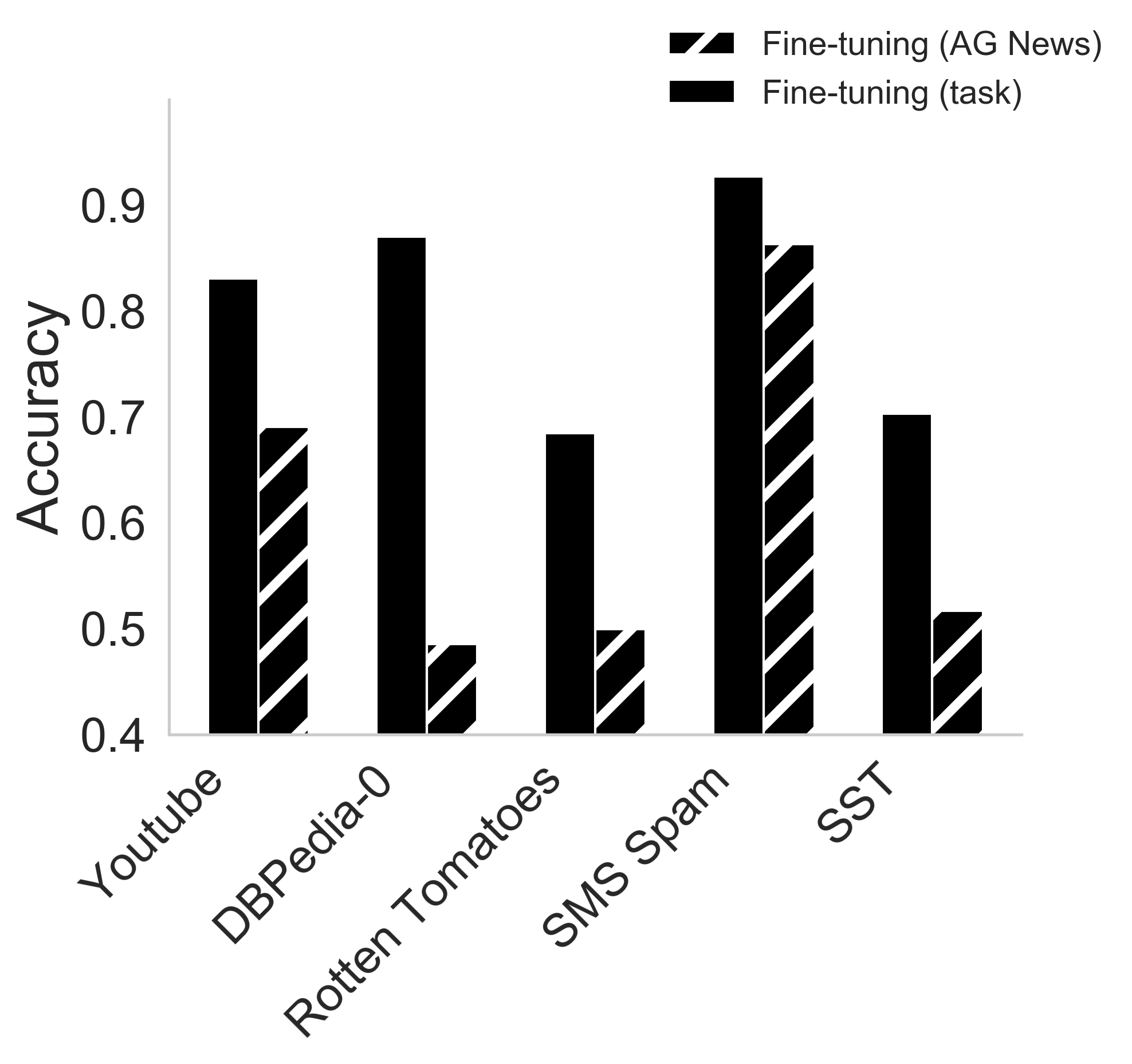}
    \subcaption{\bloomsmall}
    \end{subfigure}
    \caption{\textbf{Effects of fine-tuning on task-agnosticity (dataset level)} Accuracy of task-specific fine-tuned model vs. accuracy of model fine-tuned on AG-News-0 and evaluated on task. Fine-tuning consistently hurts task-agnosticity across all three model families and datasets.}
    \label{fig:ft-ag-bar}
\end{figure*}

\begin{figure*}[t!]
    \centering
    \begin{subfigure}[b]{0.31\textwidth}
        \centering
\includegraphics[width=1\textwidth]{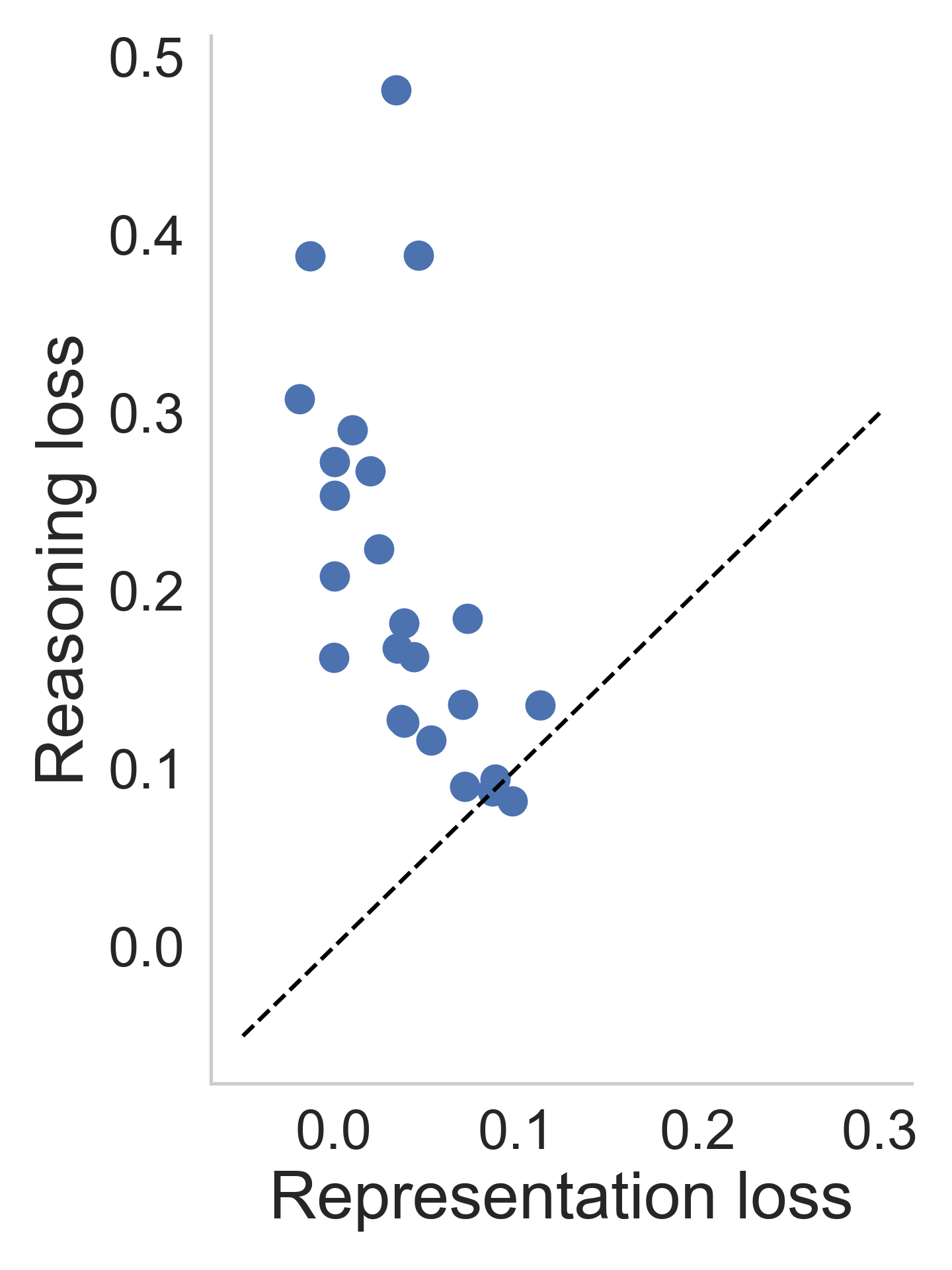}
    \subcaption{\gptsmall}
    \end{subfigure}
    \begin{subfigure}[b]{0.31\textwidth}
        \centering
\includegraphics[width=1\textwidth]{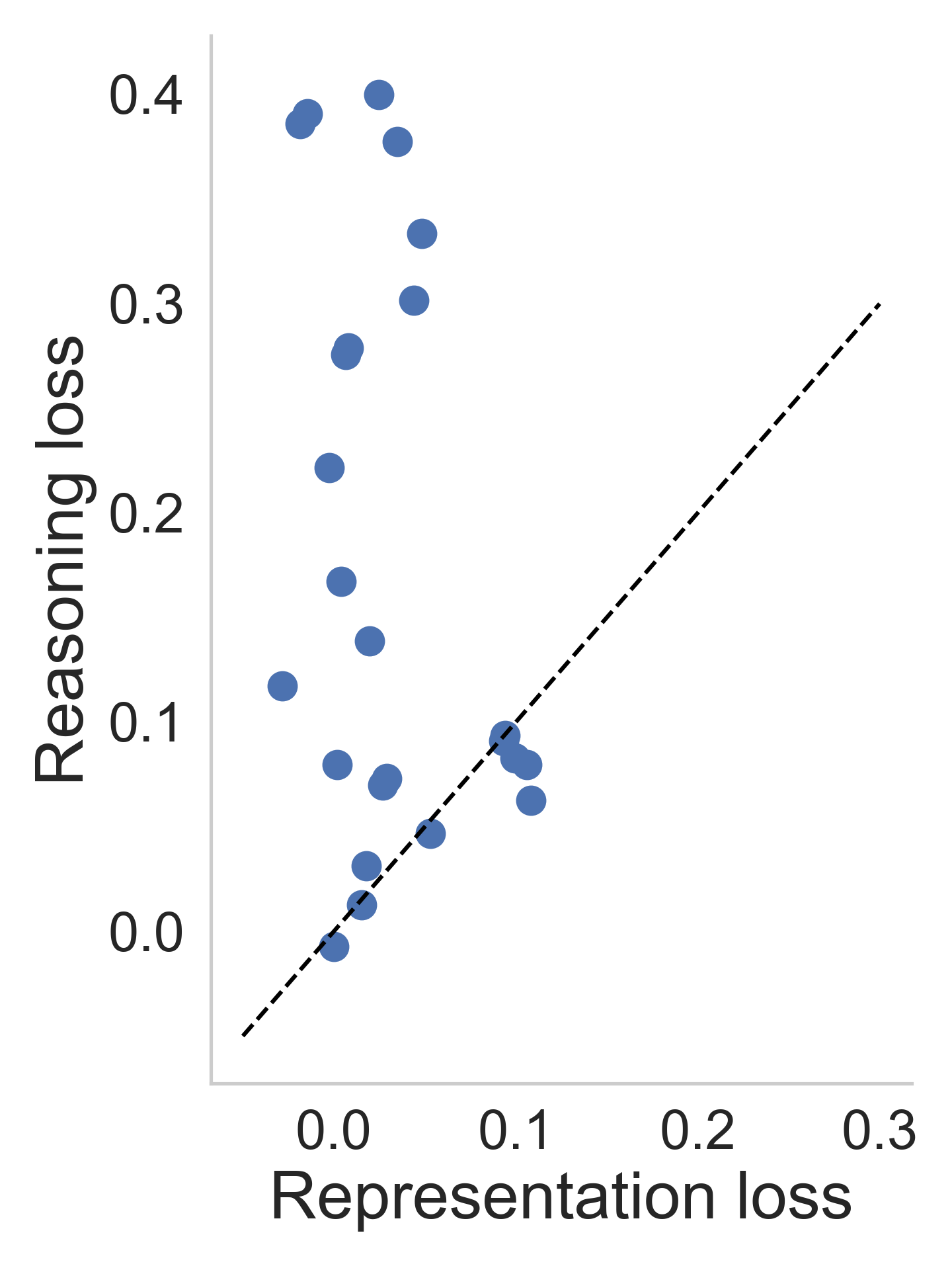}
    \subcaption{\pythiasmall}
    \end{subfigure}
    \begin{subfigure}[b]{0.31\textwidth}
        \centering
\includegraphics[width=1\textwidth]{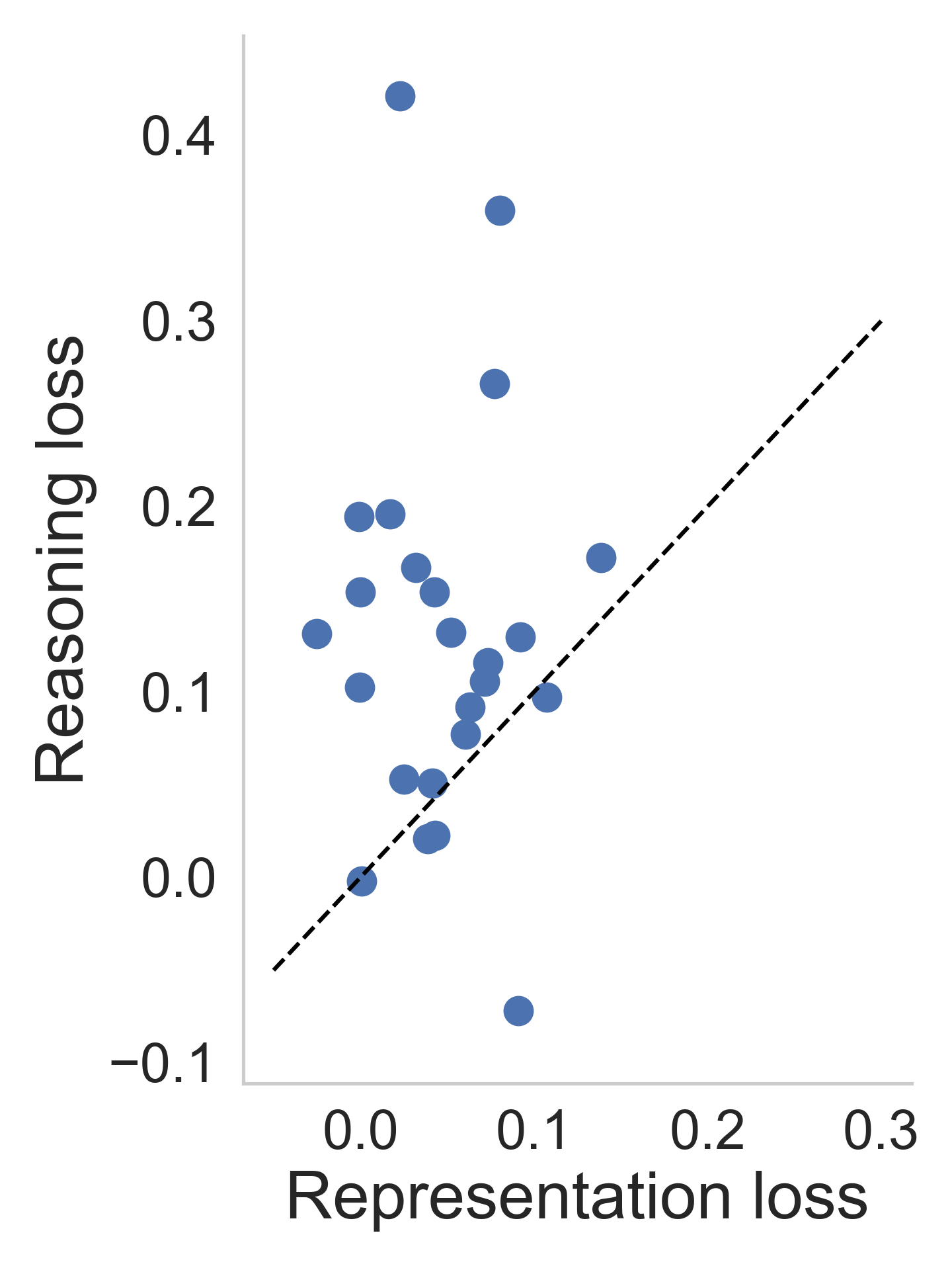}
    \subcaption{\bloomsmall}
    \end{subfigure}
    \caption{\textbf{Effects of fine-tuning on task agnosticity}. Scatter plot of reasoning loss against representation loss when the base model is trained on AG-News-0 and evaluated on other tasks. Across datasets (each point in plot represents a dataset), fine-tuning majorly impairs reasoning when transferring to tasks outside the specific fine-tuned task.}
    \label{fig:ft-ag-rep-reas}
\end{figure*}

\paragraph{Fine-tuning.} For the fine-tuning approach, we are interested in understanding two hypotheses: a) how does fine-tuning improve the model performance, and b) whether fine-tuning hurts task-agnosticity of the base model and if yes, what is the underlying reason for it. 

For the first hypothesis, we evaluate the proportion of gains that can be attributed to improved representations of the the underlying model. This is computed as the difference in performance of linear probing over the base model and over the fine-tuned model --- this evaluates how much the representations have changed specifically for this task. The reasoning gains are then computed by subtracting the representation gains from the total gain (fine-tuning accuracy minus in-context accuracy). Figure~\ref{fig:ft-rep-reas} shows a scatter plot of these representation gains and reasoning gains, plotted across different datasets and number of examples ($k$). Most of the gains which are realized by fine-tuning are because of improved task-specific reasoning capabilities across the model families.

For the second hypothesis, we first evaluate \emph{whether} fine-tuning hurts task-agnosticity. For this we evaluate two sets of accuracies: accuracy of a model fine-tuned for the specific task and the accuracy of a model on the task but fine-tuned on the AG News dataset. From Figures~\ref{fig:ft-ag-acc} and~\ref{fig:ft-ag-bar}, we see that there is a drop in accuracy---over $25.77\%$ across models and datasets. For the second part, we again decompose the drop in accuracy into a representation drop and a reasoning drop. The representation drop is computed by training a linear probe over the two models (task-specific fine-tuned and AG News fine-tuned) and looking at the difference between them. The reasoning drop, as before, is computed by subtracting this representation drop from the total drop. Figure~\ref{fig:ft-ag-rep-reas} shows that most of this drop in task-agnosticity can be attributed to over-fitting of the reasoning abilities over the task for which the models are fine-tuned.

\begin{figure*}[t!]
    \centering
    \begin{subfigure}[b]{0.31\textwidth}
        \centering
\includegraphics[width=1\textwidth]{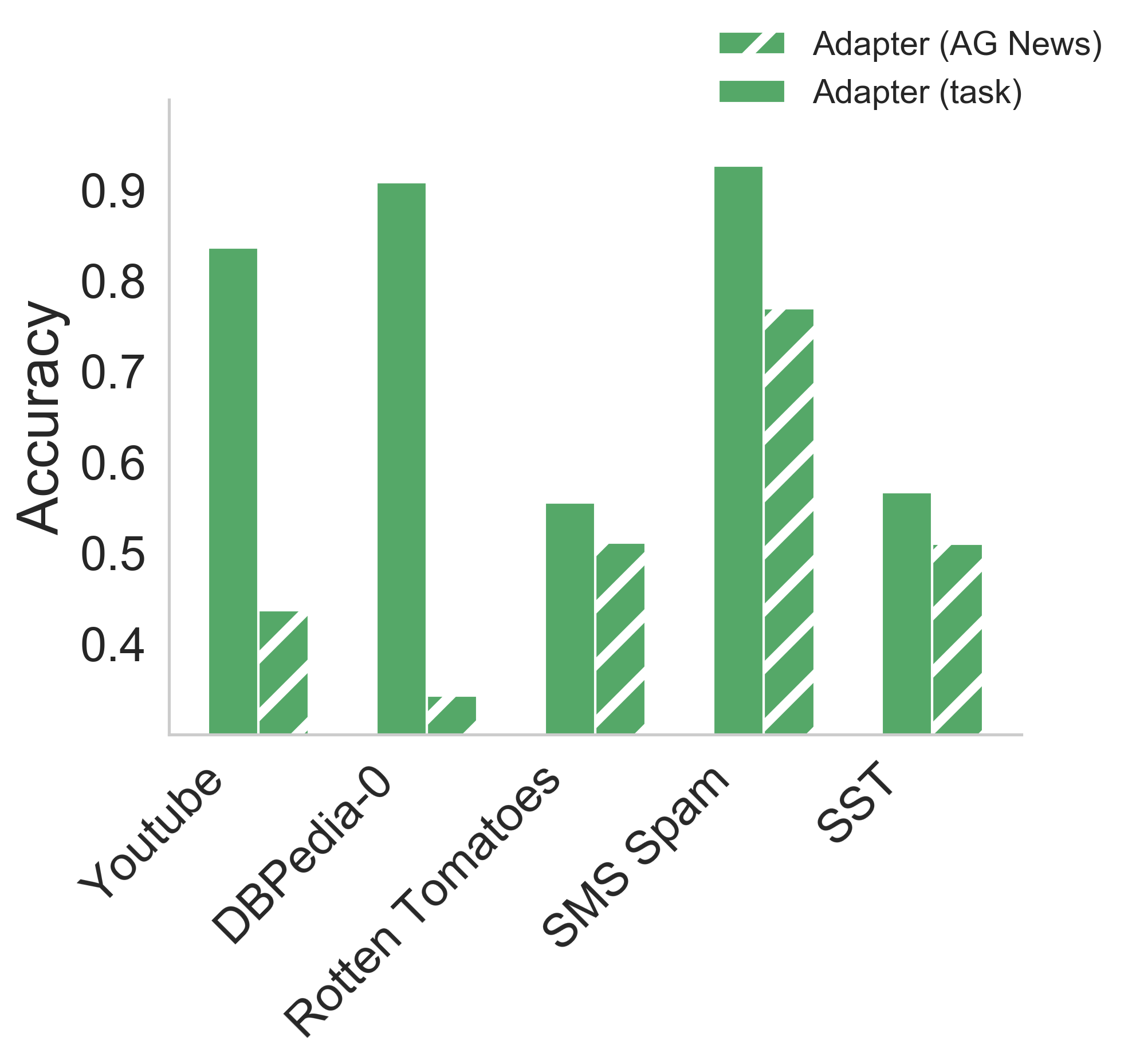}
    \subcaption{\gptsmall}
    \end{subfigure}
    \begin{subfigure}[b]{0.31\textwidth}
        \centering
\includegraphics[width=1\textwidth]{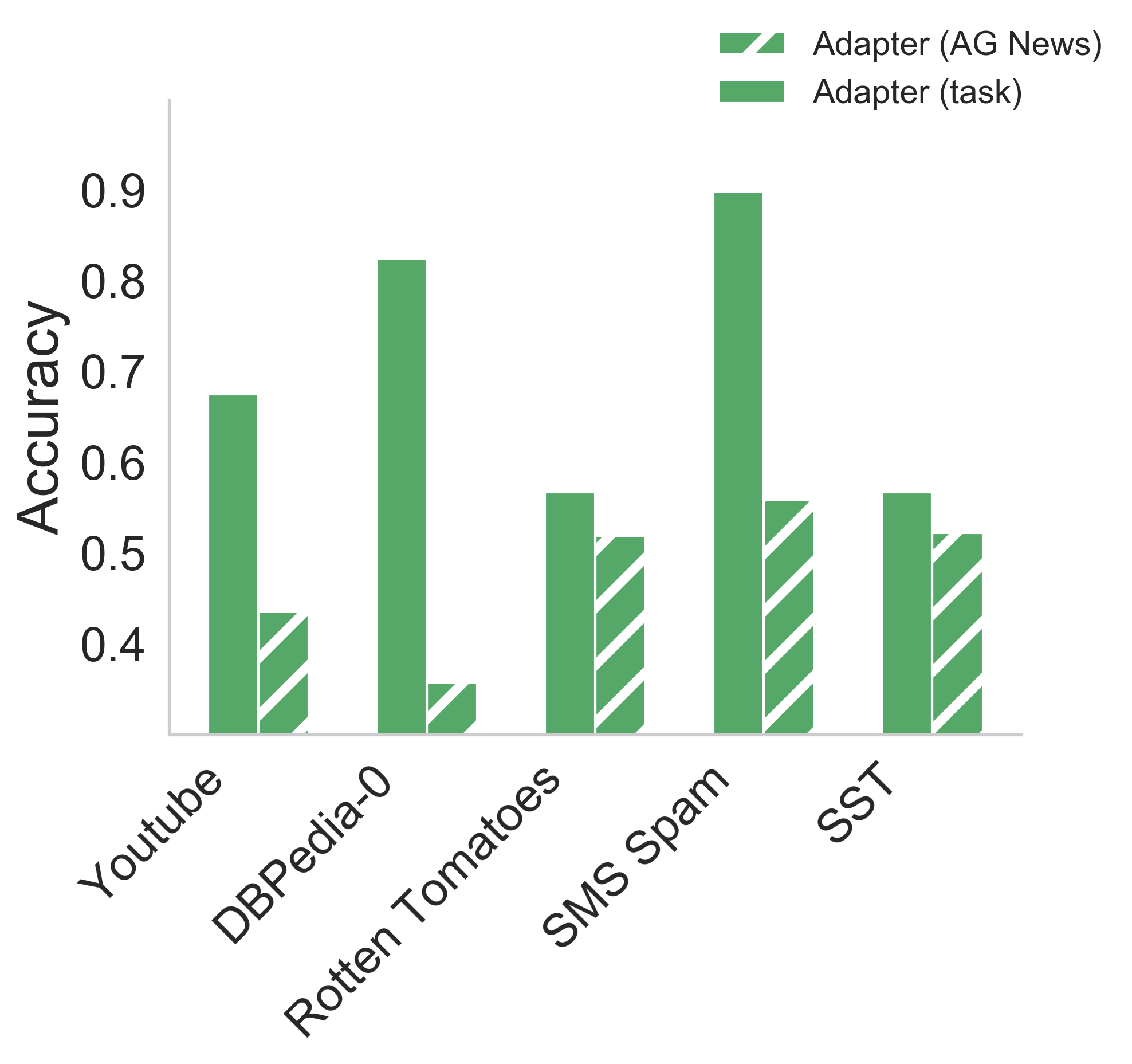}
    \subcaption{\pythiasmall}
    \end{subfigure}
    \begin{subfigure}[b]{0.31\textwidth}
        \centering
\includegraphics[width=1\textwidth]{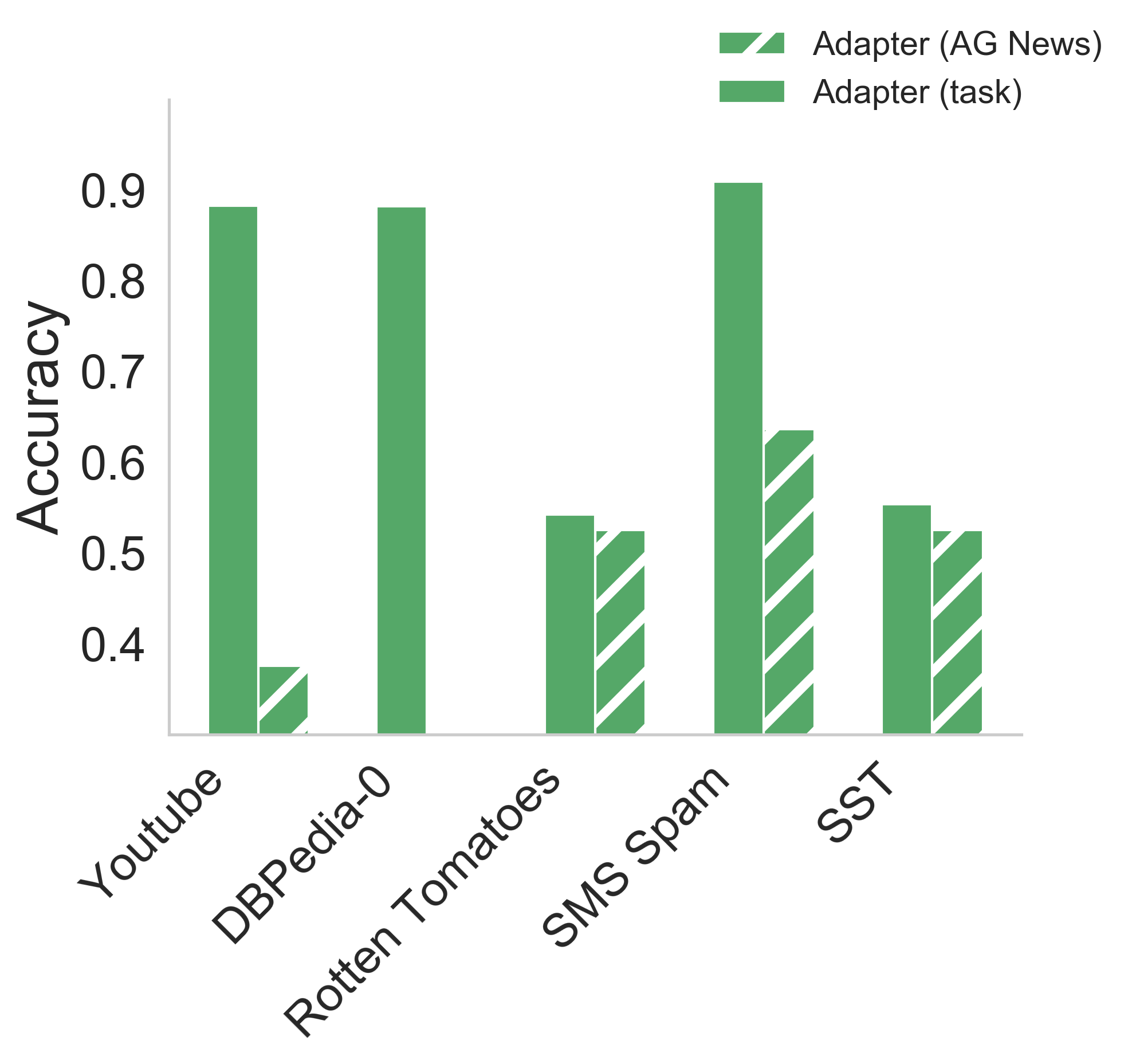}
    \subcaption{\bloomsmall}
    \end{subfigure}
    \caption{\textbf{Effects of fine-tuning on task-agnosticity of adapters}. Accuracy of task-specific fine-tuned adapter vs. accuracy of adapter fine-tuned on AG-News-0 and evaluated on task. Fine-tuning consistently hurts the generalization ability of adapters across datasets.}
    \label{fig:adapt-micro}
\end{figure*}
\paragraph{Adapters.} Since adapters do not modify the underlying representations of the model, we look at how a single adapter generalizes across tasks. For this we train an adapter on the AG News dataset and evaluate it on the other datasets. We compare this set of accuracies with those obtained by task-specific adapters in Figure~\ref{fig:adapt-micro}. The main conclusion is that task-specific adapters are not agnostic learners and over-fit to the task for which they are fine-tuned.

\section{Details for \alg implementation}\label{app:sec-3}
This section contains the details on training \alg's reasoning module and extended results on the choice of embeddings from Section~\ref{sec:fluffy}. 

\subsection{\alg's reasoning module}\label{app:tr-reas-det}
\begin{figure}[t!]
    \centering
    \includegraphics[width=1\textwidth]{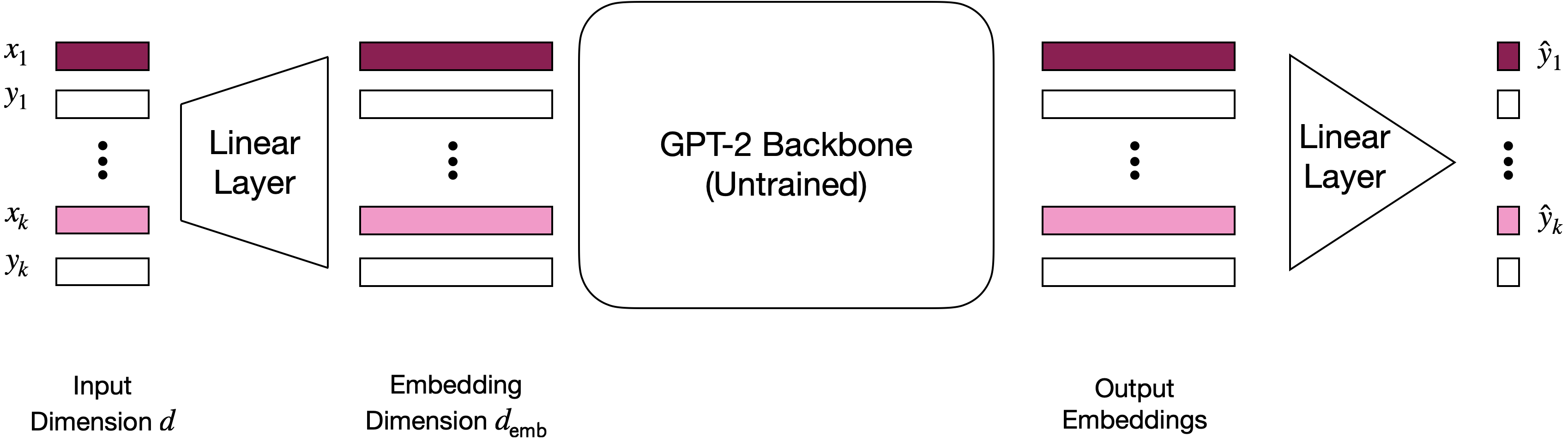}
    \caption{\textbf{\alg reasoning module architecture}. The reasoning module takes as input sequences of $(x,y)$ pairs of dimension $d$. A linear layer is used to project $d$ to the hidden dimension size of the GPT-2 backbone. Finally, a linear layer is applied to the outputs of the backbone to generate predictions for each $x_k$ in the input sequence.}
    \label{fig:arch}
\end{figure}

\paragraph{Architecture details.} We use the standard GPT-2 architecture~\citep{radford2018improving} for training our reasoning module. We set the embedding size to 256, number of decoder layers to 12, and number of heads to 8 for a total of 22 million parameters. Since the GPT-2 backbone outputs a sequence of embeddings, we additionally add a linear layer in the end to convert the output to scalar values (see Figure~\ref{fig:arch}). Additionally, the binary labels $\y$ are encoded as a one-hot vector to match the input dimension $d$ of the corresponding covariates $\x$.

\begin{figure}[t!]
    \centering
    \includegraphics[width=0.5\textwidth]{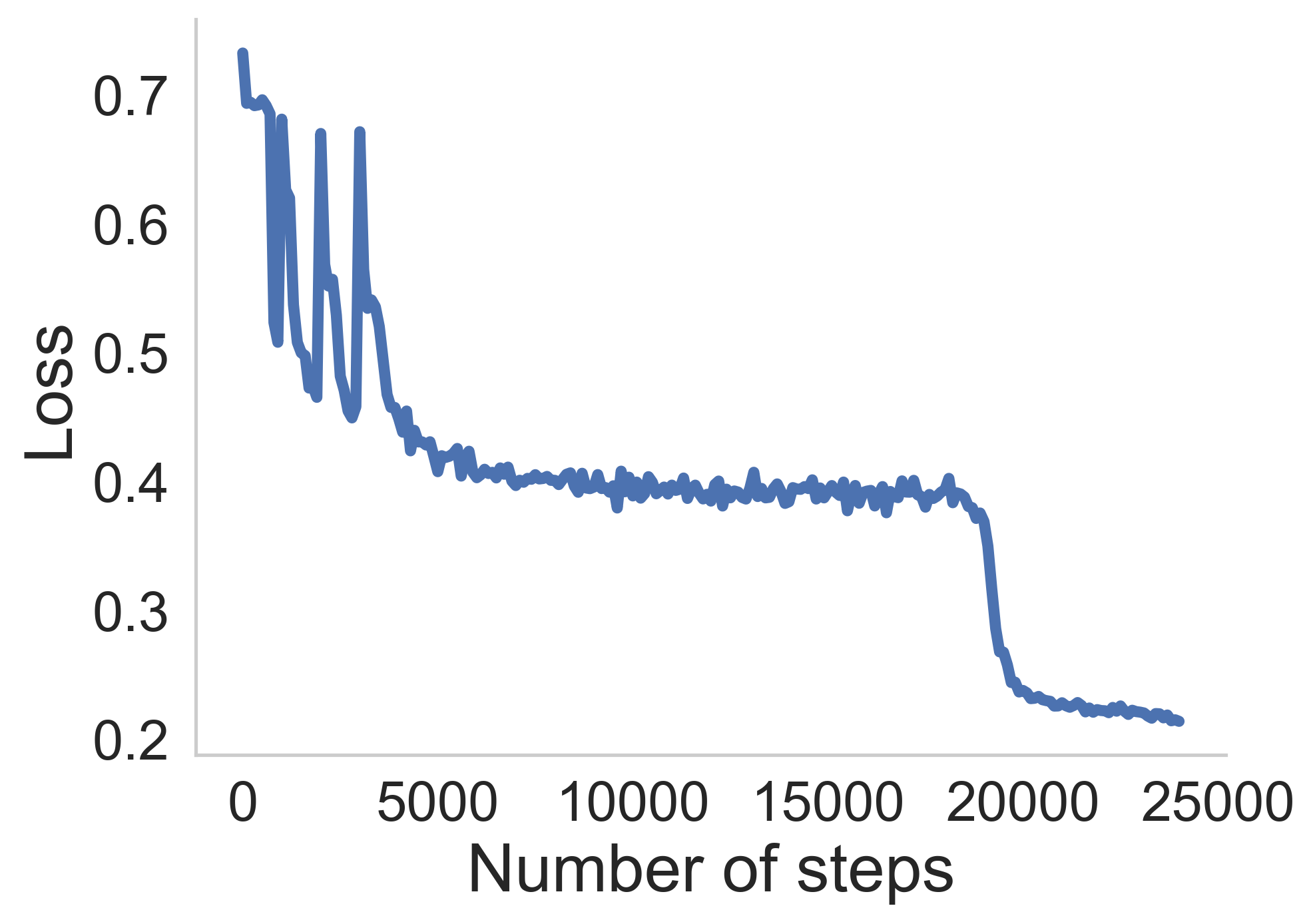}
    \caption{\textbf{Training loss vs. number of steps.} The plot shows the variation in training loss as a function of the number of steps of gradient descent for \alg's reasoning module.}
    \label{fig:train-curve}
\end{figure}

\paragraph{Training procedure.} We trained \alg's reasoning module with a context length of 258 (allowing for up to 256 in-context examples). The batch size was set to 64, learning rate to 0.0001 and the model was trained for a total of 24000 epochs. Each batch of training data consists of sampling a sequence of 258 examples using eq.~\eqref{eq:data-gen-syn}. In addition to these hyperparameters, we used a curriculum on the input dimensions and on the number of examples in the sequence to train our module---the input dimensions started from a value of 4 and were incremented by 4 every 1000 epochs while the number of examples started from 18 and were incremented by 30 every 1000 epochs.

\paragraph{Combining reasoning module with base LLM.} We trained the reasoning module with input dimension set to 16. However, most base models produce representations which are much higher dimensional (ranging from 784 to 2048). In order to reduce the dimensionality of these representations, we perform PCA on the output embeddings of the base model, learning the components using only the training points available for that specific task. The test examples are then projected onto these principal components to produce 16 dimensional input representations.

\begin{figure*}[t!]
    \centering
    \begin{subfigure}[b]{0.31\textwidth}
        \centering
\includegraphics[width=1\textwidth]{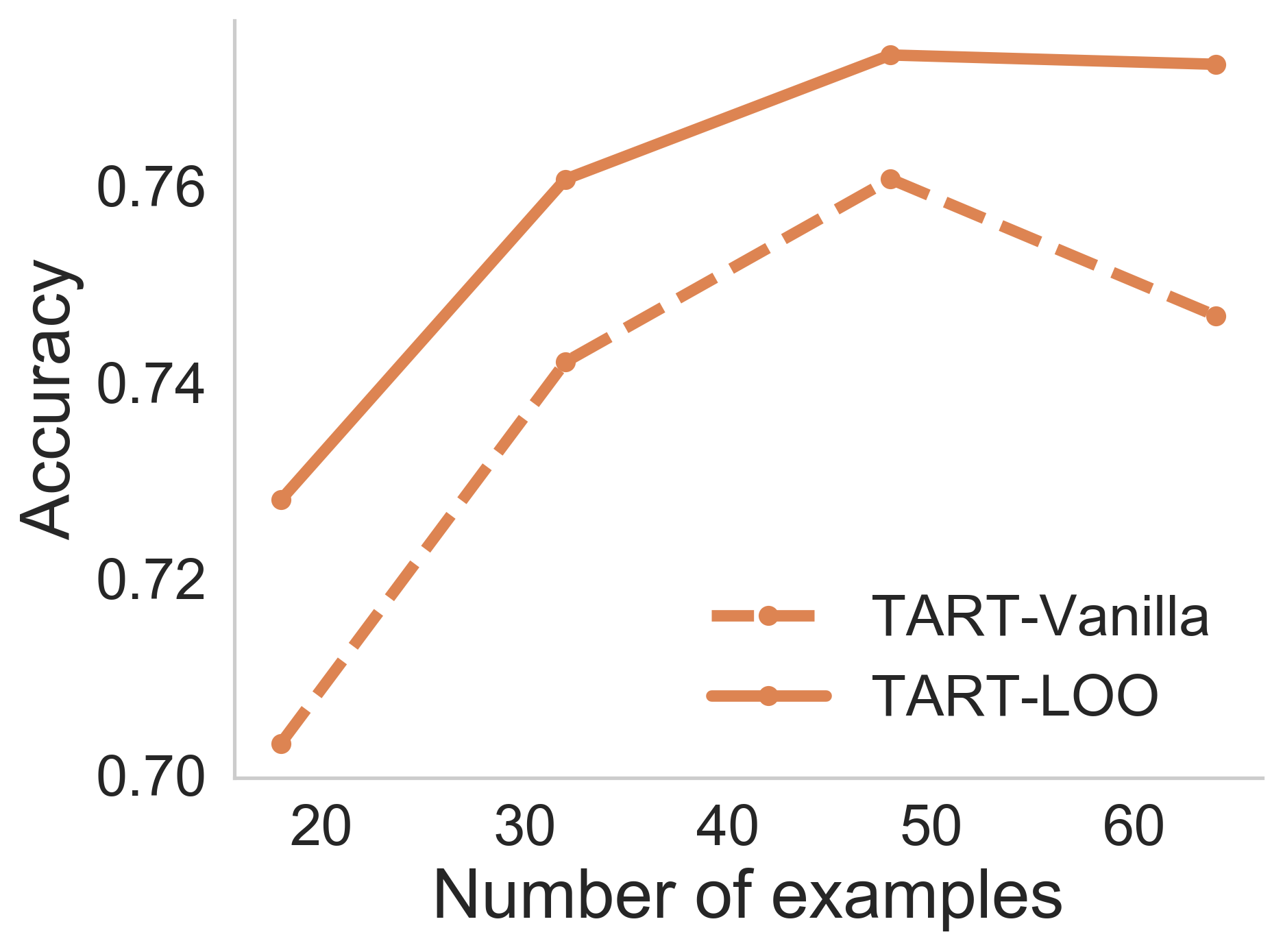}
    \subcaption{\gptsmall}
    \end{subfigure}
    \begin{subfigure}[b]{0.31\textwidth}
        \centering
\includegraphics[width=1\textwidth]{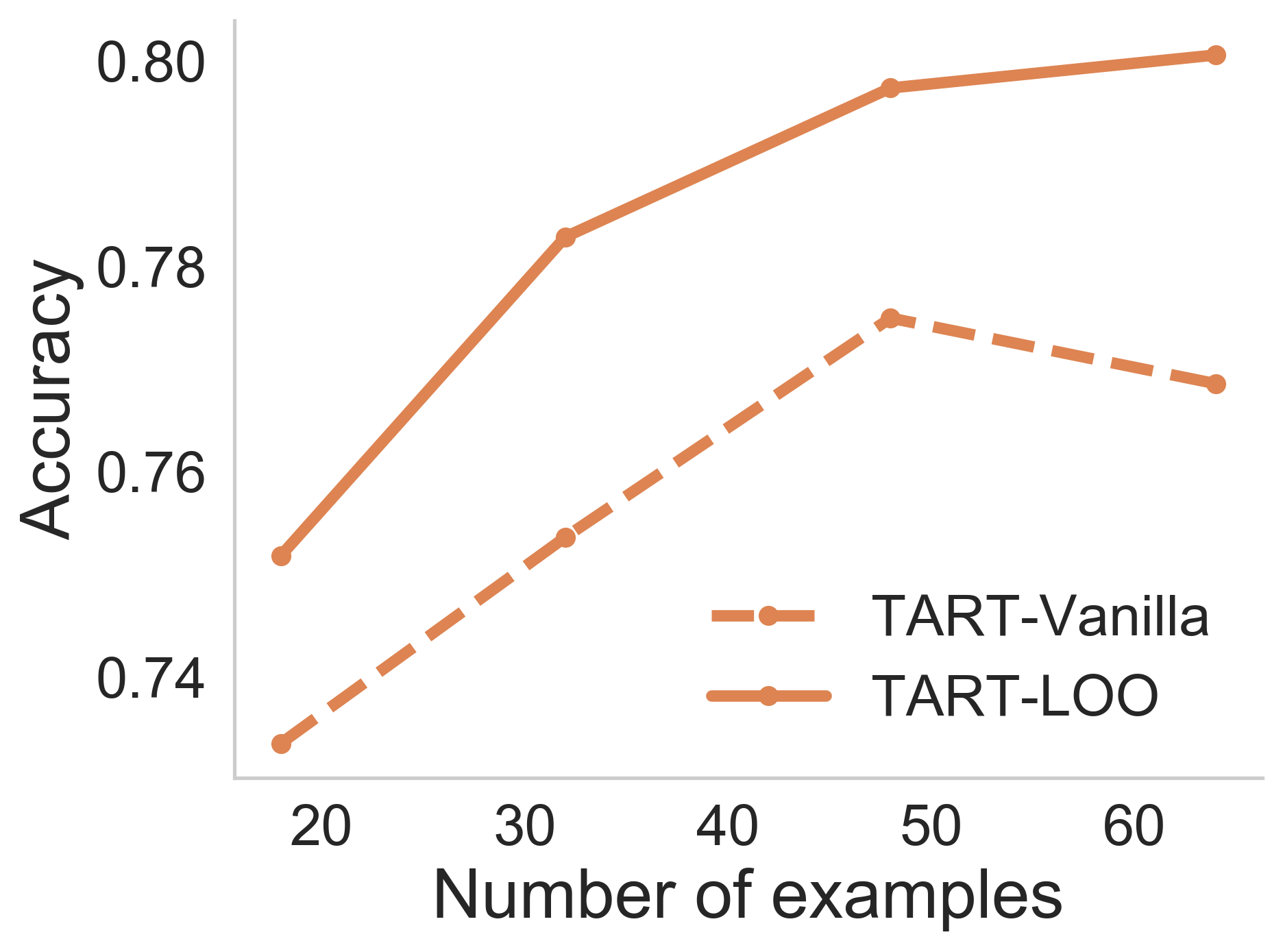}
    \subcaption{\pythiasmall}
    \end{subfigure}
    \begin{subfigure}[b]{0.31\textwidth}
        \centering
\includegraphics[width=1\textwidth]{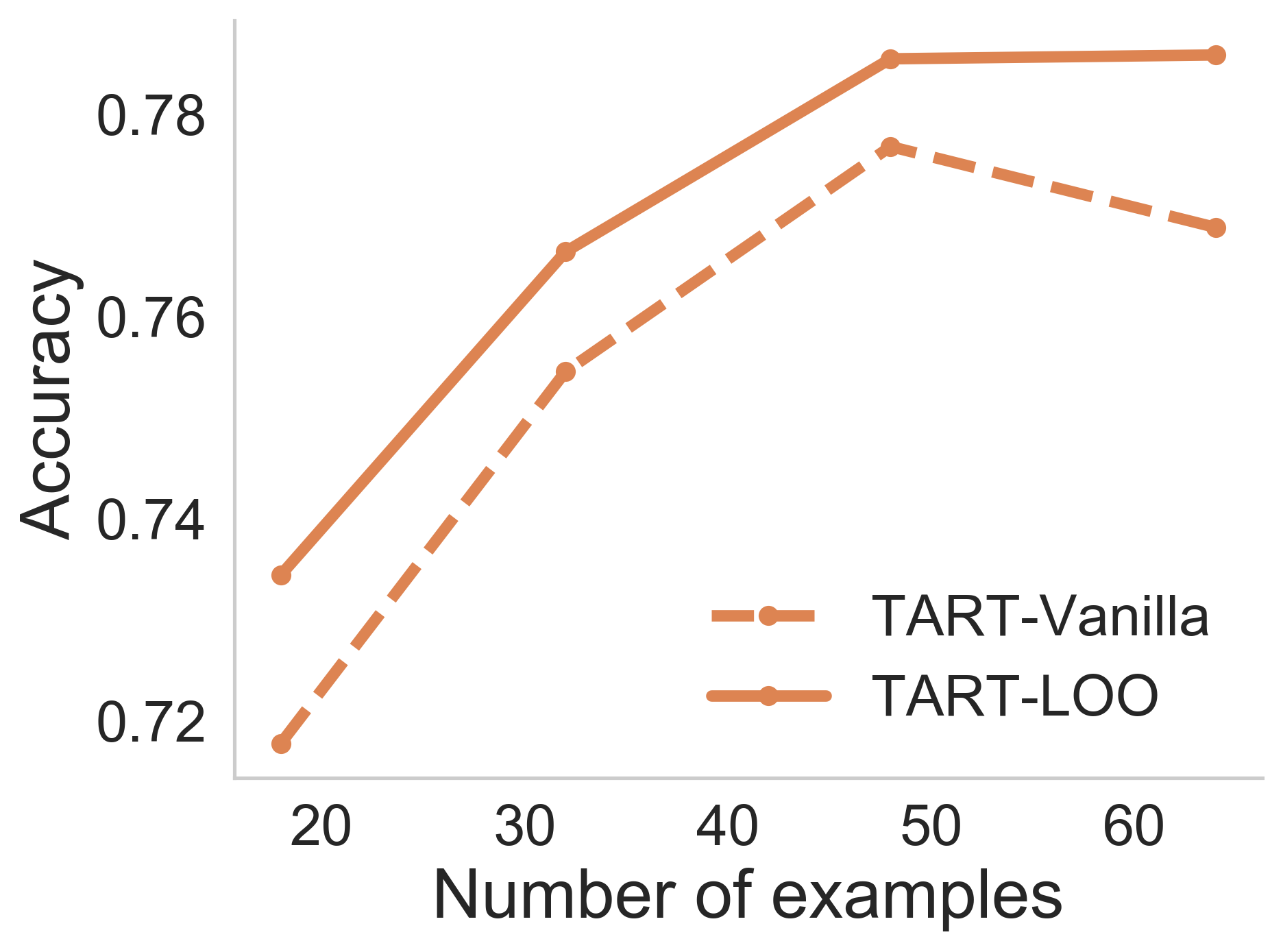}
    \subcaption{\bloomsmall}
    \end{subfigure}
    \caption{\textbf{LOO embeddings vs. Vanilla Embeddings}. Comparison of \alg performance when using LOO embeddings and vanilla embeddings. Vanilla embeddings see a performance collapse, but LOO embeddings do not.}
    \label{fig:vanilla-loo}
\end{figure*}

\begin{figure}[t!]
    \centering
    \centering
    \begin{subfigure}[b]{1\textwidth}
        \centering
\includegraphics[width=0.90\textwidth]{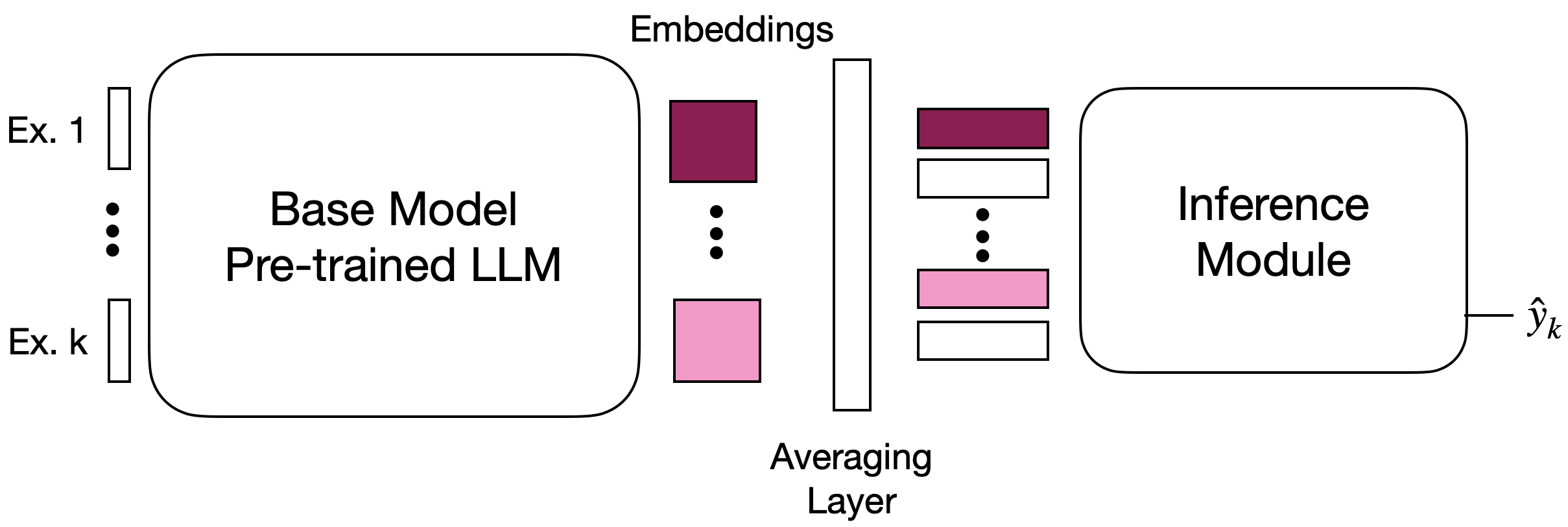}
    \subcaption{Vanilla embeddings}
    \end{subfigure}
    \begin{subfigure}[b]{1\textwidth}
        \centering
\includegraphics[width=0.80\textwidth]{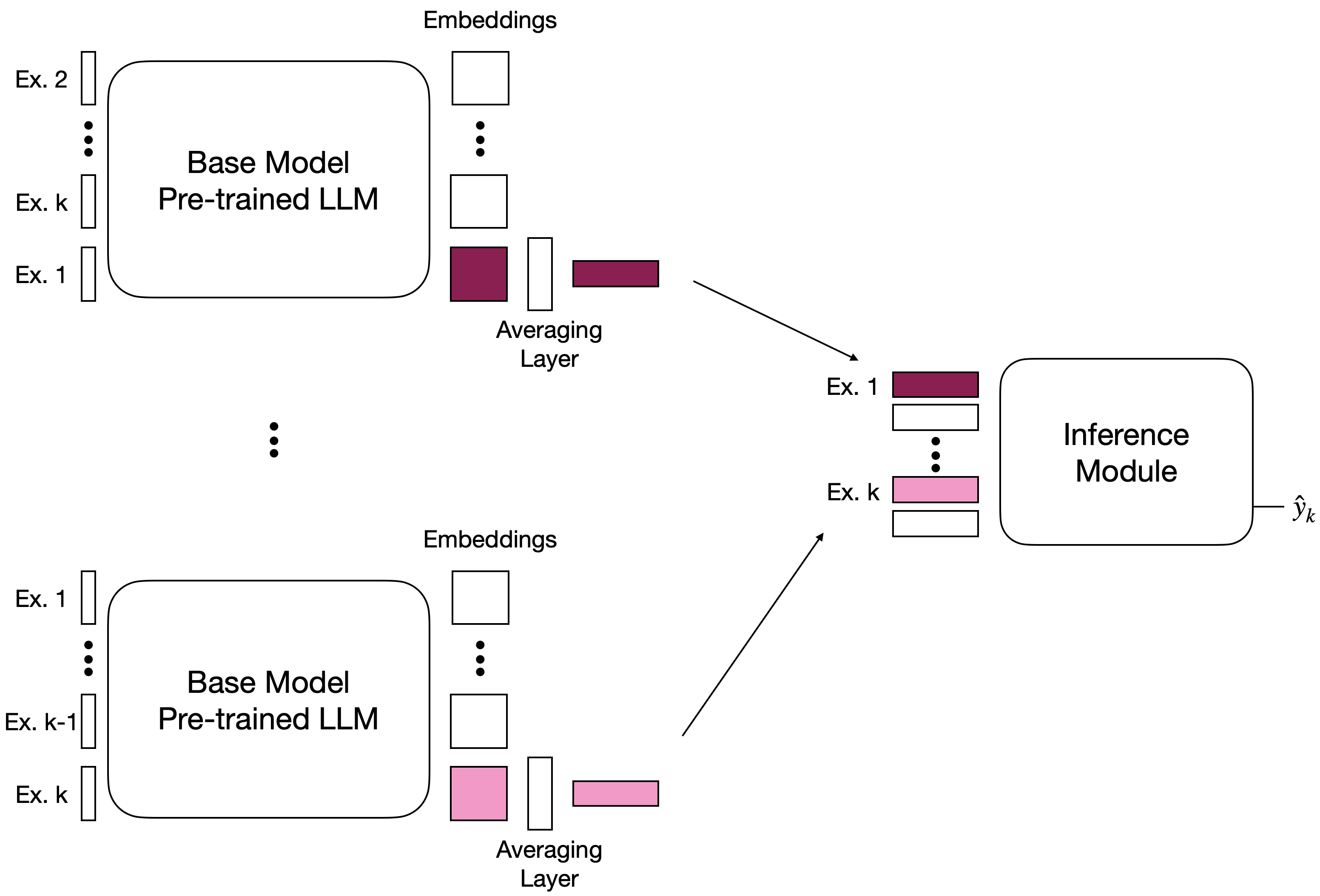}
    \subcaption{LOO embeddings}
    \end{subfigure}
    \caption{\textbf{\alg Embedding Protocols}. (a) For the vanilla embeddings, the test example is appended to the training set and the sequence is passed to the base model. The representation for each train example in this sequence is taken as the average embedding across all its tokens. (b) For the LOO embeddings, we generate embeddings for each train example separately by placing all the other train examples before it in the prompt and averaging the embeddings over the final example's tokens.}
    \label{fig:embed}
\end{figure}

\subsection{Choice of representations}\label{app:rep-emb}
As discussed in Section~\ref{sec:rep-emb} there are two possible options for forming the representations, the vanilla embeddings and the leave-one-out (LOO) embeddings. Figure~\ref{fig:embed} shows the schematic differences between the two style of embedding. In Figure~\ref{fig:vanilla-loo}, we plot the average accuracies across different datasets for both vanilla and LOO embeddings, observing that the LOO embeddings consistently perform better across the different model families. 

\subsection{Proof of Theorem~\ref{thm:gen}}\label{app:theory}
In this section, we provide a formal statement of Theorem~\ref{thm:gen} from Section~\ref{sec:theory}. Our theorem quantifies the expected error of the Transformer, trained on synthetic data, on natural language tasks in terms of the change in the two input distributions.

We begin by introducing some notation. We denote the class of Transformer family by 
\begin{equation}
\T_\PT \defn \{ \Tr_\paramT : \real^{(2\kt+1)\times \dt} \mapsto \real \quad | \quad \paramT \in \PT\}\;,
\end{equation}
where $\kt$ represents the maximum number of in-context examples the Transformer can support, $\dt$ represents the input dimensions, and $\PT$ represents the corresponding parameter class over which the Transformer family is defined. 

Observe that the Transformer family $\T_\PT$ takes as input a sequence of $\kt$ train examples, each corresponding to two tokens of hidden dimension $\dt$: a covariate $\x \in \real^d$ and a binary label, encoded as a one-hot vector in $\dt$ dimension. This sequence of train examples is followed by a test example, for which we only have the features $\x_{k+1}$. 

Given this background, let $\Psyn$ denote the synthetic distribution over sequences $\{(\x_1, \y_1), \ldots, (\x_k, \y_x), (\x_{k+1})\}$. Similarly, let $\Pnl$ denote the corresponding distribution over sequences derived from natural language tasks where $\x_i$ denotes the LLM embeddings of the example. Recall from Section~\ref{sec:training}, the synthetic training distribution $\Psyn$ is given by
\begin{equation}
\text{Sequence } \seq_t: \param_t \sim \normal(0, I_\dx), \quad \x_{i,t}\sim \normal(0, I_\dx), \quad \y_{i,t} \sim \sigmoid(\weight\inner{\x_{i,t}}{\param_t}) \qquad \text{for }i \in [\klr]\;,
\end{equation}
for each training point $(\x_{i,t}, \y_{i,t})$. The test point is also sampled similarly from an independent standard normal distribution. Let $\ell:\real\times \real \mapsto \real$ be the loss function used for evaluating the performance of the reasoning module. Further, let use denote the expected loss under a distribution~$\Pd$
\begin{equation}
    \err_{\Pd}(\Tr) \defn \En_{(s,\y) \sim \Pd} [\ell(\Tr(s), \y)]\;,
\end{equation}
and the corresponding empirical distribution over samples $\samp$ by $\hat{\err}_\Pd$, where the dependence on the samples is implicit. Given these samples, we denote the empirical risk minimizer 
\begin{equation}
    \Tres  = \arg\min_{\Tr \in \T_\PT} \frac{1}{|\samp|} \sum_{s \in \samp} \ell(\Tr(s), \y)\;.
\end{equation}
In addition to these notation, we make the following Lipschitz assumption on the the loss function $\loss$ and the Transformer model $\Tr$.
\begin{assumption}\label{ass:lip-loss}[Lipschitz loss.] For any two output labels $y, y'$, the loss function $\ell$ is Lipschitz with constant $c_1$, that is,
\begin{equation}
    |\ell(Tr(s), \y) - \ell(Tr(s), \y')| \leq c_1 |\y - \y'|\;.
\end{equation}
\end{assumption}

\begin{assumption}\label{ass:lip-model}[Lipschitz models.] For any two input sequences $s, s'$, each any model $\Tr \in \T_\PT$ is Lipschitz with constant $L$, that is,
\begin{equation}
    |\Tr(s) - \Tr(s')| \leq L\|s-s'\|.
\end{equation}
\end{assumption}

Given this setup, we are now ready to state a formal version of Theorem~\ref{thm:gen}.  

\begin{theorem}[Formal version of Theorem~\ref{thm:gen}]\label{thm:gen-formal}
Let $\Tres \in \T_\PT$ denote the trained reasoning module on set $S$ of synthetic logistic regression tasks with $\nsyn$ sequences sampled from distribution $\Psyn$ in eq.~\eqref{eq:data-gen-syn}. Let the loss function $\ell$ satisfy Assumption~\ref{ass:lip-loss} and the model class $\T_\PT$ satisfy Assumption~\ref{ass:lip-model}. Then, with probability at least $1-\delta$, we have
{\small
\begin{equation}
\err_{\Pnl}(\Tres) \leq  c_1\max(1,L)\cdot  W_1(\Pnl, \Psyn) 
+ c_1\cdot \sqrt{\frac{2\text{VC}(\T_\PT) \ln m}{\nsyn}} + 4\sqrt{\frac{2 \ln(4/\delta)}{\nsyn}}
+ \hat{\err}_{\Psyn}(\Tres)\;,
\end{equation}
}
where $W_1$ denotes the Wasserstein-1 metric and  $\text{VC}(\T_\PT)$ represents the VC dimension of class $\T_\PT$
\end{theorem}
\begin{proof}
    We begin by decomposing the error $\err_{\Pnl}(\Tres)$ into three components as
    \begin{equation}
        \err_{\Pnl}(\Tres) = \underbrace{\err_{\Pnl}(\Tres) - \err_{\Psyn}(\Tres)}_{\text{(I)}} + \underbrace{\err_{\Psyn}(\Tres) - \hat{\err}_{\Psyn}(\Tres)}_{\text{(II)}} + \hat{\err}_{\Psyn}(\Tres)\;.
    \end{equation}

    We now upper bound each of the terms (I) and (II) separately. 

    \paragraph{Bound on Term (I).} Let $\joint$ denote an arbitrary joint distribution over the distributions $\Psyn$ and $\Pnl$. Then, we can bound the first term as
    \begin{align}
    \err_{\Pnl}(\Tres) - \err_{\Psyn}(\Tres) &= \En_{\Pnl}[\ell(\Tr(s), \y)] - \En_{\Psyn}[\ell(\Tr(s'), \y')]\nonumber\\
    &\stackrel{\1}{=} \En_{\joint}[\ell(\Tr(s), \y) - \ell(\Tr(s'), \y')]\nonumber\\
    &\stackrel{\2}{\leq} \inf_{\joint} \En_{\joint}\left|\ell(\Tr(s), \y) - \ell(\Tr(s'), \y') \right|\;,
    \end{align}
    where $\1$ follows from the independence of the two expectations and $\2$ follows from that $\1$ holds for any arbitrary joint distribution $\joint$. The final bound on this term now follows:
    \begin{align}\label{eq:bnd-t1}
    \inf_{\joint}\En_{\joint}\left|\ell(\Tr(s), \y) - \ell(\Tr(s'), \y') \right| &= \inf_{\joint} \int\left|\ell(\Tr(s), \y) - \ell(\Tr(s), \y') + \ell(\Tr(s), \y') - \ell(\Tr(s'), \y')\right| d\joint\nonumber \\
    &\stackrel{\1}{\leq} c_1 \inf_{\joint} \int|\y-\y'| - \|\Tr(s')-\Tr(s)\| d\joint\nonumber \\
    &\stackrel{\2}{\leq} c_1\max(1, L)\cdot \inf_{\joint} \int|\y-\y'| - \|s'-s\| d\joint\nonumber \\
    &= c_1\max(1,L)\cdot  W_1(\Pnl, \Psyn)\;,
    \end{align}
    where the inequalities $\1$ follows from Assumption~\ref{ass:lip-loss} and $\2$ follows from Assumption~\ref{ass:lip-model}. This completes the bound on Term (I). 
    \newcommand{\rad}{\mathcal{R}}
    \paragraph{Bound on Term (II).} Using a standard generalization bound~\citep[see Theorem 26.5]{shalev2014understanding}, we have with probability at least $1-\delta$
    \begin{align}\label{eq:bnd-term-2}
        \err_{\Psyn}(\Tres) - \hat{\err}_{\Psyn}(\Tres) &\leq \rad(\ell\circ\T_\PT) + 4\sqrt{\frac{2 \ln(4/\delta)}{\nsyn}}\nonumber\\
        &\leq c_1\cdot \rad(\T_\PT) + 4\sqrt{\frac{2 \ln(4/\delta)}{\nsyn}}\nonumber \\
        &\stackrel{\1}{\leq} c_1\cdot \sqrt{\frac{2\text{VC}(\T_\PT) \ln m}{\nsyn}} + 4\sqrt{\frac{2 \ln(4/\delta)}{\nsyn}}
    \end{align}
    where $\rad(\ell\circ\T_\PT)$ denotes the Rademacher complexity of the class $\T_\PT$ composed with the loss function $\ell$ and inequality $\1$ follows from Sauer’s Lemma. 

    Combining the bounds in equations~\eqref{eq:bnd-t1} and~\eqref{eq:bnd-term-2} completes the proof of the theorem.
\end{proof}

\section{Details for experimental evaluation}\label{app:exp}
We describe supplementary experimental details from Section~\ref{sec:exp} as well as additional results for the natural language benchmark evaluations (Section~\ref{app:nl-eval}) and results for other modalities (vision and audio) (Section~\ref{app:modalities}).

\subsection{Experimental setup}\label{app:exp-setup}
We begin by providing dataset statistics and details of the  baselines.

\begin{table}
\centering

\begin{tabular}{cccccc}
\toprule
\multirow{2}{*}{Dataset}  & \multirow{2}{*}{Test size}  & Max char & Max token & Avg. char & Avg. token\\
& &length & length & length & length \\
\midrule
AG-News-0 & 3800 & 732 & 259 & 237.07 & 51.36 \\
AG-News-1 & 3800 & 814 & 213 & 232.01 & 51.46 \\
AG-News-2 & 3800 & 814 & 225 & 236.10 & 52.25 \\
AG-News-3 & 3800 & 892 & 259 & 234.86 & 51.38 \\
Civil Comments & 11576 & 1000 & 634 & 272.72 & 61.73 \\
DBPedia-0 & 10000 & 2081 & 629 & 300.94 & 65.83 \\
DBPedia-1 & 10000 & 2081 & 629 & 298.81 & 66.91 \\
DBPedia-2 & 10000 & 2081 & 883 & 286.85 & 66.53 \\
DBPedia-3 & 10000 & 2081 & 629 & 275.81 & 63.88 \\
IMDB & 25000 & 12988 & 2972 & 1293.79 & 292.82 \\
Rotten Tomatoes & 1066 & 261 & 63 & 115.52 & 25.36 \\
SMS Spam & 4181 & 612 & 258 & 81.46 & 23.76 \\
SST & 2210 & 256 & 60 & 102.40 & 22.34 \\
Youtube & 250 & 1125 & 292 & 112.50 & 31.84 \\
\bottomrule
\end{tabular}
\vspace{2mm}
\caption{Dataset (test) statistics for all NLP datasets.}
\label{tab:dataset-stats-tst}
\end{table}

\begin{table}
\centering

\begin{tabular}{cccccc}
\toprule
\multirow{2}{*}{Dataset}    & Max char & Max token & Avg. char & Avg. token &\\
&length & length & length & length \\
\midrule
AG-News-0  & 701 & 256 & 236.15 & 51.53 \\
AG-News-1  & 749 & 180 & 232.48 & 51.61 \\
AG-News-2  & 735 & 256 & 241.70 & 53.91 \\
AG-News-3  & 1002 & 258 & 241.21 & 53.21 \\
Civil Comments  & 1000 & 347 & 280.97 & 63.44 \\
DBPedia-0  & 707 & 207 & 300.48 & 65.79 \\
DBPedia-1  & 1023 & 280 & 299.89 & 66.57 \\
DBPedia-2  & 628 & 203 & 288.21 & 66.22 \\
DBPedia-3  & 758 & 203 & 279.45 & 64.24 \\
IMDB  & 7068 & 1630 & 1284.20 & 290.00 \\
Rotten Tomatoes  & 260 & 62 & 112.46 & 24.82 \\
SMS Spam  & 911 & 217 & 106.90 & 31.87 \\
SST  & 248 & 56 & 101.97 & 22.34 \\
Youtube  & 1089 & 767 & 90.12 & 29.98 \\
\bottomrule
\end{tabular}
\vspace{2mm}
\caption{Dataset (train) statistics for all NLP datasets.}
\label{tab:dataset-stats-tr}
\end{table}

\subsubsection{Dataset construction and statistics}
Table~\ref{tab:dataset-stats-tr} and~\ref{tab:dataset-stats-tst} provides a detailed breakdown of dataset statistics. For each dataset, we use the original test sets with the exception of Civil Comments~\citep{DBLP:journals/corr/abs-1903-04561}, AG News~\citep{zhang2015character} and DBPedia~\citep{zhang2015character}. 
For the multi-class datasets---AG News~\citep{zhang2015character} and DBPedia~\citep{zhang2015character} --- we construct 4 binary classification tasks for each datasets. More concretely, AG News labels news articles into four categories: World, Sports, Business, and Science/Technology. We create a separate binary classification task for each category, sampling negatives from the remaining classes. DBPedia is a 14-way ontology classification dataset. We create 4 separate binary classification tasks for the educational institution, company, artist, and athlete ontologies, sampling negatives from the remaining classes. For the train set, we sample a class-balanced set of 64 examples from the original dataset. For each dataset, we sample 5 separate training sets, using 5 different random seeds. In evaluations, we evaluate \alg and the baseline methods across each of these 5 different training sets. 

\begin{figure*}[t!]
    \centering
    \begin{subfigure}[b]{0.31\textwidth}
        \centering
\includegraphics[width=1\textwidth]{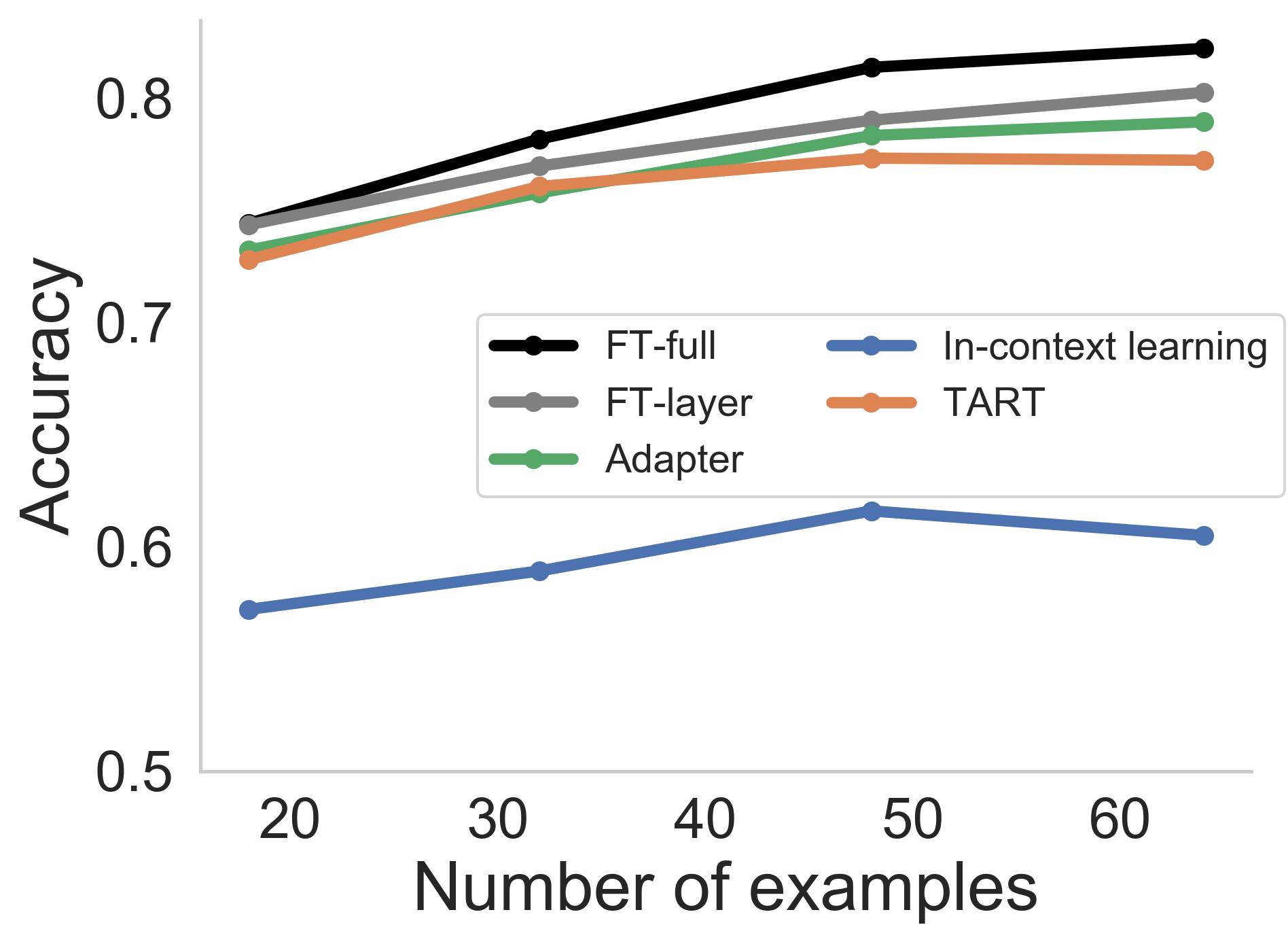}
    \subcaption{\gptsmall}
    \end{subfigure}
    \begin{subfigure}[b]{0.31\textwidth}
        \centering
\includegraphics[width=1\textwidth]{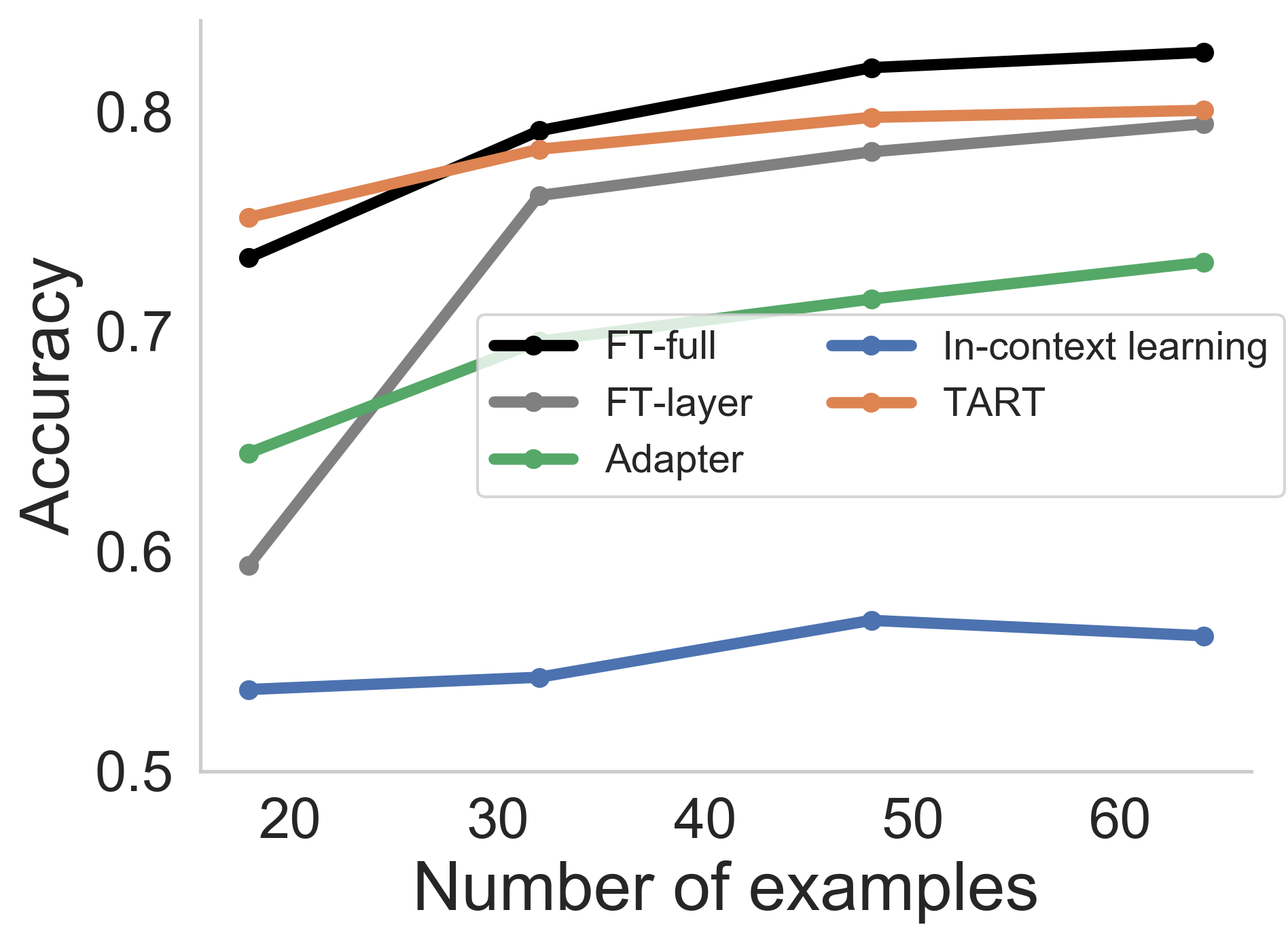}
    \subcaption{\pythiasmall}
    \end{subfigure}
    \begin{subfigure}[b]{0.31\textwidth}
        \centering
\includegraphics[width=1\textwidth]{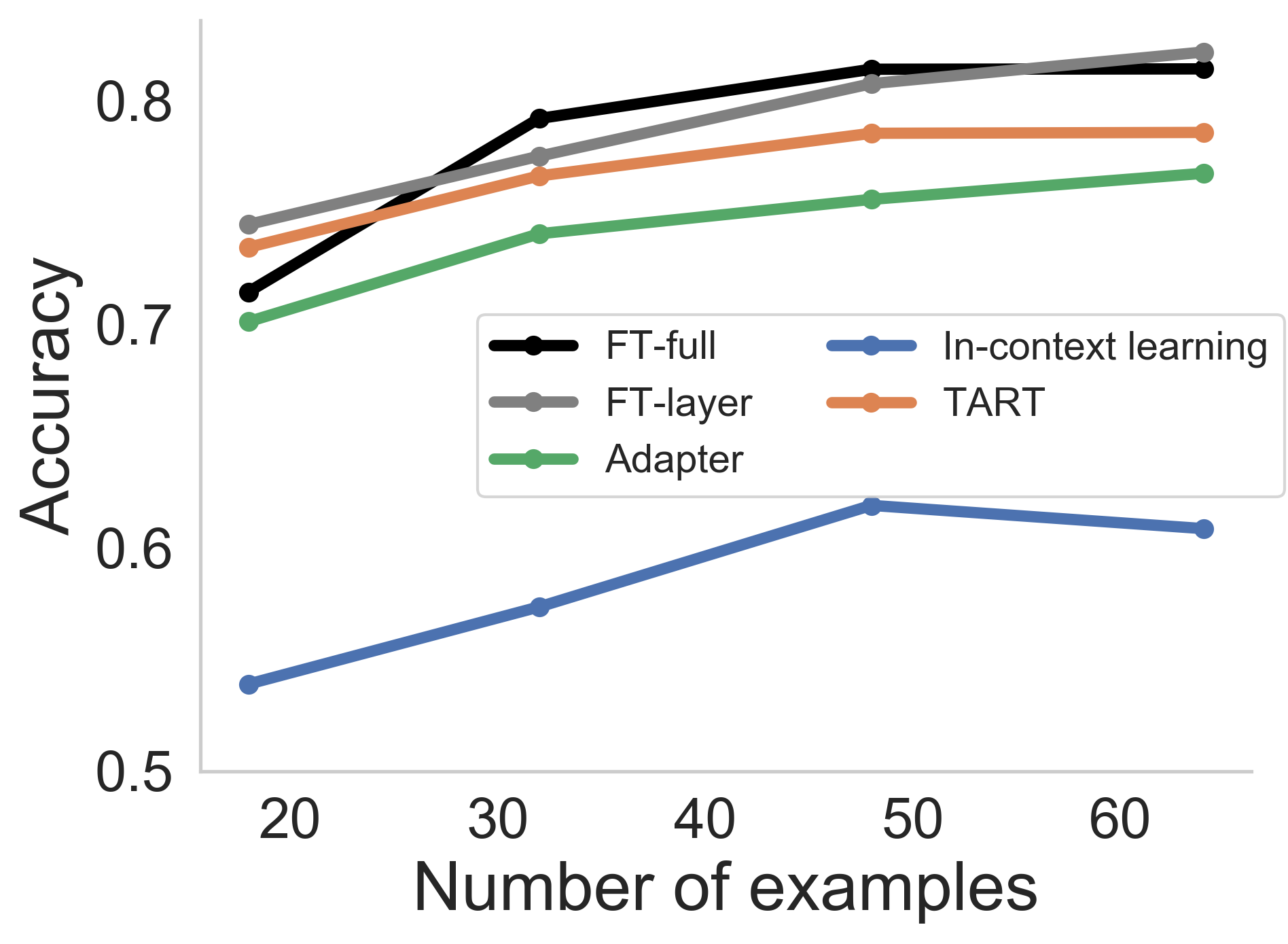}
    \subcaption{\bloomsmall}
    \end{subfigure}
    \caption{\textbf{Comparison of all methods}. \alg significantly improves base in-context learning performance and is competitive with full-finetuning across model families.}
    \label{fig:macro-models}
\end{figure*}

\begin{figure*}[t!]
    \centering
    \begin{subfigure}[b]{1\textwidth}
        \centering
\includegraphics[width=1\textwidth]{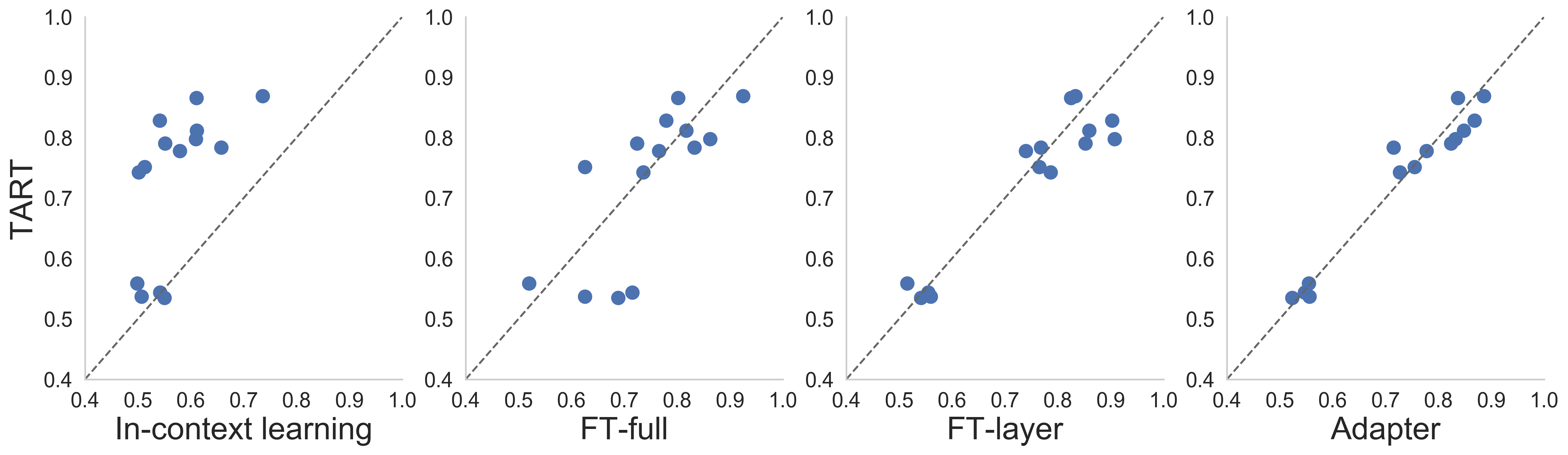}
    \subcaption{Number of examples = 18}
    \vspace{3mm}
    \end{subfigure}
    \begin{subfigure}[b]{1\textwidth}
        \centering
\includegraphics[width=1\textwidth]{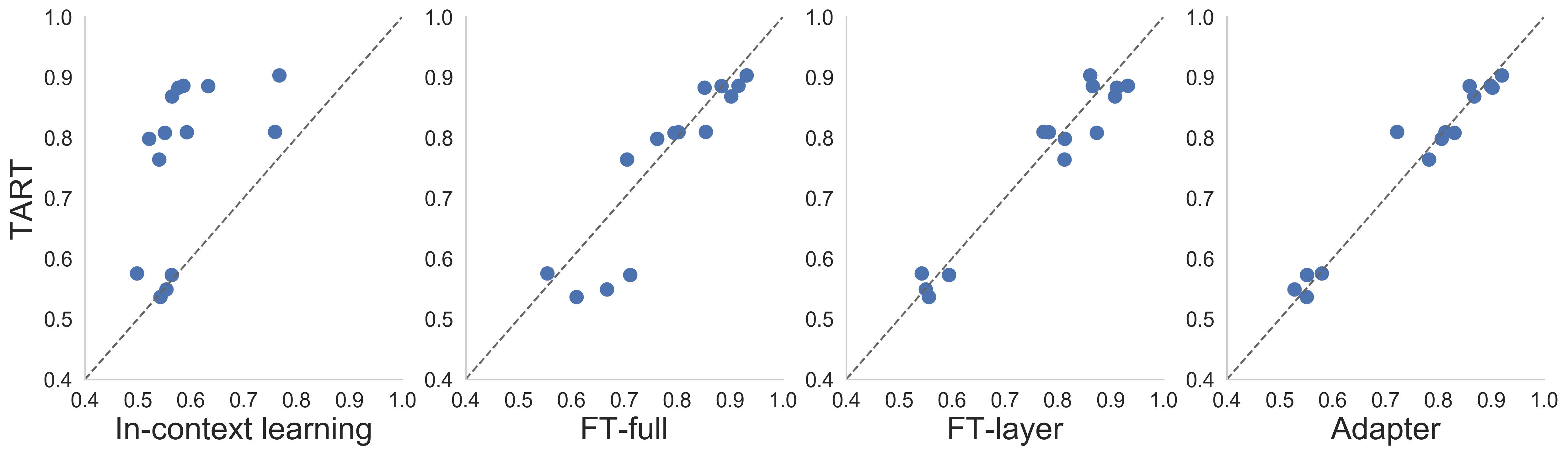}
    \subcaption{Number of examples = 32}
    \vspace{3mm}
    \end{subfigure}
    \begin{subfigure}[b]{1\textwidth}
        \centering
\includegraphics[width=1\textwidth]{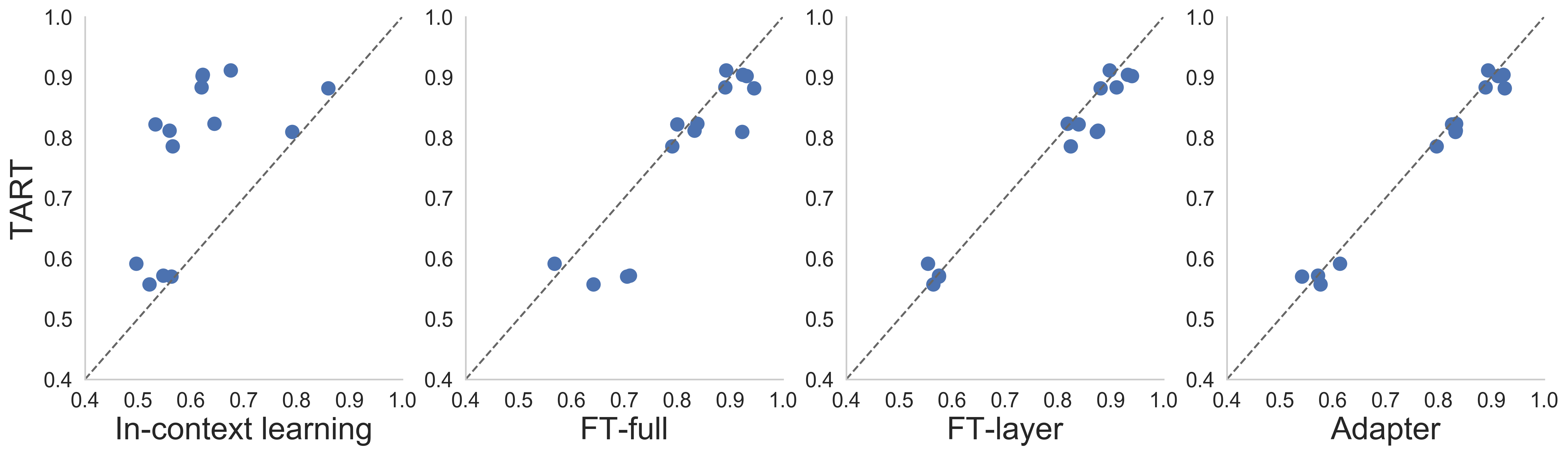}
    \subcaption{Number of examples = 48}
    \vspace{3mm}
    \end{subfigure}
     \begin{subfigure}[b]{1\textwidth}
        \centering
\includegraphics[width=1\textwidth]{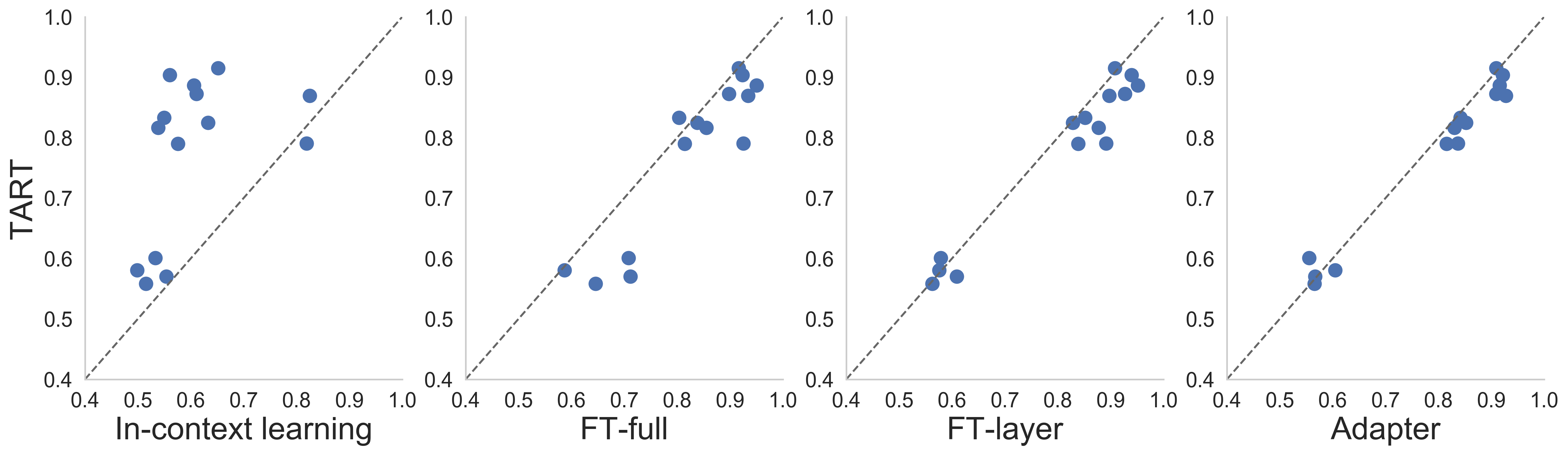}
    \subcaption{Number of examples = 64}
    \end{subfigure}
    \caption{\textbf{Comparison of \alg and task-adaptation approaches (\gptsmall)}. We see that for \gptsmall, \alg outperforms in-context learning and is competitive with full fine-tuning and adapters across all $k$.}
    \label{fig:macro-gpt}
\end{figure*}

\begin{figure*}[t!]
    \centering
    \begin{subfigure}[b]{1\textwidth}
        \centering
\includegraphics[width=1\textwidth]{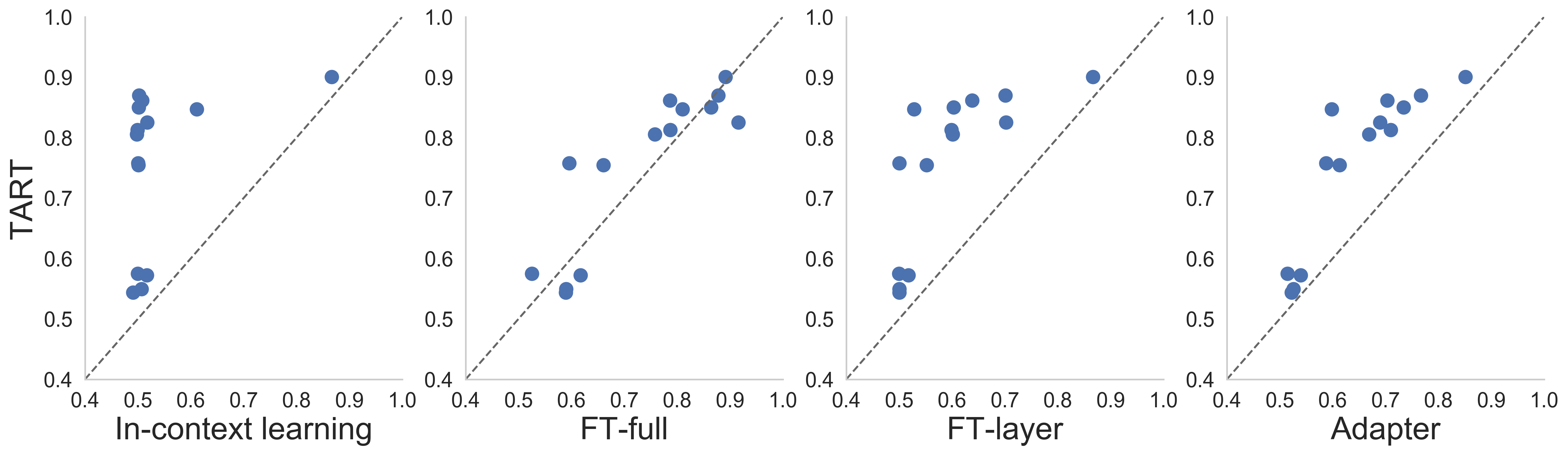}
    \subcaption{Number of examples = 18}
    \vspace{3mm}
    \end{subfigure}
    \begin{subfigure}[b]{1\textwidth}
        \centering
\includegraphics[width=1\textwidth]{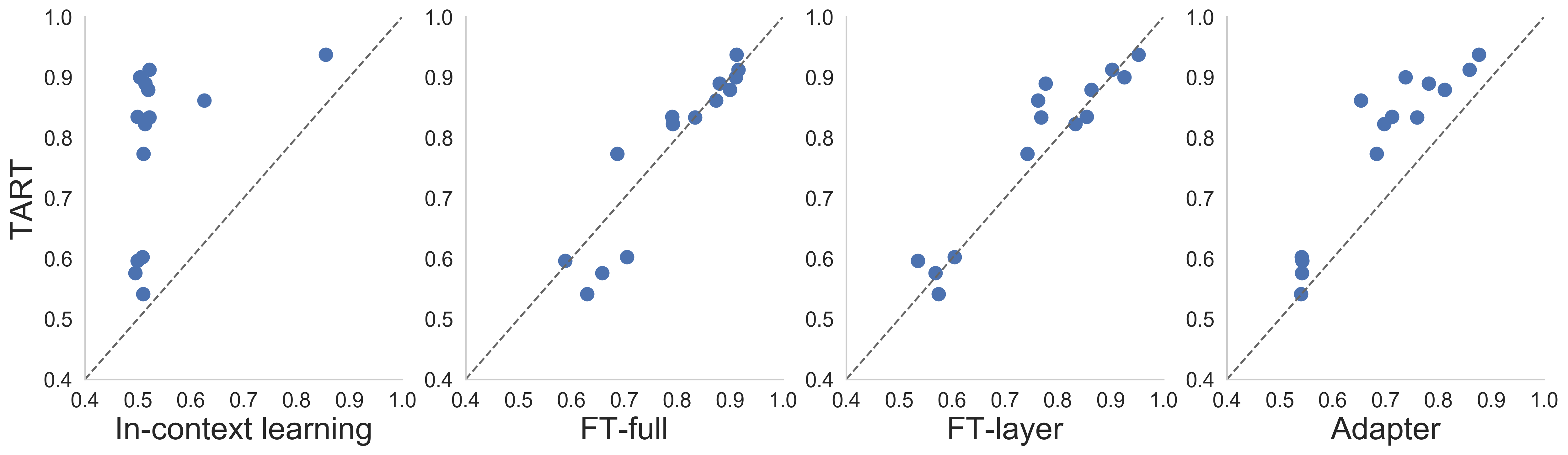}
    \subcaption{Number of examples = 32}
    \vspace{3mm}
    \end{subfigure}
    \begin{subfigure}[b]{1\textwidth}
        \centering
\includegraphics[width=1\textwidth]{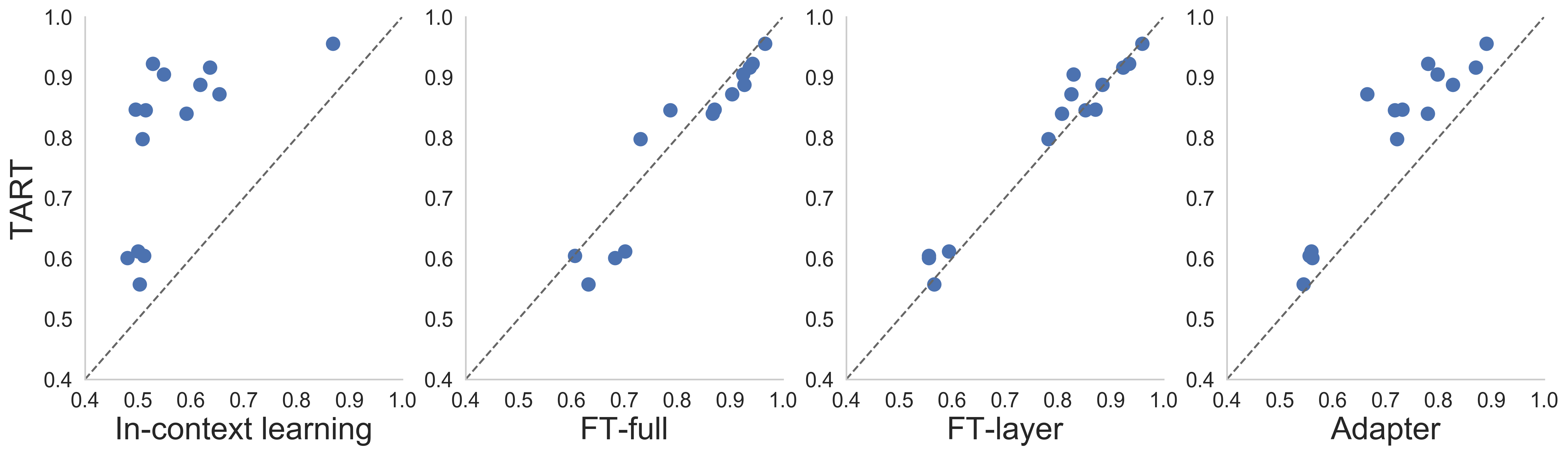}
    \subcaption{Number of examples = 48}
    \vspace{3mm}
    \end{subfigure}
     \begin{subfigure}[b]{1\textwidth}
        \centering
\includegraphics[width=1\textwidth]{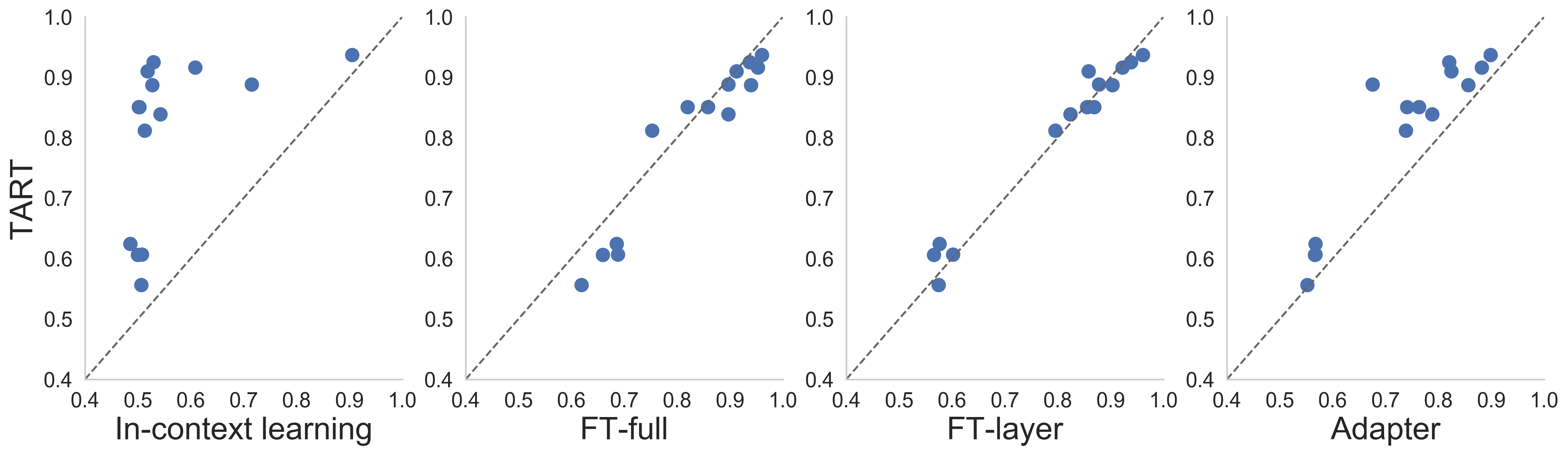}
    \subcaption{Number of examples = 64}
    \end{subfigure}
    \caption{\textbf{Comparison of \alg and task-adaptation approaches (\pythiasmall)}. We see that for \pythiasmall, \alg outperforms in-context learning and adapters and is competitive with full fine-tuning across all $k$.}
    \label{fig:macro-pythia}
\end{figure*}

\begin{figure*}[t!]
    \centering
    \begin{subfigure}[b]{1\textwidth}
        \centering
\includegraphics[width=1\textwidth]{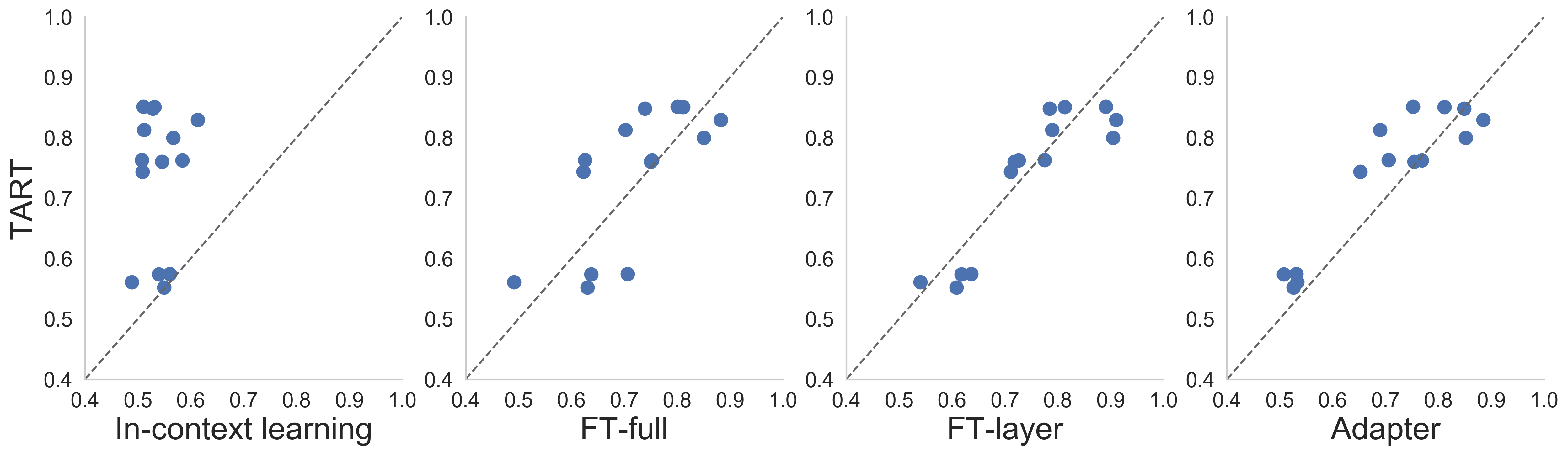}
    \subcaption{Number of examples = 18}
    \vspace{3mm}
    \end{subfigure}
    \begin{subfigure}[b]{1\textwidth}
        \centering
\includegraphics[width=1\textwidth]{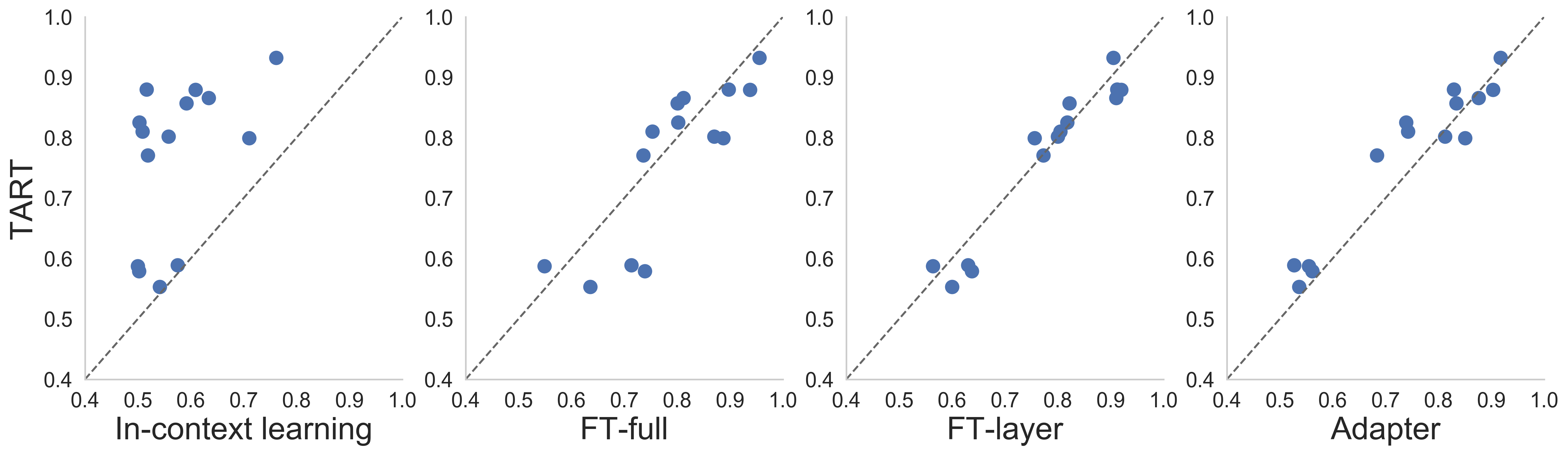}
    \subcaption{Number of examples = 32}
    \vspace{3mm}
    \end{subfigure}
    \begin{subfigure}[b]{1\textwidth}
        \centering
\includegraphics[width=1\textwidth]{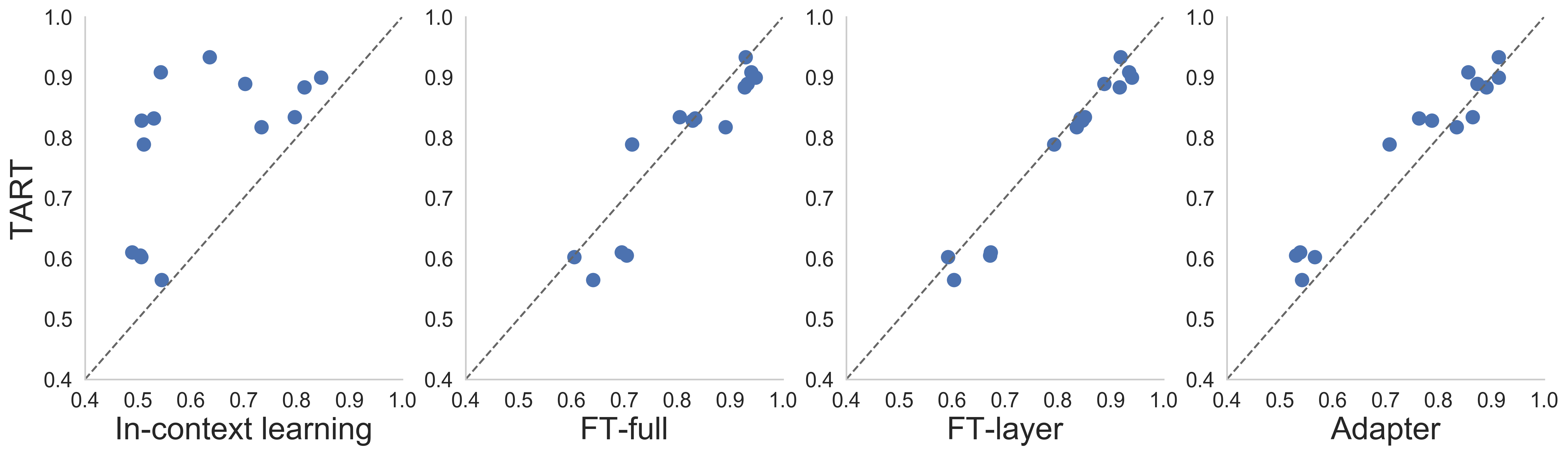}
    \subcaption{Number of examples = 48}
    \vspace{3mm}
    \end{subfigure}
     \begin{subfigure}[b]{1\textwidth}
        \centering
\includegraphics[width=1\textwidth]{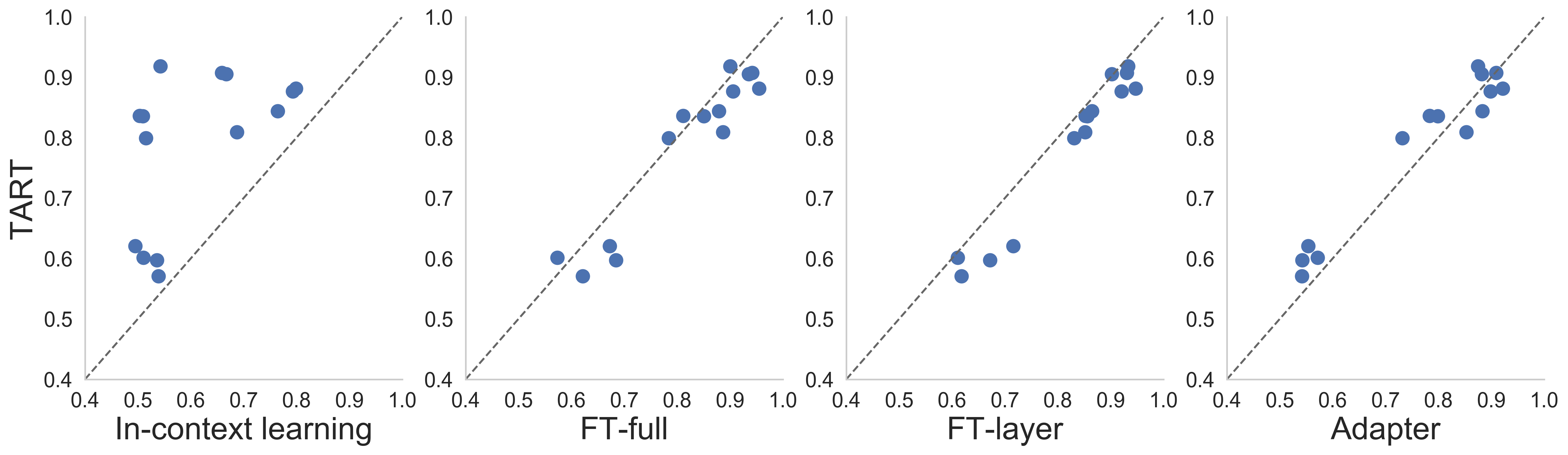}
    \subcaption{Number of examples = 64}
    \end{subfigure}
    \caption{\textbf{Comparison of \alg and task-adaptation approaches (\bloomsmall)}.  We see that for \bloomsmall, \alg outperforms in-context learning and adapters and is competitive with full fine-tuning across all $k$.}
    \label{fig:macro-bloom}
\end{figure*}

\newpage
\subsubsection{Baseline methods}\label{app:ft-arch}
For each dataset, we compare \alg to 4 baseline task-adaptation methods: 1) in-context learning, 2) full fine-tuning, 3) last layer fine-tuning,  and 4) adapters.  The last layer fine-tuning and the adapters are trained as follows:
\begin{itemize}
    \item Last layer fine-tuning: Freeze all layers of transformer but the final transformer block and the language modeling head.
    \item Adapter: Combine a frozen LLM base transformer model with a trainable adapter head---an MLP composed of a single linear layer followed by non-linearity.
\end{itemize}

\paragraph{Hyperparameter search.} For each baseline, we perform an extensive hyperparameter search over number of epochs and learning rate for each dataset in order to optimize performance. We search over a range of learning rates (1e-3, 1e-4, 3e-5, 1e-5, 8e-6), and range of epochs (5, 10, 15, 20, 50). For all models < 1B parameters, we use a batch size of 1. For all models > 1B parameters, we use a batch size of 8. We use these same batch sizes at evaluation time. We perform our hyperparameter searches with a fixed number of train samples (64). We run our hyperparameter searches over 3 random seeds.

\subsection{NL benchmarks}\label{app:nl-eval}
In this section, we provide additional results deferred from Section~\ref{sec:exp} on the NLP benchmark evaluations, RAFT evaluations and demonstration of \alg's data-scalability. 

\subsubsection{Performance on benchmark datasets}
Figure~\ref{fig:macro-models} shows the performance of the baseline methods with \alg averaged across the suite of 14 datasets. \alg, while being task-agnostic, shows similar performance quality to task-specific approaches across the different model families, and consistently outperforms in-context learning. 

Figures~\ref{fig:macro-gpt},~\ref{fig:macro-pythia}, and~\ref{fig:macro-bloom} show the scatter plots of the accuracies of \alg with the baseline methods across datasets and different values of in-context examples $k$. An interesting observation is that as the number of examples $k$ increases from $18$ to $64$, the performance of fine-tuning improves at a better rate than that of \alg.

\subsubsection{Real-world Annotated Few-shot Tasks (RAFT) evaluation} 
For our evaluations on the RAFT benchmark~\citep{alexraft}, we follow the protocol (same train and test sets) used in HELM benchmark. The HELM benchmark~\citep{liang2022holistic} contains the evaluation results for many open and closed models enabling us to accurately compare the performance of \alg with other models. We evaluate \alg on all RAFT binary classification datasets (twitter-complaints, neurips-impact-statement-risks, overulling, ade-corpusv2, tweet-eval-hate, terms-of-service, tai-safety-research) with the exception of systematic-review-inclusion which contains zero positive samples in the train set. \alg requires at least one example of each class in the training set. Table~\ref{tab:raft} contains a detailed performance comparison of \alg with respect to other models. \alg when combined with \gptsmall is able to outperform \bloomhuge  and is competitive with \opthuge and \gpt, all of which have 1000x more parameters.

\begin{figure*}[t!]
    \centering
    \begin{subfigure}[b]{0.31\textwidth}
        \centering
\includegraphics[width=1\textwidth]{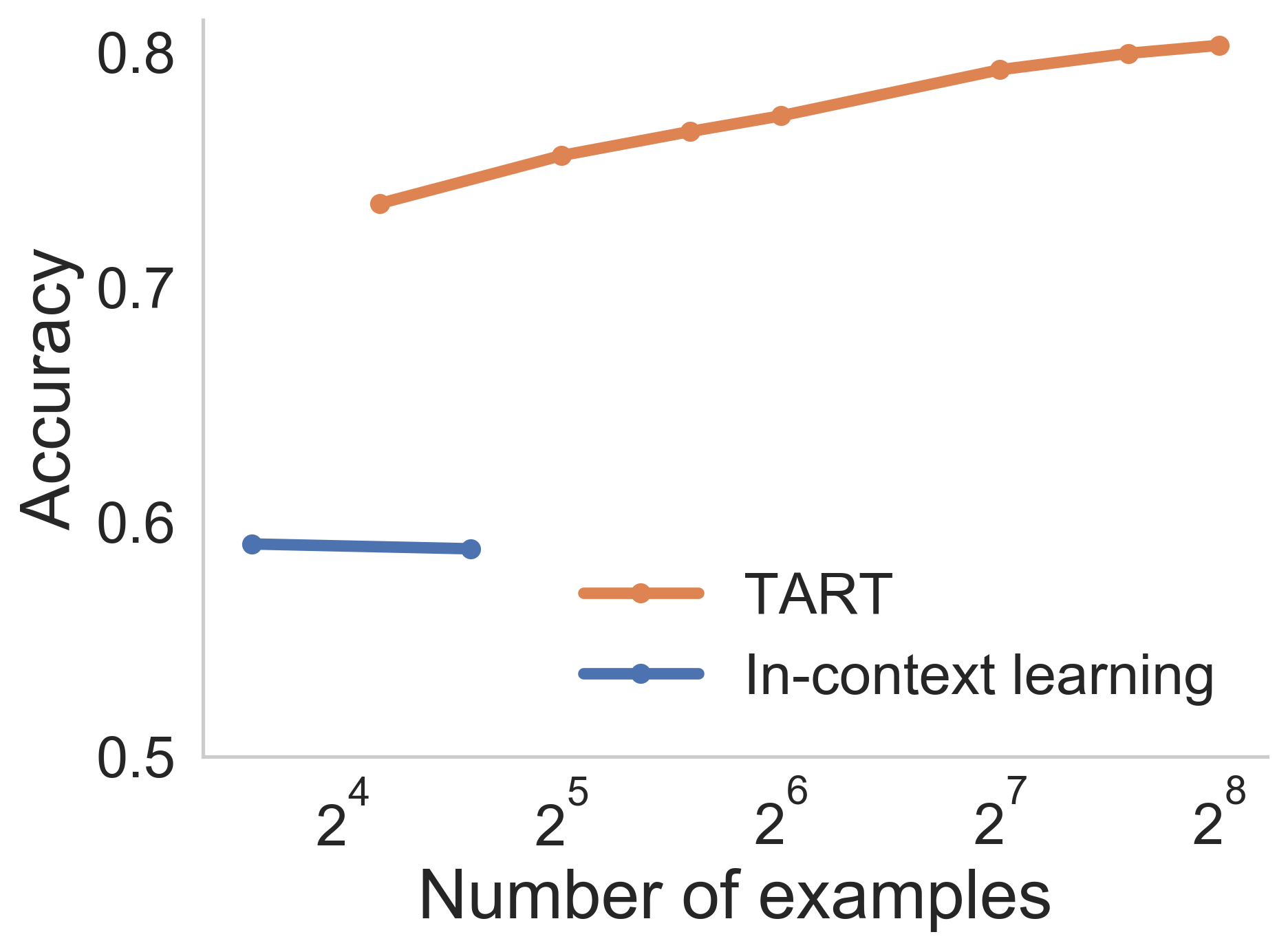}
    \subcaption{\gptsmall}
    \end{subfigure}
    \begin{subfigure}[b]{0.31\textwidth}
        \centering
\includegraphics[width=1\textwidth]{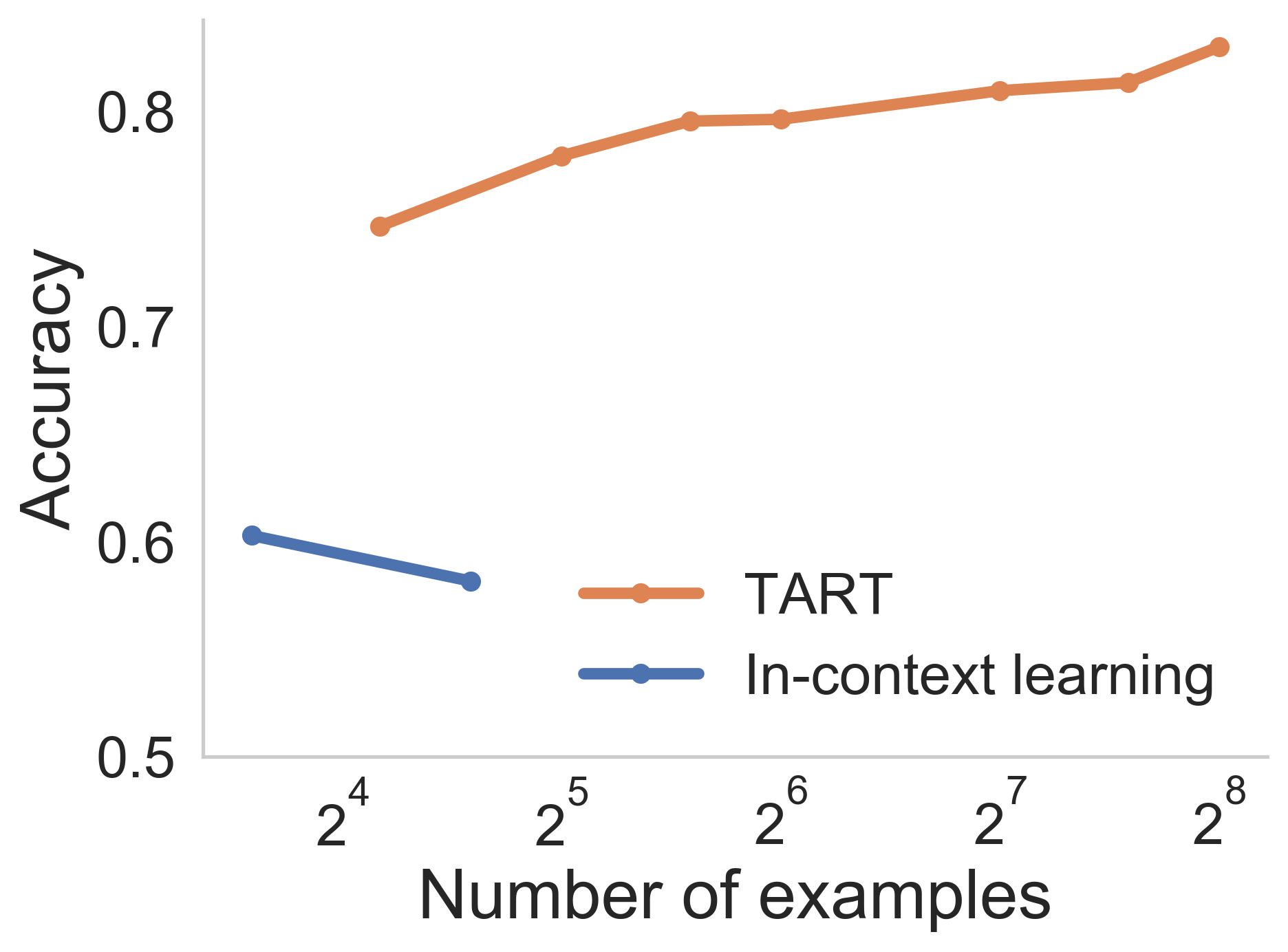}
    \subcaption{\pythiasmall}
    \end{subfigure}
    \begin{subfigure}[b]{0.31\textwidth}
        \centering
\includegraphics[width=1\textwidth]{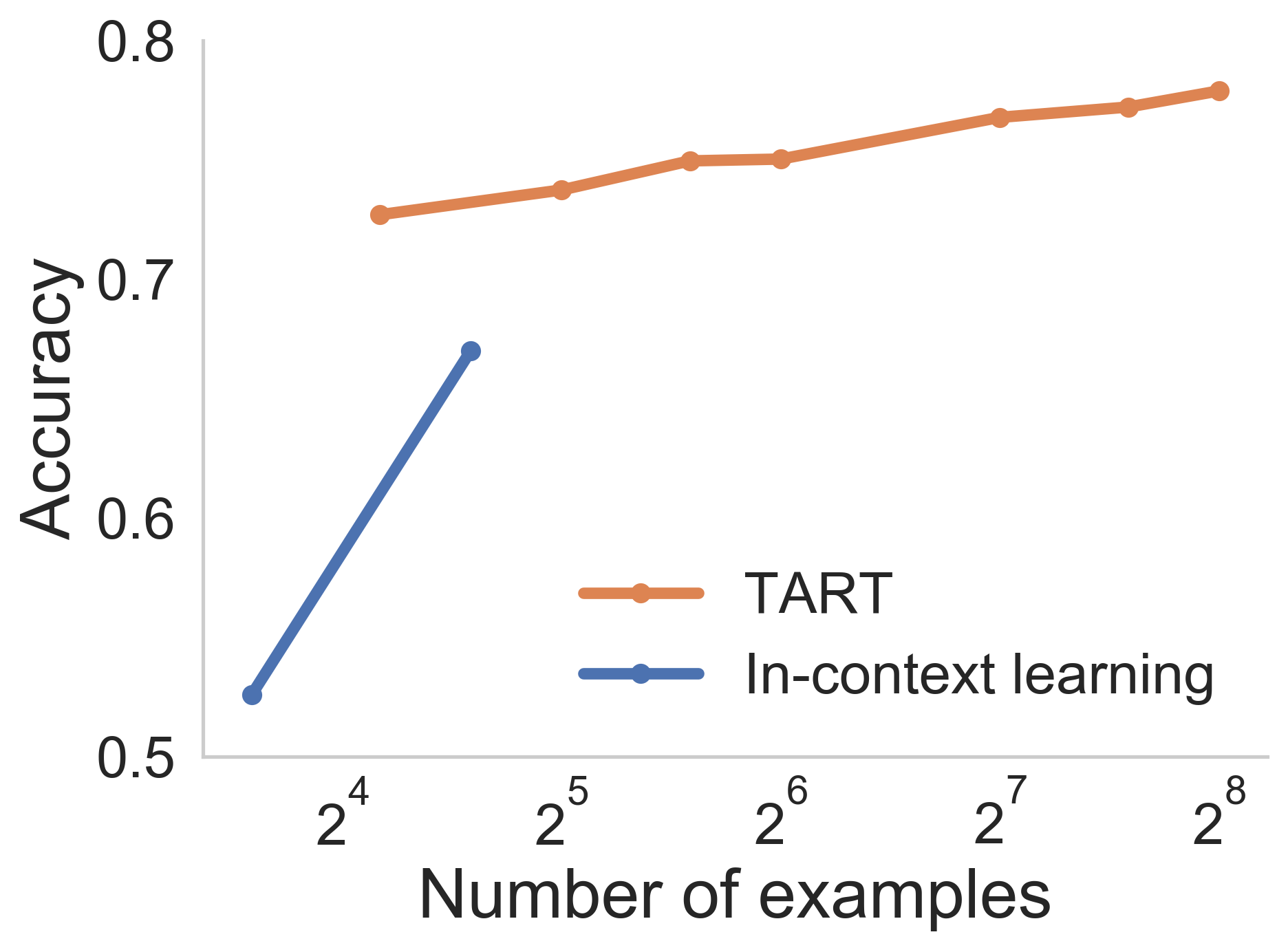}
    \subcaption{\bloomsmall}
    \end{subfigure}
    \caption{\textbf{Beyond context window constraints}. Performance comparison with respect to number of in-context examples. Base in-context learning is bound with respect to total numbers of examples and performance saturates. \alg is not bound by context length, and performance continues to scale as number of examples increases.}
    \label{fig:long-context}
\end{figure*}

\subsubsection{Beyond context length: \alg is data-scalable} 
\paragraph{Setup.} For these evaluations, we use the a subset of 6 datasets: AG-News-0, DBPedia-0, SST, SMS Spam, Youtube and Rotten Tomatoes. We evaluate the performance of \alg over $k$=[18, 32, 48, 64, 128, 192, 256] where $k$ is the number of in-context examples. When evaluating our base models, we evaluate over $k$=[8, 24]---values of $k$ that maximize the context window. We use a lower-bound of 8 given that the maximum input sequence length in the training set for AG News is 256. With such a sequence length, the maximum number of in-context examples that fit in the context-window is 8, hence the lower bound. 

\begin{figure}[t!]
    \centering
\includegraphics[width=0.90\textwidth]{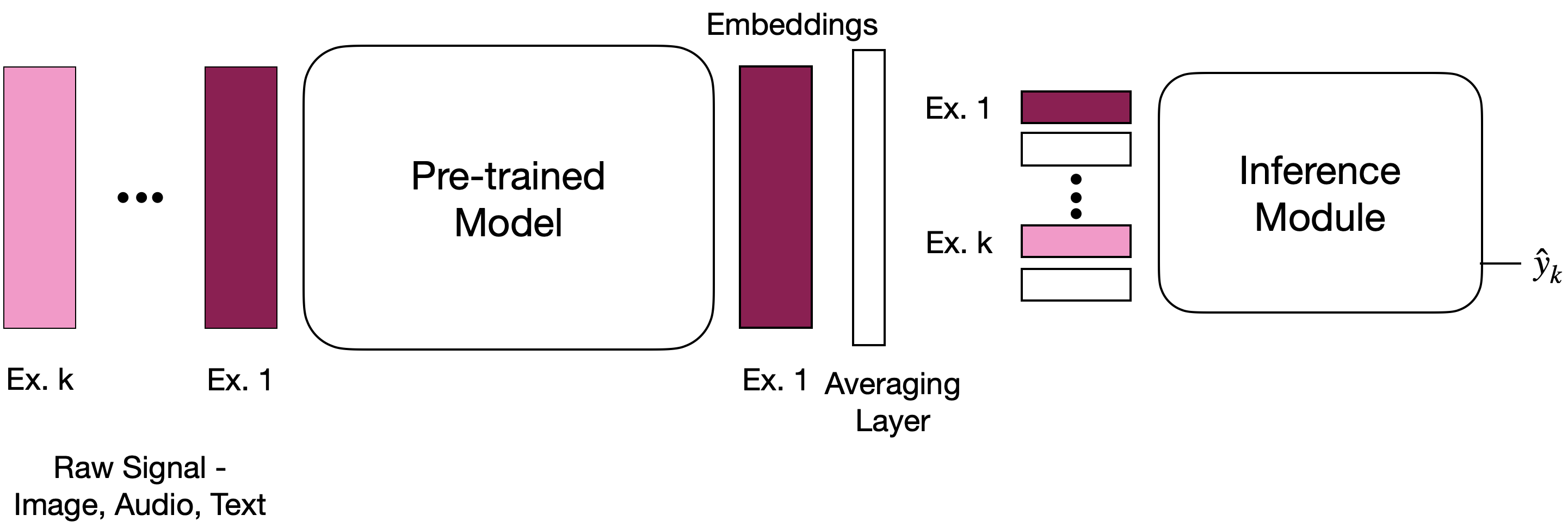}
    \caption{\textbf{Stream embeddings}. Another protocol for generating representations for in-context examples where each example is embedded by the base model separately.}
    \label{fig:stream-embed}
\end{figure}

\paragraph{Embeddings.} For these evaluations, we use what we call ``streaming'' embeddings (see Figure~\ref{fig:stream-embed}). In this setup, we use the context window of the LLM to encode a single example at a time. The final embeddings are then averaged and used in-context with \alg's reasoning module. This is in contrast to the vanilla and LOO embeddings which use multiple examples in-context with the base LLM to obtain the embeddings. 

\paragraph{Evaluation.} Figure~\ref{fig:long-context} shows the performance of base in-context learning with \alg across the three different model families. Observe that while in-context learning is bottlenecked by the context window of the base LLM, \alg is able to learn from 10x more examples and exhibits an increasing trend in accuracy with number of examples across models.



\subsection{Extension to other modalities: \alg is domain-agnostic!}\label{app:modalities}
We begin by providng a description of the datasets we used to evaluate \alg on audio and vision tasks, and then provide additional results comparing our algorithm with baselines.

\subsubsection{Dataset details}
For audio classification, we use the Speech Commands (Version 0.01) dataset~\citep{speechcommandsv2}. Speech Commands is a multi-class classification task where the task is to detect preregistered keywords by classifying utterances into a predefined set of words. We  construct a 3 binary classification task over the keywords ``stop'' and ``go'', ``up'' and ``down'', and ``yes'' and ``no'' (see Table~\ref{tab:modalities} for more details).

For image classification, we use CIFAR-10~\cite{Krizhevsky09learningmultiple} and MNIST~\citep{lecun2010mnist}. Both tasks are multi-class classification tasks. We create 3 binary classification tasks for each of the datasets. For CIFAR-10 the tasks are: airplane vs. bird, bird vs. horse, and ship vs. automobile. For MNIST the tasks are: 0 vs. 8, 1 vs. 6 and 2 vs. 4. See Table~\ref{tab:modalities} for more details.

For both the audio and image datasets, we sample a class-balanced set of 256 samples from the training set. For the test sets, we filter the original test sets to only include samples of the two classes we are learning to predict for (i.e., airplane and bird for CIFAR10 and 0 and 8 for MNIST). 

\begin{table}[t!]
\centering
\begin{tabular}{llrr}
\toprule
                 Dataset & Modality &  Train size &  Test size \\
\midrule
              MNIST (0 vs. 8) &    image &                 256 &               1954 \\
              MNIST (1 vs. 6) &    image &                 256 &               2093 \\
              MNIST (2 vs. 4) &    image &                 256 &               2014 \\
Speech Commands (stop vs. go) &    audio &                 256 &                500 \\
Speech Commands (up vs. down) &    audio &                 256 &                508 \\
Speech Commands (yes vs. no) &    audio &                 256 &                525 \\
    CIFAR-10 (airplane vs. bird) &    image &                 256 &               2000 \\
    CIFAR-10 (bird vs. horse) &    image &                 256 &               2000 \\
    CIFAR-10 (ship vs. automobile) &    image &                 256 &               2000 \\
\bottomrule
\end{tabular}
\vspace{2mm}
\caption{Dataset statistics for all audio and image evaluation datasets.}
\label{tab:modalities}
\end{table}

\subsubsection{Algorithms for comparison}
For these evaluations, we use the ``streaming embeddings'' described in Figure~\ref{fig:stream-embed} to obtain the embedding for \alg. We evaluate over $k$=[18, 32, 48, 64, 128, 256].

We compare against two baseline task-adaptation methods: 1) full fine-tuning and 2) adapters. We use the same architectures as described in Appendix~\ref{app:ft-arch}. For vision tasks, we use Google's 307M parameter pretrained Vision Transformer (ViT) model~\cite{wu2020visual}: \vit. For audio tasks, we use OpenAI's 1.5B parameter pretrained Whisper model~\citep{radford2022whisper}: \whisper.

\paragraph{Hyperparameter search} For each baseline, we perform an extensive hyperparameter search over number of epochs and learning rate for each dataset in order to optimize performance. We search over a range of learning rates (1e-3, 5e-04, 1e-4, 5e-5, 1e-5, and 8e-6) and a range of epochs (5, 10, 15 and 20). For all models we use a batch size of 1. We perform our hyperparameter searches for a fixed number of train samples (128) and run our hyperparameter searches over 3 random seeds. 

\begin{figure*}[t!]
    \centering
    \begin{subfigure}[b]{0.45\textwidth}
        \centering
\includegraphics[width=1\textwidth]{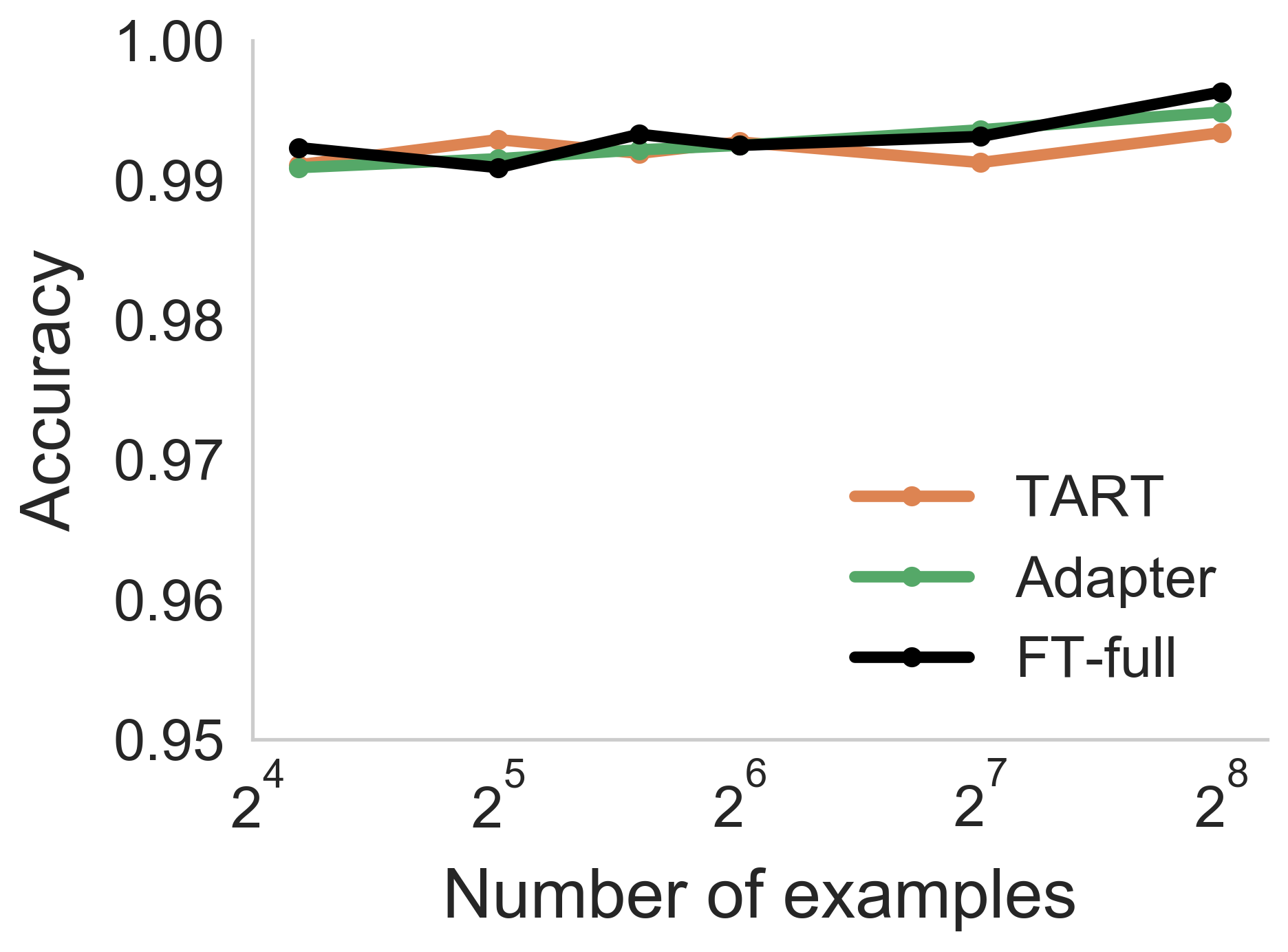}
    \subcaption{MNIST (1 vs. 6)}
    \end{subfigure}
    \begin{subfigure}[b]{0.45\textwidth}
        \centering
\includegraphics[width=1\textwidth]{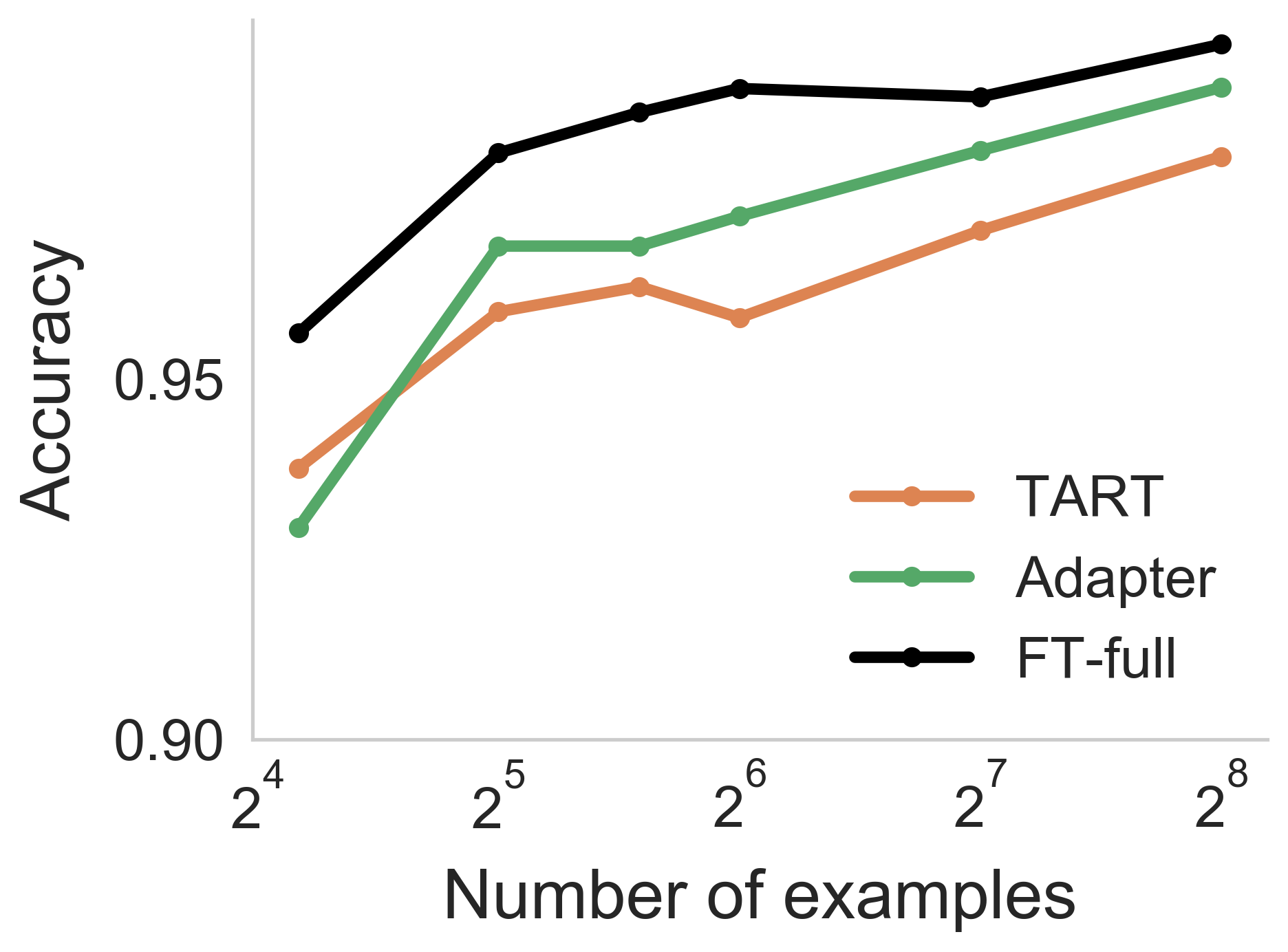}
    \subcaption{MNIST (2 vs. 4)}
    \end{subfigure}
    \caption{\textbf{Additional MNIST binary classification tasks}. TART is competitive with task-specific full fine-tuning and adapters.}
    \label{fig:modalitity-vision-MNIST}
\end{figure*}

\begin{figure*}[t!]
    \centering
    \begin{subfigure}[b]{0.45\textwidth}
        \centering
\includegraphics[width=1\textwidth]{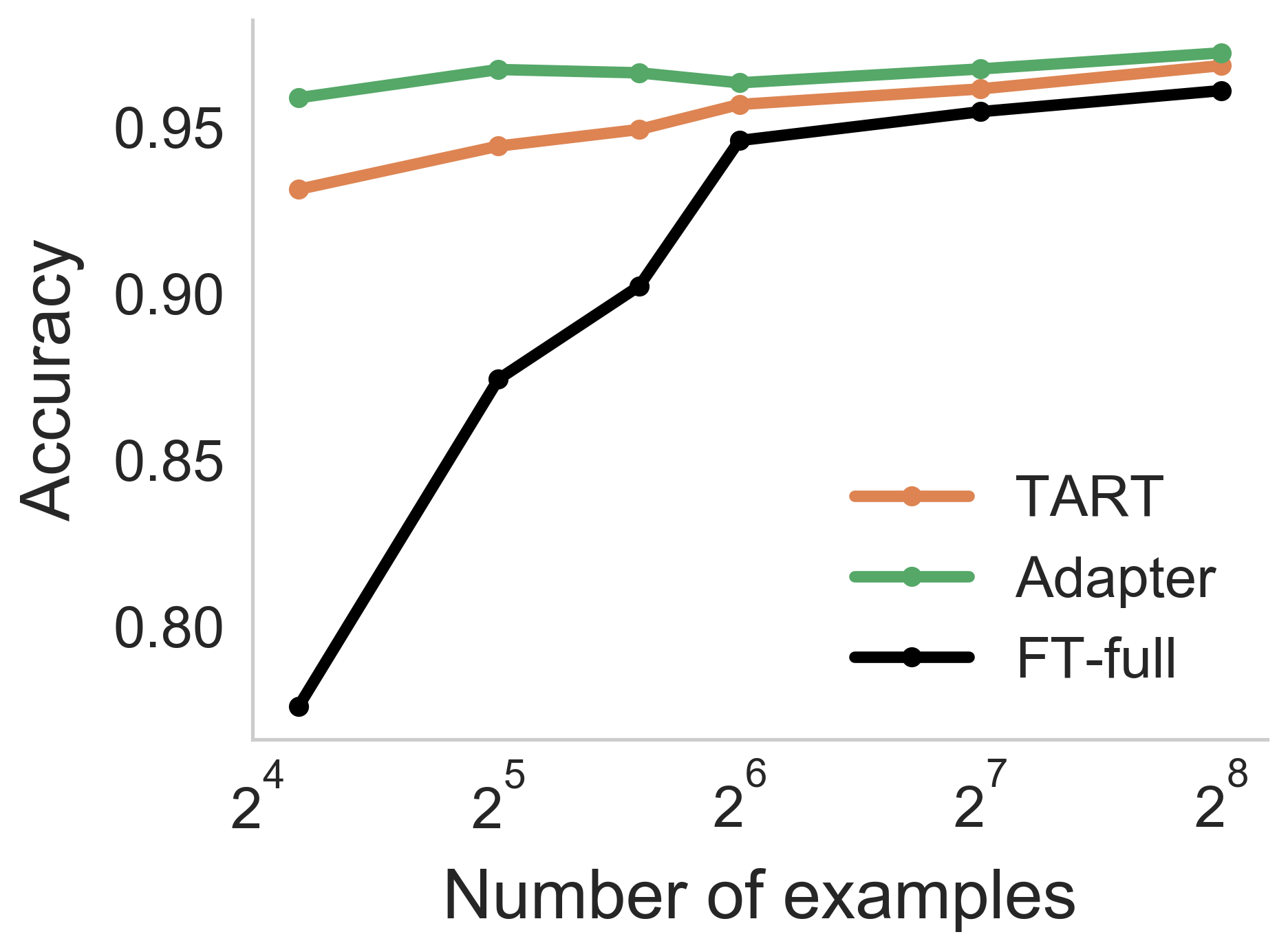}
    \subcaption{CIFAR-10 (ship vs. automobile)}
    \end{subfigure}
    \begin{subfigure}[b]{0.45\textwidth}
        \centering
\includegraphics[width=1\textwidth]{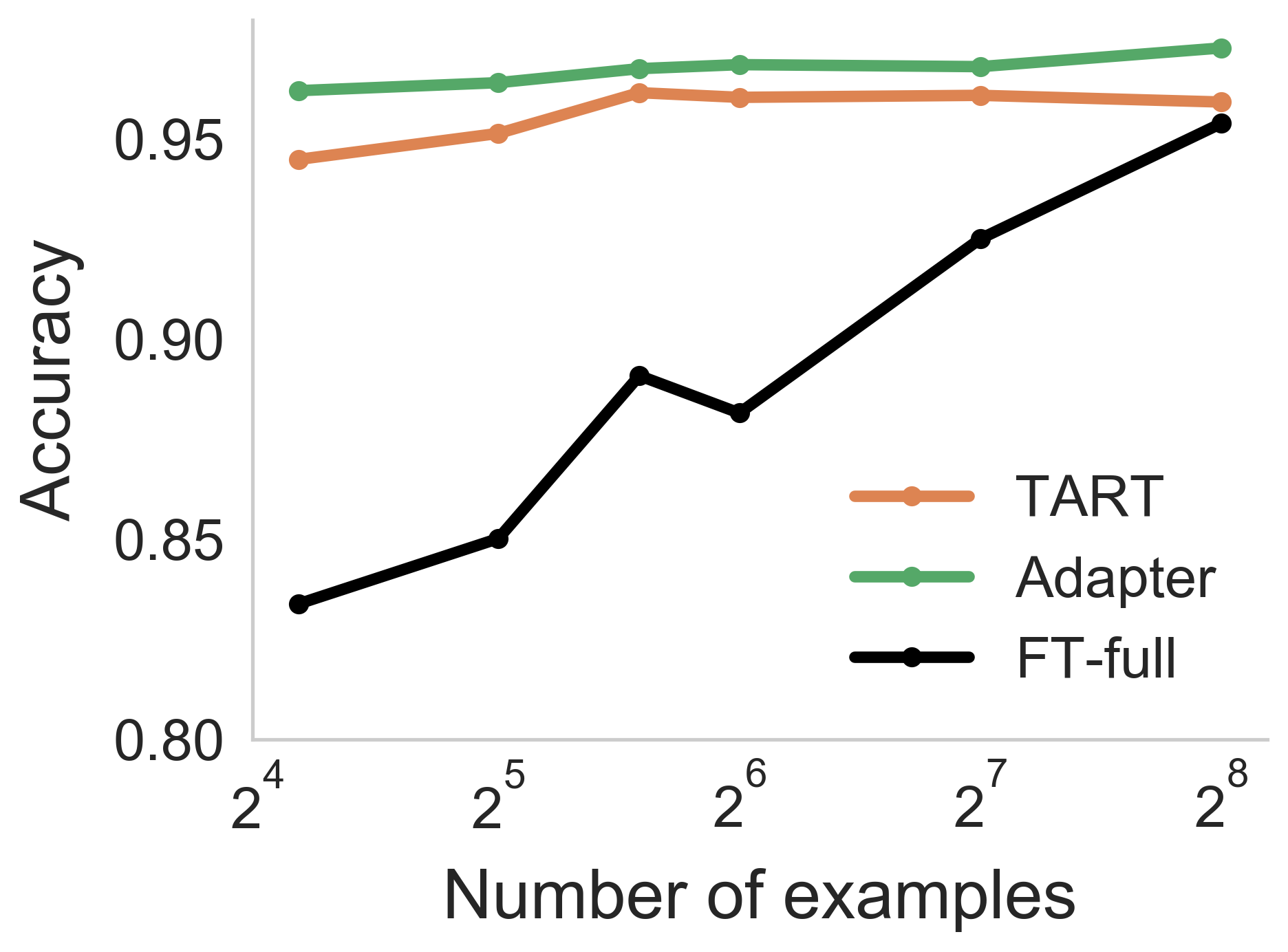}
    \subcaption{CIFAR-10 (bird vs. horse)}
    \end{subfigure}
    \caption{\textbf{Additional CIFAR-10 binary classification tasks}. TART is competitive with task-specific full fine-tuning and adapters.}
    \label{fig:modalitity-vision-CIFAR}
\end{figure*}

\begin{figure*}[t!]
    \centering
    \begin{subfigure}[b]{0.45\textwidth}
        \centering
\includegraphics[width=1\textwidth]{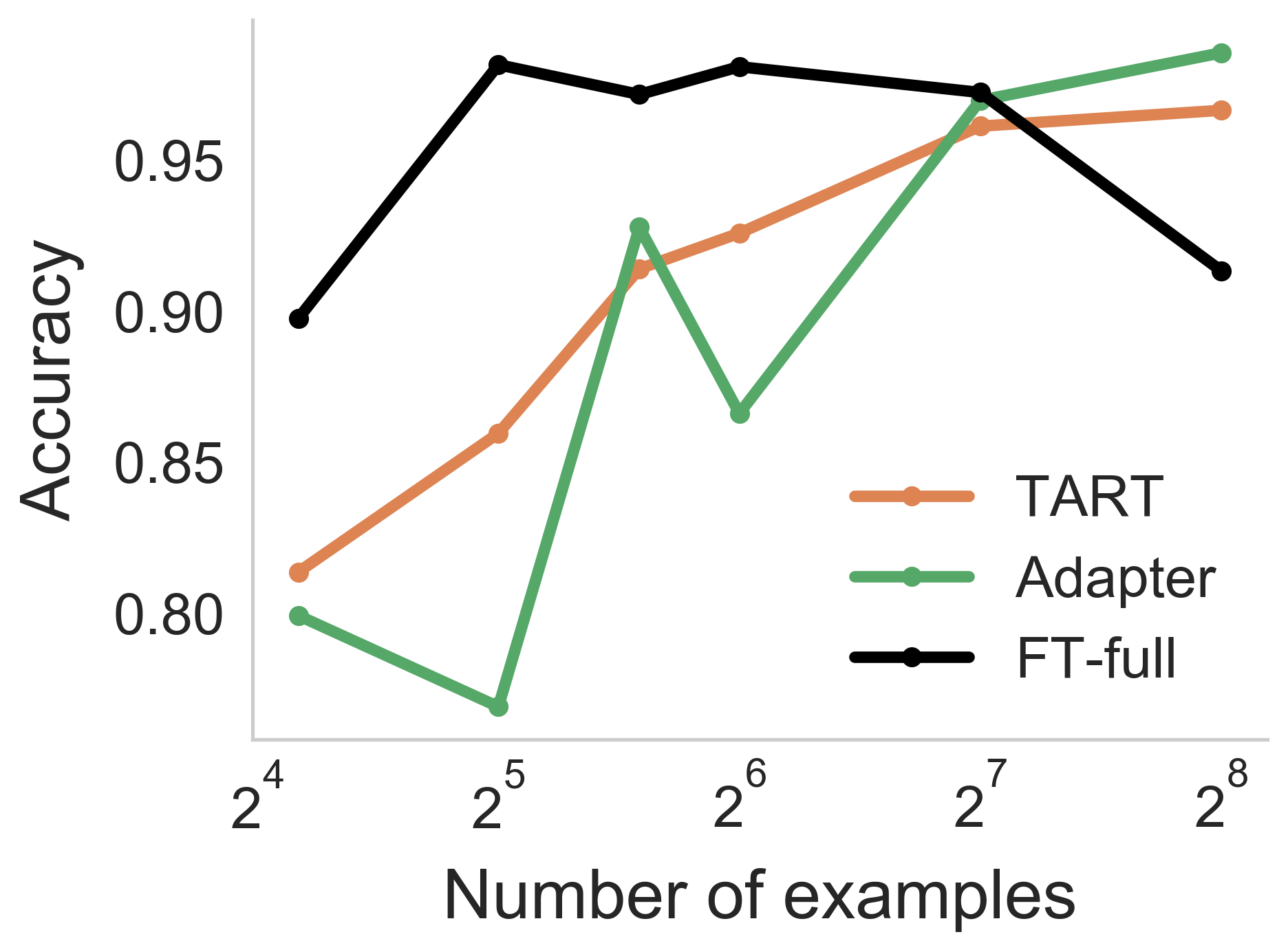}
    \subcaption{Speech Commands (yes vs. no)}
    \end{subfigure}
    \begin{subfigure}[b]{0.45\textwidth}
        \centering
\includegraphics[width=1\textwidth]{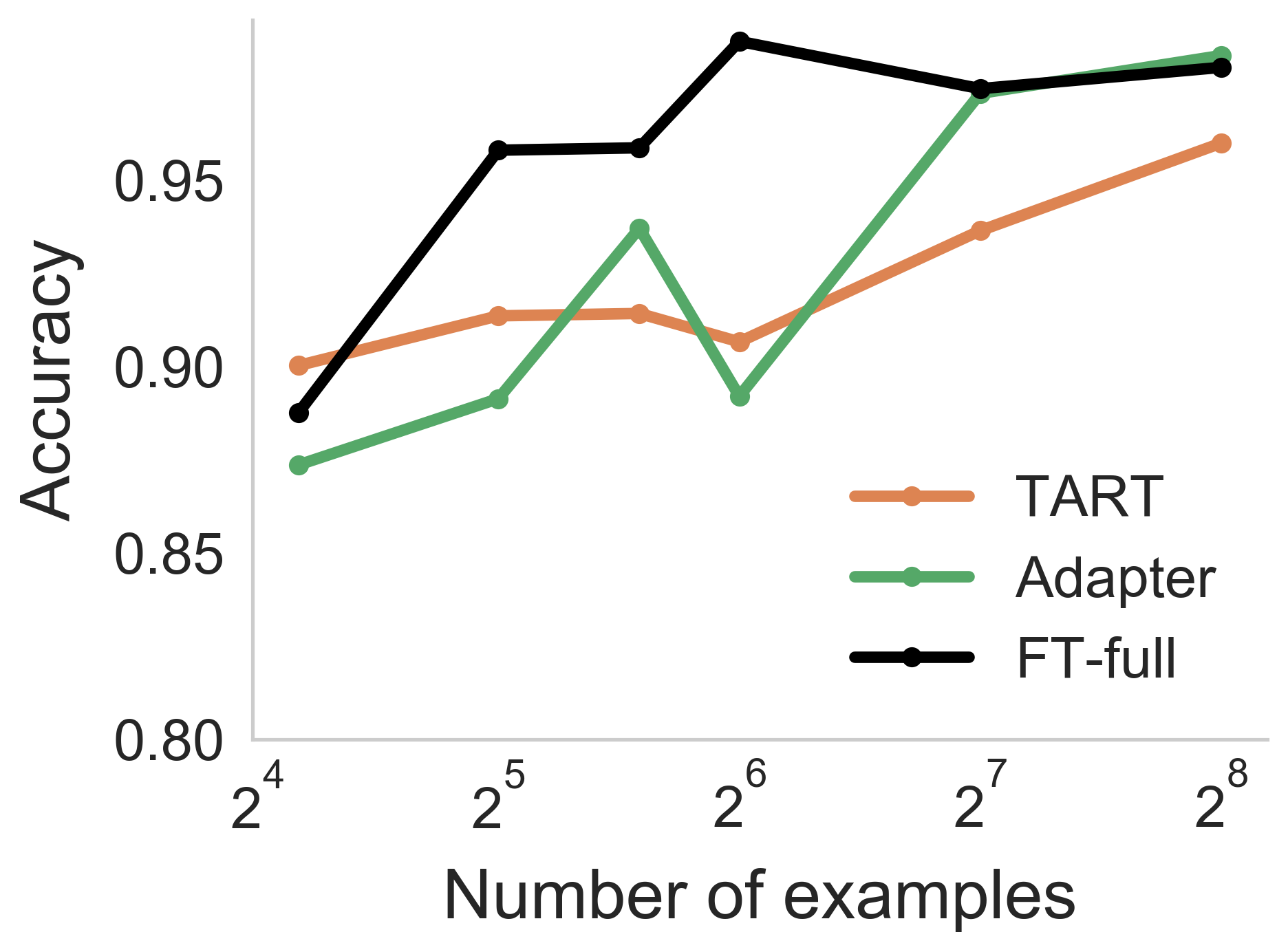}
    \subcaption{Speech Commands  (up vs. down)}
    \end{subfigure}
    \caption{\textbf{Additional Speech Commands binary classification tasks}. TART is competitive with task-specific full fine-tuning and adapters.}
    \label{fig:modalitity-vision-speech}
\end{figure*}

\subsubsection{Evaluation}
We plot the accuracy as a function of the number of examples for \alg, fine-tuning and adapter in Figure~\ref{fig:modalitity-vision-MNIST} (MNIST), Figure~\ref{fig:modalitity-vision-CIFAR} (CIFAR-10), and Figure~\ref{fig:modalitity-vision-speech} (Speech Commands). \alg is competitive with both these baselines, showing how task-agnostic methods can compete with task-specific adaptation methods across different modalities. 

\subsection{Generalization across architectures}

\begin{figure}[t!]
    \centering
    \begin{subfigure}[b]{0.45\textwidth}
        \centering
\includegraphics[width=1\textwidth]{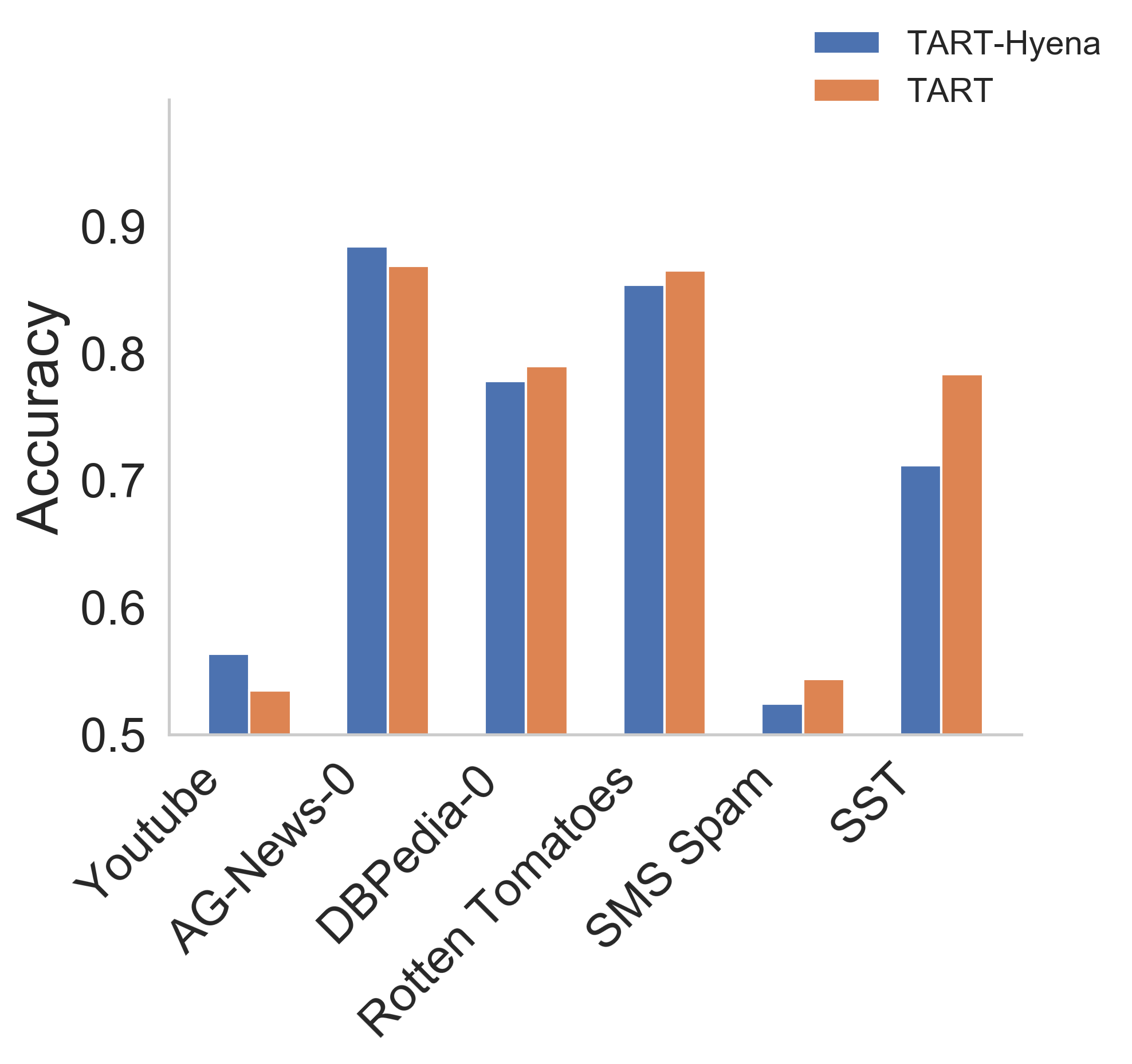}
    \subcaption{k = 18}
    \end{subfigure}
    \begin{subfigure}[b]{0.45\textwidth}
        \centering
\includegraphics[width=1\textwidth]{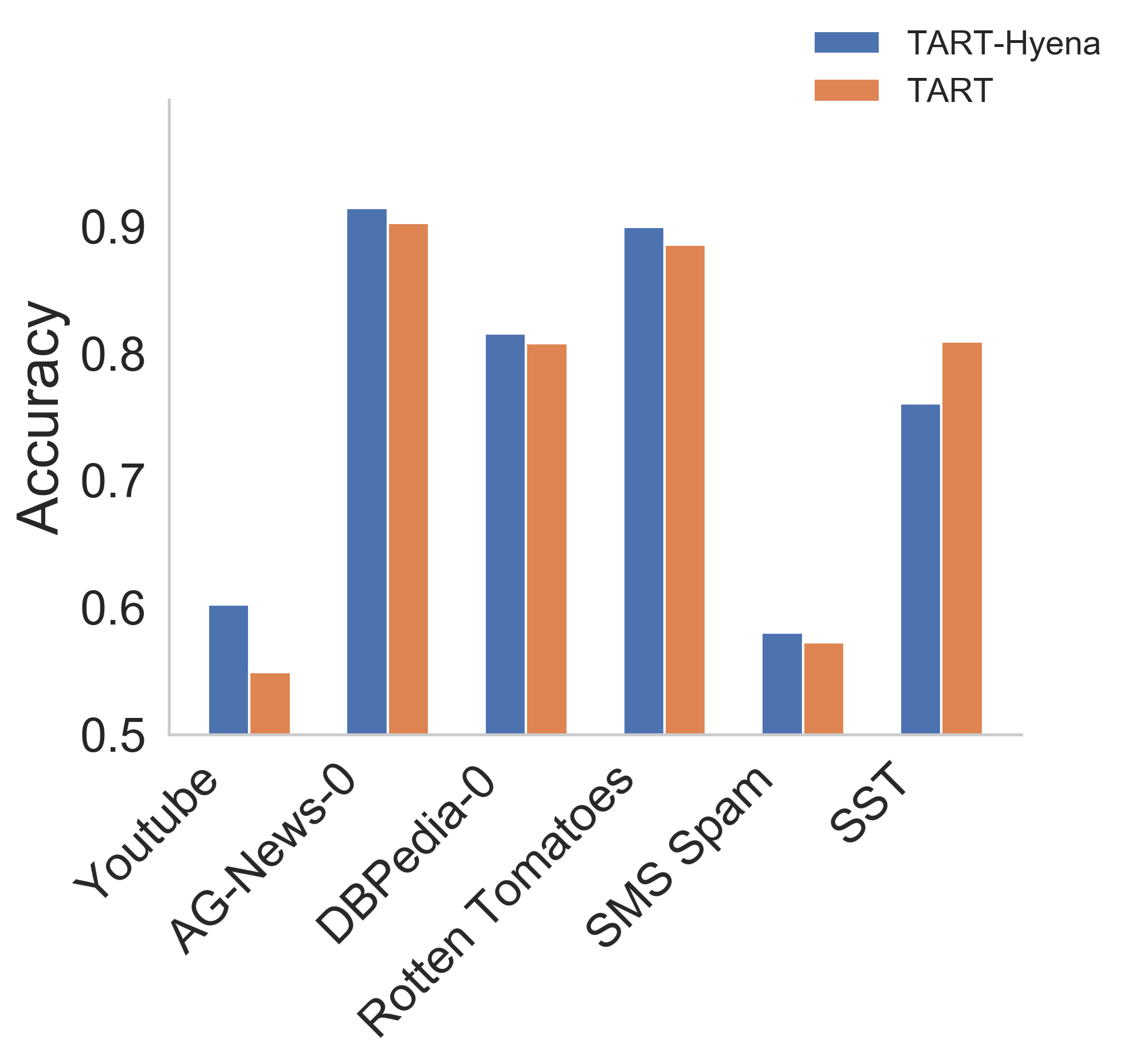}
    \subcaption{k = 32}
    \end{subfigure}
        \begin{subfigure}[b]{0.45\textwidth}
        \centering
\includegraphics[width=1\textwidth]{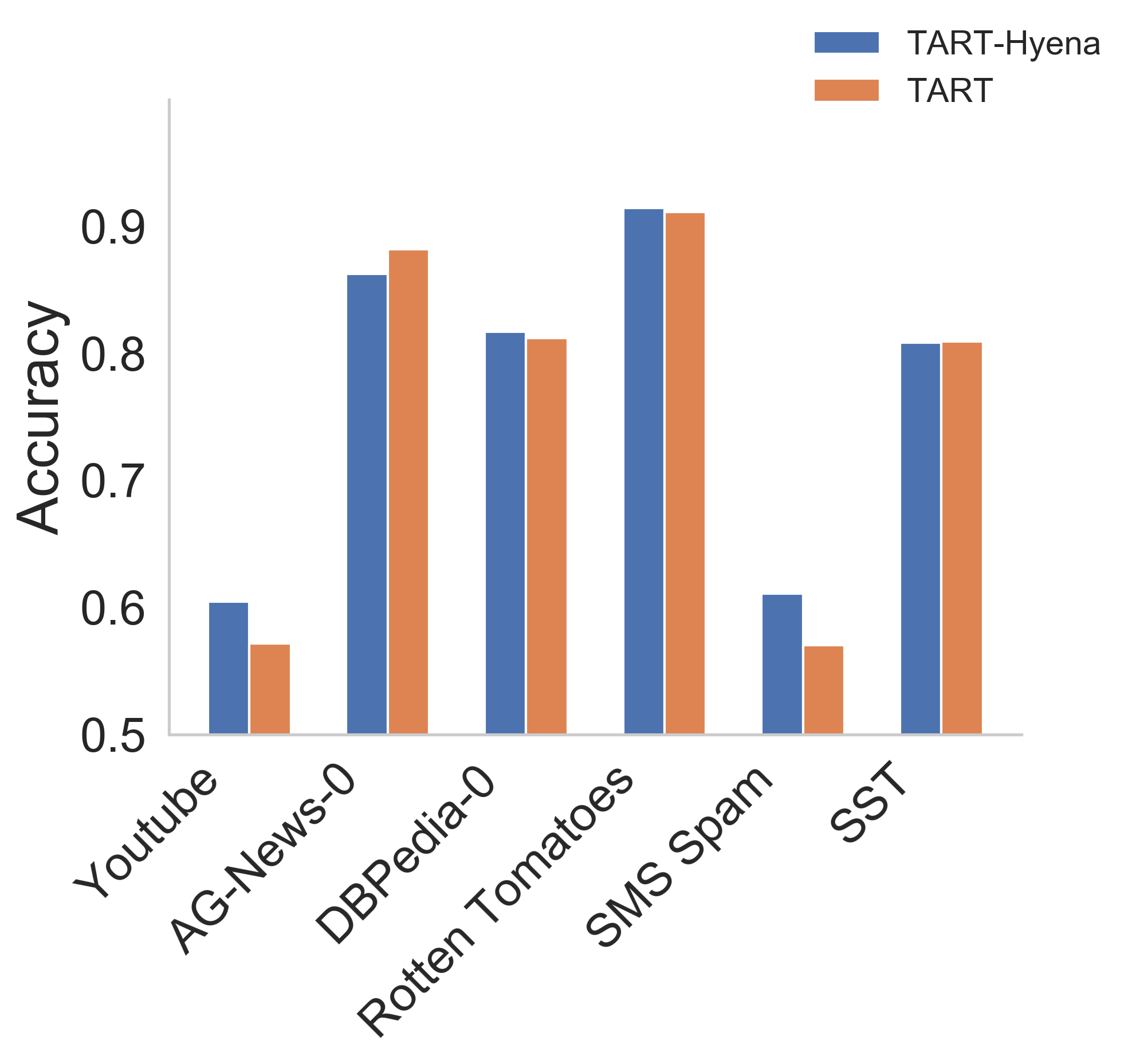}
    \subcaption{k = 48}
    \end{subfigure}
    \begin{subfigure}[b]{0.45\textwidth}
        \centering
\includegraphics[width=1\textwidth]{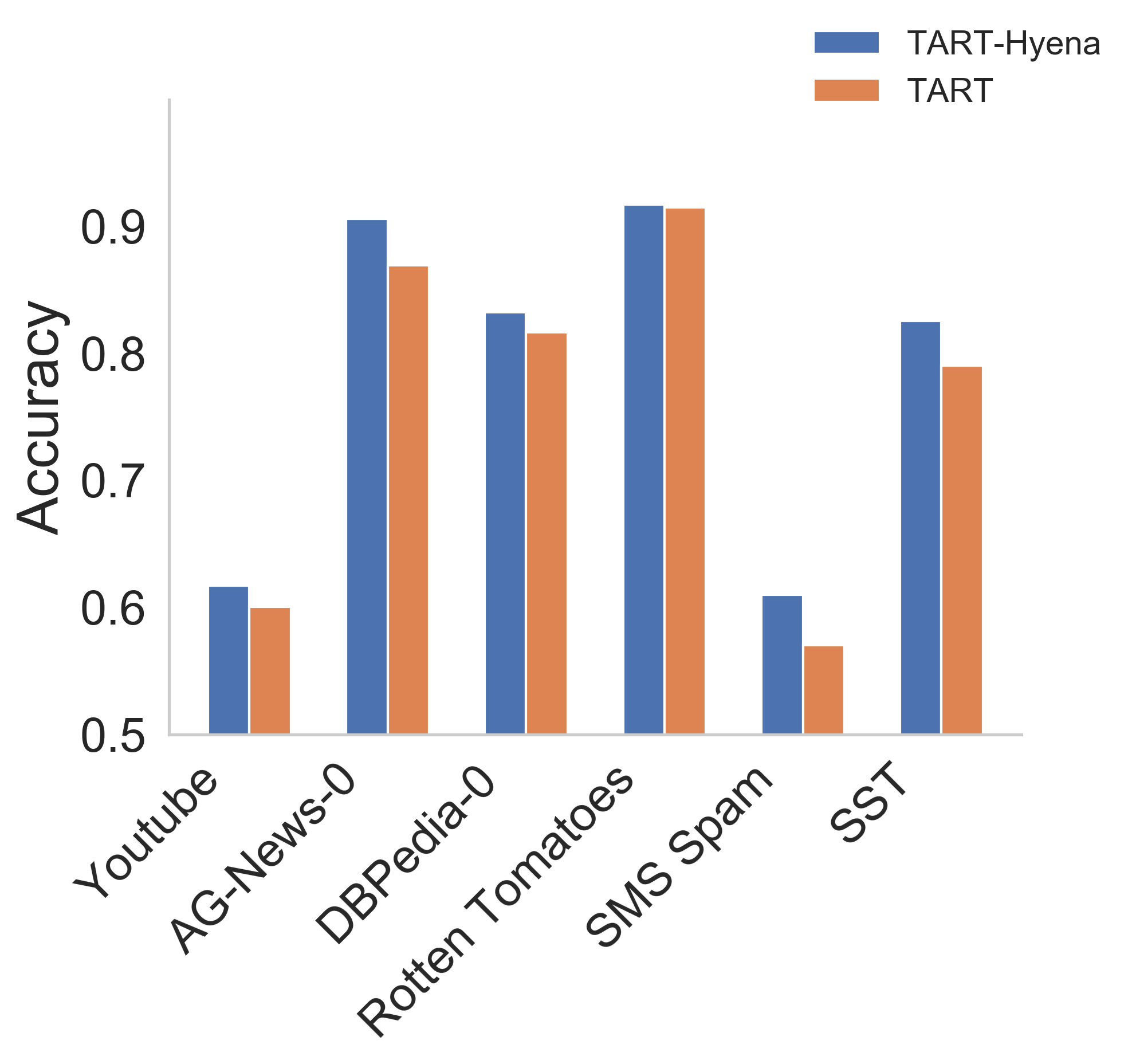}
    \subcaption{k = 64}
    \end{subfigure}
    \caption{\textbf{Comparing TART with a Hyena version of TART.} Instead of using a Transformer module, we use a Hyena model to learn the probabilistic reasoning task. Hyena-TART is comparable to TART across datasets and achieves better performance for larger value of $k$ ($=64$). }
    \label{fig:hyena}
\end{figure}

In this section, we demonstrate that it is possible to train \alg reasoning modules on architectures beyond transformers. More concretely, we train a reasoning module that has the Hyena architecture~\citep{poli2023hyena} using the same synthetic logistic regression tasks. 

\subsubsection{Setup and Hyper-parameters} 
We instantiate a reasoning module with 12 Hyena blocks, a hidden dimension size of 256, and a sequence length of 2050. We train with a batch size of 16, using a learning rate of 5e-05. We sample data with a noise parameter ($\weight$) equal to 1. We train the model for 5000 steps. For our evaluations, we use the final checkpoint (i.e., 5000) of the reasoning module.

\subsubsection{Evaluation} 
We evaluate the Hyena-based reasoning module applied to \gptsmall on 6 datasets: SMS Spam, SST, AG-News-0, DBpedia-14-0, Youtube, and Rotten Tomatoes. As seen in Figure~\ref{fig:hyena}, the Hyena-based reasoning module transfers to natural language tasks and performs competitively with the transformer-based reasoning module

\newpage

\begin{table}
\centering
\begin{tabular}{lccccc}
\toprule
Dataset & \alg & In-context learning & Fine-tuning full & Fine-tuning layer & Adapters\\
\midrule
AG-News-0 & 0.790 $\pm$ 0.036 & 0.552 $\pm$ 0.049 & 0.724 $\pm$ 0.070 & 0.852 $\pm$ 0.019 & 0.824 $\pm$ 0.031 \\
AG-News-1 & 0.828 $\pm$ 0.021 & 0.541 $\pm$ 0.037 & 0.779 $\pm$ 0.137 & 0.903 $\pm$ 0.021 & 0.868 $\pm$ 0.020 \\
AG-News-2 & 0.751 $\pm$ 0.023 & 0.513 $\pm$ 0.014 & 0.626 $\pm$ 0.057 & 0.765 $\pm$ 0.025 & 0.755 $\pm$ 0.012 \\
AG-News-3 & 0.743 $\pm$ 0.031 & 0.502 $\pm$ 0.017 & 0.736 $\pm$ 0.035 & 0.786 $\pm$ 0.025 & 0.727 $\pm$ 0.066 \\
Civil Comments & 0.559 $\pm$ 0.027 & 0.499 $\pm$ 0.002 & 0.520 $\pm$ 0.033 & 0.515 $\pm$ 0.019 & 0.555 $\pm$ 0.036 \\
DBPedia-0 & 0.866 $\pm$ 0.030 & 0.611 $\pm$ 0.091 & 0.802 $\pm$ 0.020 & 0.825 $\pm$ 0.012 & 0.837 $\pm$ 0.022 \\
DBPedia-1 & 0.778 $\pm$ 0.036 & 0.579 $\pm$ 0.100 & 0.766 $\pm$ 0.056 & 0.740 $\pm$ 0.041 & 0.778 $\pm$ 0.044 \\
DBPedia-2 & 0.798 $\pm$ 0.042 & 0.609 $\pm$ 0.136 & 0.862 $\pm$ 0.041 & 0.908 $\pm$ 0.011 & 0.832 $\pm$ 0.048 \\
DBPedia-3 & 0.812 $\pm$ 0.032 & 0.611 $\pm$ 0.135 & 0.817 $\pm$ 0.034 & 0.859 $\pm$ 0.025 & 0.848 $\pm$ 0.028 \\
IMDB & 0.537 $\pm$ 0.022 & 0.507 $\pm$ 0.007 & 0.625 $\pm$ 0.013 & 0.560 $\pm$ 0.021 & 0.556 $\pm$ 0.014 \\
Rotten Tomatoes & 0.535 $\pm$ 0.030 & 0.550 $\pm$ 0.043 & 0.689 $\pm$ 0.019 & 0.541 $\pm$ 0.037 & 0.524 $\pm$ 0.018 \\
SMS Spam & 0.869 $\pm$ 0.063 & 0.736 $\pm$ 0.099 & 0.925 $\pm$ 0.011 & 0.833 $\pm$ 0.015 & 0.886 $\pm$ 0.023 \\
SST & 0.544 $\pm$ 0.021 & 0.542 $\pm$ 0.024 & 0.715 $\pm$ 0.009 & 0.555 $\pm$ 0.026 & 0.547 $\pm$ 0.009 \\
Youtube & 0.784 $\pm$ 0.047 & 0.658 $\pm$ 0.089 & 0.833 $\pm$ 0.089 & 0.768 $\pm$ 0.100 & 0.715 $\pm$ 0.136 \\
\bottomrule
\end{tabular}

\caption{\textbf{Standard deviation of accuracy, number of examples = 18, \gptsmall }}
\label{tab:gpt-std-macro-18}
\end{table}

\begin{table}
\centering
\begin{tabular}{lccccc}
\toprule
Dataset & \alg & In-context learning & Fine-tuning full & Fine-tuning layer & Adapters\\
\midrule
AG-News-0 & 0.805 $\pm$ 0.029 & 0.498 $\pm$ 0.002 & 0.758 $\pm$ 0.127 & 0.601 $\pm$ 0.072 & 0.669 $\pm$ 0.044 \\
AG-News-1 & 0.825 $\pm$ 0.023 & 0.517 $\pm$ 0.019 & 0.916 $\pm$ 0.015 & 0.702 $\pm$ 0.097 & 0.690 $\pm$ 0.034 \\
AG-News-2 & 0.758 $\pm$ 0.027 & 0.500 $\pm$ 0.000 & 0.596 $\pm$ 0.031 & 0.500 $\pm$ 0.001 & 0.588 $\pm$ 0.026 \\
AG-News-3 & 0.754 $\pm$ 0.033 & 0.501 $\pm$ 0.003 & 0.661 $\pm$ 0.093 & 0.552 $\pm$ 0.041 & 0.613 $\pm$ 0.022 \\
Civil Comments & 0.575 $\pm$ 0.019 & 0.500 $\pm$ 0.003 & 0.525 $\pm$ 0.049 & 0.500 $\pm$ 0.000 & 0.515 $\pm$ 0.008 \\
DBPedia-0 & 0.861 $\pm$ 0.018 & 0.508 $\pm$ 0.016 & 0.786 $\pm$ 0.054 & 0.638 $\pm$ 0.109 & 0.704 $\pm$ 0.038 \\
DBPedia-1 & 0.813 $\pm$ 0.020 & 0.499 $\pm$ 0.010 & 0.787 $\pm$ 0.067 & 0.599 $\pm$ 0.078 & 0.710 $\pm$ 0.059 \\
DBPedia-2 & 0.870 $\pm$ 0.035 & 0.502 $\pm$ 0.025 & 0.878 $\pm$ 0.071 & 0.701 $\pm$ 0.113 & 0.767 $\pm$ 0.052 \\
DBPedia-3 & 0.850 $\pm$ 0.050 & 0.502 $\pm$ 0.003 & 0.864 $\pm$ 0.034 & 0.603 $\pm$ 0.114 & 0.734 $\pm$ 0.030 \\
IMDB & 0.550 $\pm$ 0.027 & 0.507 $\pm$ 0.005 & 0.590 $\pm$ 0.046 & 0.500 $\pm$ 0.000 & 0.526 $\pm$ 0.012 \\
Rotten Tomatoes & 0.544 $\pm$ 0.027 & 0.491 $\pm$ 0.013 & 0.589 $\pm$ 0.074 & 0.500 $\pm$ 0.000 & 0.522 $\pm$ 0.023 \\
SMS Spam & 0.901 $\pm$ 0.038 & 0.867 $\pm$ 0.030 & 0.892 $\pm$ 0.052 & 0.867 $\pm$ 0.004 & 0.851 $\pm$ 0.042 \\
SST & 0.572 $\pm$ 0.016 & 0.517 $\pm$ 0.002 & 0.617 $\pm$ 0.074 & 0.517 $\pm$ 0.000 & 0.539 $\pm$ 0.007 \\
Youtube & 0.847 $\pm$ 0.047 & 0.611 $\pm$ 0.084 & 0.810 $\pm$ 0.060 & 0.528 $\pm$ 0.000 & 0.598 $\pm$ 0.103 \\
\bottomrule
\end{tabular}

\caption{\textbf{Standard deviation of accuracy, number of examples = 18, \pythiasmall }}
\label{tab:pythia-std-macro-18}
\end{table}

\begin{table}
\centering
\begin{tabular}{lccccc}
\toprule
Dataset & \alg & In-context learning & Fine-tuning full & Fine-tuning layer & Adapters\\
\midrule
AG-News-0 & 0.813 $\pm$ 0.021 & 0.511 $\pm$ 0.013 & 0.702 $\pm$ 0.102 & 0.790 $\pm$ 0.033 & 0.689 $\pm$ 0.060 \\
AG-News-1 & 0.851 $\pm$ 0.020 & 0.511 $\pm$ 0.014 & 0.801 $\pm$ 0.086 & 0.891 $\pm$ 0.037 & 0.752 $\pm$ 0.032 \\
AG-News-2 & 0.744 $\pm$ 0.034 & 0.509 $\pm$ 0.009 & 0.622 $\pm$ 0.083 & 0.711 $\pm$ 0.070 & 0.652 $\pm$ 0.047 \\
AG-News-3 & 0.763 $\pm$ 0.026 & 0.508 $\pm$ 0.014 & 0.626 $\pm$ 0.016 & 0.775 $\pm$ 0.035 & 0.706 $\pm$ 0.027 \\
Civil Comments & 0.561 $\pm$ 0.029 & 0.489 $\pm$ 0.009 & 0.491 $\pm$ 0.032 & 0.540 $\pm$ 0.029 & 0.533 $\pm$ 0.030 \\
DBPedia-0 & 0.851 $\pm$ 0.009 & 0.531 $\pm$ 0.044 & 0.811 $\pm$ 0.116 & 0.813 $\pm$ 0.057 & 0.812 $\pm$ 0.043 \\
DBPedia-1 & 0.760 $\pm$ 0.037 & 0.546 $\pm$ 0.088 & 0.750 $\pm$ 0.129 & 0.718 $\pm$ 0.067 & 0.754 $\pm$ 0.041 \\
DBPedia-2 & 0.800 $\pm$ 0.032 & 0.567 $\pm$ 0.110 & 0.850 $\pm$ 0.086 & 0.904 $\pm$ 0.018 & 0.851 $\pm$ 0.055 \\
DBPedia-3 & 0.848 $\pm$ 0.025 & 0.528 $\pm$ 0.046 & 0.739 $\pm$ 0.133 & 0.785 $\pm$ 0.132 & 0.849 $\pm$ 0.011 \\
IMDB & 0.552 $\pm$ 0.031 & 0.550 $\pm$ 0.044 & 0.630 $\pm$ 0.016 & 0.608 $\pm$ 0.014 & 0.526 $\pm$ 0.025 \\
Rotten Tomatoes & 0.574 $\pm$ 0.029 & 0.539 $\pm$ 0.031 & 0.638 $\pm$ 0.037 & 0.618 $\pm$ 0.049 & 0.507 $\pm$ 0.010 \\
SMS Spam & 0.830 $\pm$ 0.112 & 0.613 $\pm$ 0.261 & 0.883 $\pm$ 0.130 & 0.911 $\pm$ 0.035 & 0.885 $\pm$ 0.045 \\
SST & 0.574 $\pm$ 0.019 & 0.561 $\pm$ 0.062 & 0.707 $\pm$ 0.024 & 0.636 $\pm$ 0.025 & 0.531 $\pm$ 0.015 \\
Youtube & 0.762 $\pm$ 0.116 & 0.584 $\pm$ 0.144 & 0.753 $\pm$ 0.140 & 0.726 $\pm$ 0.052 & 0.769 $\pm$ 0.084 \\
\bottomrule
\end{tabular}

\caption{\textbf{Standard deviation of accuracy, number of examples = 18, \bloomsmall }}
\label{tab:bloom-std-macro-18}
\end{table}

\begin{table}
\centering
\begin{tabular}{lccccc}
\toprule
Dataset & \alg & In-context learning & Fine-tuning full & Fine-tuning layer & Adapters\\
\midrule
AG-News-0 & 0.808 $\pm$ 0.030 & 0.551 $\pm$ 0.041 & 0.795 $\pm$ 0.049 & 0.874 $\pm$ 0.022 & 0.830 $\pm$ 0.034 \\
AG-News-1 & 0.883 $\pm$ 0.014 & 0.577 $\pm$ 0.055 & 0.852 $\pm$ 0.070 & 0.911 $\pm$ 0.021 & 0.902 $\pm$ 0.011 \\
AG-News-2 & 0.764 $\pm$ 0.019 & 0.540 $\pm$ 0.027 & 0.705 $\pm$ 0.063 & 0.812 $\pm$ 0.019 & 0.782 $\pm$ 0.019 \\
AG-News-3 & 0.798 $\pm$ 0.025 & 0.521 $\pm$ 0.039 & 0.762 $\pm$ 0.046 & 0.813 $\pm$ 0.012 & 0.806 $\pm$ 0.028 \\
Civil Comments & 0.575 $\pm$ 0.054 & 0.498 $\pm$ 0.002 & 0.554 $\pm$ 0.029 & 0.543 $\pm$ 0.022 & 0.579 $\pm$ 0.044 \\
DBPedia-0 & 0.886 $\pm$ 0.021 & 0.632 $\pm$ 0.096 & 0.884 $\pm$ 0.019 & 0.866 $\pm$ 0.013 & 0.859 $\pm$ 0.020 \\
DBPedia-1 & 0.809 $\pm$ 0.031 & 0.593 $\pm$ 0.062 & 0.802 $\pm$ 0.030 & 0.783 $\pm$ 0.022 & 0.813 $\pm$ 0.023 \\
DBPedia-2 & 0.886 $\pm$ 0.011 & 0.586 $\pm$ 0.078 & 0.916 $\pm$ 0.029 & 0.932 $\pm$ 0.013 & 0.899 $\pm$ 0.017 \\
DBPedia-3 & 0.868 $\pm$ 0.022 & 0.565 $\pm$ 0.041 & 0.902 $\pm$ 0.016 & 0.908 $\pm$ 0.012 & 0.868 $\pm$ 0.023 \\
IMDB & 0.537 $\pm$ 0.040 & 0.543 $\pm$ 0.029 & 0.609 $\pm$ 0.029 & 0.556 $\pm$ 0.024 & 0.551 $\pm$ 0.036 \\
Rotten Tomatoes & 0.549 $\pm$ 0.035 & 0.554 $\pm$ 0.036 & 0.667 $\pm$ 0.036 & 0.550 $\pm$ 0.045 & 0.528 $\pm$ 0.020 \\
SMS Spam & 0.903 $\pm$ 0.027 & 0.768 $\pm$ 0.105 & 0.931 $\pm$ 0.014 & 0.861 $\pm$ 0.027 & 0.920 $\pm$ 0.011 \\
SST & 0.573 $\pm$ 0.009 & 0.564 $\pm$ 0.045 & 0.711 $\pm$ 0.022 & 0.594 $\pm$ 0.022 & 0.551 $\pm$ 0.022 \\
Youtube & 0.810 $\pm$ 0.072 & 0.759 $\pm$ 0.074 & 0.854 $\pm$ 0.044 & 0.773 $\pm$ 0.052 & 0.722 $\pm$ 0.113 \\
\bottomrule
\end{tabular}
\vspace{2mm}
\caption{\textbf{Standard deviation of accuracy, number of examples = 32, \gptsmall}}
\label{tab:gpt-std-macro-32}
\end{table}

\begin{table}
\centering
\begin{tabular}{lccccc}
\toprule
Dataset & \alg & In-context learning & Fine-tuning full & Fine-tuning layer & Adapters\\
\midrule
AG-News-0 & 0.835 $\pm$ 0.020 & 0.499 $\pm$ 0.002 & 0.791 $\pm$ 0.100 & 0.855 $\pm$ 0.017 & 0.712 $\pm$ 0.028 \\
AG-News-1 & 0.900 $\pm$ 0.012 & 0.504 $\pm$ 0.004 & 0.911 $\pm$ 0.022 & 0.926 $\pm$ 0.009 & 0.738 $\pm$ 0.028 \\
AG-News-2 & 0.773 $\pm$ 0.012 & 0.510 $\pm$ 0.006 & 0.687 $\pm$ 0.034 & 0.742 $\pm$ 0.044 & 0.683 $\pm$ 0.025 \\
AG-News-3 & 0.823 $\pm$ 0.024 & 0.514 $\pm$ 0.017 & 0.792 $\pm$ 0.022 & 0.833 $\pm$ 0.011 & 0.697 $\pm$ 0.023 \\
Civil Comments & 0.596 $\pm$ 0.024 & 0.499 $\pm$ 0.007 & 0.588 $\pm$ 0.052 & 0.536 $\pm$ 0.020 & 0.542 $\pm$ 0.016 \\
DBPedia-0 & 0.890 $\pm$ 0.016 & 0.514 $\pm$ 0.026 & 0.880 $\pm$ 0.049 & 0.777 $\pm$ 0.055 & 0.782 $\pm$ 0.044 \\
DBPedia-1 & 0.833 $\pm$ 0.018 & 0.522 $\pm$ 0.023 & 0.834 $\pm$ 0.061 & 0.768 $\pm$ 0.024 & 0.760 $\pm$ 0.022 \\
DBPedia-2 & 0.912 $\pm$ 0.010 & 0.522 $\pm$ 0.031 & 0.916 $\pm$ 0.033 & 0.903 $\pm$ 0.012 & 0.859 $\pm$ 0.019 \\
DBPedia-3 & 0.879 $\pm$ 0.021 & 0.520 $\pm$ 0.026 & 0.900 $\pm$ 0.019 & 0.863 $\pm$ 0.018 & 0.812 $\pm$ 0.022 \\
IMDB & 0.541 $\pm$ 0.041 & 0.510 $\pm$ 0.008 & 0.630 $\pm$ 0.017 & 0.575 $\pm$ 0.024 & 0.540 $\pm$ 0.018 \\
Rotten Tomatoes & 0.576 $\pm$ 0.042 & 0.495 $\pm$ 0.007 & 0.659 $\pm$ 0.064 & 0.568 $\pm$ 0.027 & 0.542 $\pm$ 0.019 \\
SMS Spam & 0.937 $\pm$ 0.017 & 0.856 $\pm$ 0.080 & 0.913 $\pm$ 0.044 & 0.953 $\pm$ 0.014 & 0.877 $\pm$ 0.036 \\
SST & 0.602 $\pm$ 0.015 & 0.509 $\pm$ 0.013 & 0.705 $\pm$ 0.021 & 0.605 $\pm$ 0.043 & 0.542 $\pm$ 0.003 \\
Youtube & 0.862 $\pm$ 0.045 & 0.626 $\pm$ 0.081 & 0.874 $\pm$ 0.046 & 0.762 $\pm$ 0.052 & 0.654 $\pm$ 0.050 \\
\bottomrule
\end{tabular}
\vspace{2mm}
\caption{\textbf{Standard deviation of accuracy, number of examples = 32, \pythiasmall }}
\label{tab:pythia-std-macro-32}
\end{table}

\begin{table}
\centering
\begin{tabular}{lccccc}
\toprule
Dataset & \alg & In-context learning & Fine-tuning full & Fine-tuning layer & Adapters\\
\midrule
AG-News-0 & 0.825 $\pm$ 0.031 & 0.503 $\pm$ 0.003 & 0.802 $\pm$ 0.041 & 0.818 $\pm$ 0.014 & 0.739 $\pm$ 0.049 \\
AG-News-1 & 0.880 $\pm$ 0.012 & 0.516 $\pm$ 0.026 & 0.898 $\pm$ 0.043 & 0.912 $\pm$ 0.034 & 0.829 $\pm$ 0.010 \\
AG-News-2 & 0.771 $\pm$ 0.009 & 0.519 $\pm$ 0.025 & 0.736 $\pm$ 0.071 & 0.773 $\pm$ 0.039 & 0.684 $\pm$ 0.040 \\
AG-News-3 & 0.810 $\pm$ 0.016 & 0.509 $\pm$ 0.013 & 0.753 $\pm$ 0.058 & 0.805 $\pm$ 0.031 & 0.742 $\pm$ 0.030 \\
Civil Comments & 0.587 $\pm$ 0.025 & 0.500 $\pm$ 0.011 & 0.549 $\pm$ 0.016 & 0.564 $\pm$ 0.062 & 0.555 $\pm$ 0.009 \\
DBPedia-0 & 0.857 $\pm$ 0.035 & 0.592 $\pm$ 0.083 & 0.801 $\pm$ 0.125 & 0.822 $\pm$ 0.043 & 0.834 $\pm$ 0.038 \\
DBPedia-1 & 0.802 $\pm$ 0.031 & 0.558 $\pm$ 0.047 & 0.870 $\pm$ 0.037 & 0.800 $\pm$ 0.041 & 0.813 $\pm$ 0.034 \\
DBPedia-2 & 0.879 $\pm$ 0.018 & 0.609 $\pm$ 0.025 & 0.938 $\pm$ 0.018 & 0.920 $\pm$ 0.031 & 0.903 $\pm$ 0.021 \\
DBPedia-3 & 0.866 $\pm$ 0.035 & 0.634 $\pm$ 0.110 & 0.812 $\pm$ 0.163 & 0.911 $\pm$ 0.004 & 0.876 $\pm$ 0.037 \\
IMDB & 0.553 $\pm$ 0.046 & 0.541 $\pm$ 0.024 & 0.636 $\pm$ 0.016 & 0.600 $\pm$ 0.018 & 0.536 $\pm$ 0.024 \\
Rotten Tomatoes & 0.589 $\pm$ 0.037 & 0.575 $\pm$ 0.039 & 0.713 $\pm$ 0.029 & 0.630 $\pm$ 0.026 & 0.527 $\pm$ 0.015 \\
SMS Spam & 0.933 $\pm$ 0.017 & 0.762 $\pm$ 0.033 & 0.956 $\pm$ 0.026 & 0.905 $\pm$ 0.056 & 0.917 $\pm$ 0.017 \\
SST & 0.579 $\pm$ 0.032 & 0.502 $\pm$ 0.020 & 0.739 $\pm$ 0.013 & 0.638 $\pm$ 0.035 & 0.562 $\pm$ 0.017 \\
Youtube & 0.799 $\pm$ 0.117 & 0.710 $\pm$ 0.182 & 0.887 $\pm$ 0.033 & 0.756 $\pm$ 0.115 & 0.850 $\pm$ 0.024 \\
\bottomrule
\end{tabular}
\vspace{2mm}
\caption{\textbf{Standard deviation of accuracy, number of examples = 32, \bloomsmall }}
\label{tab:bloom-std-macro-32}
\end{table}

\begin{table}
\centering
\begin{tabular}{lccccc}
\toprule
Dataset & \alg & In-context learning & Fine-tuning full & Fine-tuning layer & Adapters\\
\midrule
AG-News-0 & 0.812 $\pm$ 0.021 & 0.560 $\pm$ 0.046 & 0.832 $\pm$ 0.023 & 0.876 $\pm$ 0.012 & 0.833 $\pm$ 0.016 \\
AG-News-1 & 0.904 $\pm$ 0.009 & 0.623 $\pm$ 0.051 & 0.924 $\pm$ 0.026 & 0.932 $\pm$ 0.007 & 0.923 $\pm$ 0.009 \\
AG-News-2 & 0.786 $\pm$ 0.008 & 0.566 $\pm$ 0.023 & 0.791 $\pm$ 0.016 & 0.824 $\pm$ 0.010 & 0.797 $\pm$ 0.010 \\
AG-News-3 & 0.822 $\pm$ 0.030 & 0.533 $\pm$ 0.029 & 0.800 $\pm$ 0.042 & 0.840 $\pm$ 0.015 & 0.825 $\pm$ 0.034 \\
Civil Comments & 0.591 $\pm$ 0.029 & 0.497 $\pm$ 0.003 & 0.568 $\pm$ 0.038 & 0.554 $\pm$ 0.029 & 0.614 $\pm$ 0.028 \\
DBPedia-0 & 0.911 $\pm$ 0.011 & 0.676 $\pm$ 0.097 & 0.892 $\pm$ 0.017 & 0.898 $\pm$ 0.012 & 0.894 $\pm$ 0.023 \\
DBPedia-1 & 0.823 $\pm$ 0.025 & 0.644 $\pm$ 0.115 & 0.838 $\pm$ 0.058 & 0.819 $\pm$ 0.018 & 0.833 $\pm$ 0.026 \\
DBPedia-2 & 0.902 $\pm$ 0.015 & 0.622 $\pm$ 0.114 & 0.931 $\pm$ 0.032 & 0.940 $\pm$ 0.005 & 0.913 $\pm$ 0.007 \\
DBPedia-3 & 0.883 $\pm$ 0.023 & 0.620 $\pm$ 0.129 & 0.891 $\pm$ 0.015 & 0.911 $\pm$ 0.006 & 0.889 $\pm$ 0.021 \\
IMDB & 0.557 $\pm$ 0.033 & 0.522 $\pm$ 0.025 & 0.641 $\pm$ 0.008 & 0.564 $\pm$ 0.017 & 0.577 $\pm$ 0.020 \\
Rotten Tomatoes & 0.572 $\pm$ 0.014 & 0.548 $\pm$ 0.045 & 0.710 $\pm$ 0.010 & 0.575 $\pm$ 0.048 & 0.572 $\pm$ 0.029 \\
SMS Spam & 0.882 $\pm$ 0.044 & 0.860 $\pm$ 0.036 & 0.946 $\pm$ 0.019 & 0.881 $\pm$ 0.013 & 0.925 $\pm$ 0.010 \\
SST & 0.570 $\pm$ 0.019 & 0.563 $\pm$ 0.032 & 0.705 $\pm$ 0.019 & 0.575 $\pm$ 0.028 & 0.542 $\pm$ 0.035 \\
Youtube & 0.810 $\pm$ 0.039 & 0.792 $\pm$ 0.081 & 0.923 $\pm$ 0.014 & 0.874 $\pm$ 0.029 & 0.832 $\pm$ 0.053\\
\bottomrule
\end{tabular}
\vspace{2mm}
\caption{\textbf{Standard deviation of accuracy, number of examples = 48, \gptsmall }}
\label{tab:gpt-std-macro-48}
\end{table}

\begin{table}
\centering

\begin{tabular}{lccccc}
\toprule
Dataset & \alg & In-context learning & Fine-tuning full & Fine-tuning layer & Adapters\\
\midrule
AG-News-0 & 0.846 $\pm$ 0.018 & 0.496 $\pm$ 0.007 & 0.871 $\pm$ 0.015 & 0.871 $\pm$ 0.005 & 0.732 $\pm$ 0.036 \\
AG-News-1 & 0.923 $\pm$ 0.005 & 0.528 $\pm$ 0.044 & 0.942 $\pm$ 0.004 & 0.935 $\pm$ 0.009 & 0.780 $\pm$ 0.040 \\
AG-News-2 & 0.798 $\pm$ 0.009 & 0.509 $\pm$ 0.004 & 0.731 $\pm$ 0.040 & 0.782 $\pm$ 0.018 & 0.722 $\pm$ 0.019 \\
AG-News-3 & 0.845 $\pm$ 0.012 & 0.515 $\pm$ 0.014 & 0.787 $\pm$ 0.046 & 0.852 $\pm$ 0.011 & 0.718 $\pm$ 0.026 \\
Civil Comments & 0.605 $\pm$ 0.037 & 0.512 $\pm$ 0.009 & 0.607 $\pm$ 0.039 & 0.556 $\pm$ 0.014 & 0.556 $\pm$ 0.018 \\
DBPedia-0 & 0.905 $\pm$ 0.020 & 0.549 $\pm$ 0.057 & 0.924 $\pm$ 0.020 & 0.830 $\pm$ 0.034 & 0.798 $\pm$ 0.021 \\
DBPedia-1 & 0.840 $\pm$ 0.021 & 0.592 $\pm$ 0.062 & 0.867 $\pm$ 0.024 & 0.808 $\pm$ 0.031 & 0.780 $\pm$ 0.011 \\
DBPedia-2 & 0.916 $\pm$ 0.009 & 0.636 $\pm$ 0.081 & 0.937 $\pm$ 0.017 & 0.923 $\pm$ 0.013 & 0.870 $\pm$ 0.022 \\
DBPedia-3 & 0.888 $\pm$ 0.022 & 0.618 $\pm$ 0.096 & 0.927 $\pm$ 0.011 & 0.885 $\pm$ 0.007 & 0.827 $\pm$ 0.026 \\
IMDB & 0.557 $\pm$ 0.034 & 0.503 $\pm$ 0.008 & 0.632 $\pm$ 0.013 & 0.566 $\pm$ 0.030 & 0.545 $\pm$ 0.008 \\
Rotten Tomatoes & 0.601 $\pm$ 0.026 & 0.480 $\pm$ 0.014 & 0.683 $\pm$ 0.033 & 0.556 $\pm$ 0.033 & 0.562 $\pm$ 0.018 \\
SMS Spam & 0.956 $\pm$ 0.008 & 0.869 $\pm$ 0.044 & 0.966 $\pm$ 0.011 & 0.960 $\pm$ 0.008 & 0.891 $\pm$ 0.035 \\
SST & 0.612 $\pm$ 0.024 & 0.500 $\pm$ 0.025 & 0.702 $\pm$ 0.046 & 0.594 $\pm$ 0.042 & 0.560 $\pm$ 0.023 \\
Youtube & 0.872 $\pm$ 0.019 & 0.654 $\pm$ 0.063 & 0.904 $\pm$ 0.033 & 0.826 $\pm$ 0.046 & 0.666 $\pm$ 0.027 \\
\bottomrule
\end{tabular}
\vspace{2mm}
\caption{\textbf{Standard deviation of accuracy, number of examples = 48, \pythiasmall }}
\label{tab:pythia-std-macro-48}
\end{table}

\begin{table}
\centering
\begin{tabular}{lccccc}
\toprule
Dataset & \alg & In-context learning & Fine-tuning full & Fine-tuning layer & Adapters\\
\midrule
AG-News-0 & 0.829 $\pm$ 0.029 & 0.507 $\pm$ 0.005 & 0.829 $\pm$ 0.071 & 0.847 $\pm$ 0.022 & 0.788 $\pm$ 0.028 \\
AG-News-1 & 0.909 $\pm$ 0.010 & 0.543 $\pm$ 0.052 & 0.940 $\pm$ 0.011 & 0.935 $\pm$ 0.004 & 0.856 $\pm$ 0.014 \\
AG-News-2 & 0.789 $\pm$ 0.016 & 0.511 $\pm$ 0.012 & 0.715 $\pm$ 0.057 & 0.793 $\pm$ 0.027 & 0.708 $\pm$ 0.011 \\
AG-News-3 & 0.832 $\pm$ 0.018 & 0.530 $\pm$ 0.038 & 0.834 $\pm$ 0.029 & 0.843 $\pm$ 0.014 & 0.763 $\pm$ 0.015 \\
Civil Comments & 0.602 $\pm$ 0.030 & 0.506 $\pm$ 0.011 & 0.605 $\pm$ 0.044 & 0.592 $\pm$ 0.039 & 0.566 $\pm$ 0.008 \\
DBPedia-0 & 0.889 $\pm$ 0.022 & 0.703 $\pm$ 0.085 & 0.933 $\pm$ 0.021 & 0.888 $\pm$ 0.024 & 0.873 $\pm$ 0.027 \\
DBPedia-1 & 0.817 $\pm$ 0.023 & 0.734 $\pm$ 0.076 & 0.891 $\pm$ 0.028 & 0.836 $\pm$ 0.036 & 0.834 $\pm$ 0.024 \\
DBPedia-2 & 0.900 $\pm$ 0.011 & 0.847 $\pm$ 0.037 & 0.949 $\pm$ 0.013 & 0.940 $\pm$ 0.009 & 0.914 $\pm$ 0.013 \\
DBPedia-3 & 0.884 $\pm$ 0.020 & 0.815 $\pm$ 0.105 & 0.928 $\pm$ 0.024 & 0.917 $\pm$ 0.015 & 0.891 $\pm$ 0.037 \\
IMDB & 0.565 $\pm$ 0.033 & 0.545 $\pm$ 0.033 & 0.641 $\pm$ 0.020 & 0.604 $\pm$ 0.022 & 0.542 $\pm$ 0.015 \\
Rotten Tomatoes & 0.605 $\pm$ 0.015 & 0.505 $\pm$ 0.005 & 0.704 $\pm$ 0.025 & 0.672 $\pm$ 0.042 & 0.531 $\pm$ 0.018 \\
SMS Spam & 0.933 $\pm$ 0.009 & 0.636 $\pm$ 0.130 & 0.929 $\pm$ 0.058 & 0.919 $\pm$ 0.032 & 0.914 $\pm$ 0.027 \\
SST & 0.610 $\pm$ 0.030 & 0.489 $\pm$ 0.004 & 0.695 $\pm$ 0.020 & 0.673 $\pm$ 0.029 & 0.538 $\pm$ 0.011 \\
Youtube & 0.834 $\pm$ 0.049 & 0.797 $\pm$ 0.090 & 0.805 $\pm$ 0.095 & 0.851 $\pm$ 0.020 & 0.865 $\pm$ 0.012 \\
\bottomrule
\end{tabular}
\vspace{2mm}
\caption{\textbf{Standard deviation of accuracy, number of examples = 48, \bloomsmall }}
\label{tab:bloom-std-macro-48}
\end{table}

\begin{table}
\centering
\begin{tabular}{lccccc}
\toprule
Dataset & \alg & In-context learning & Fine-tuning full & Fine-tuning layer & Adapters\\
\midrule
AG-News-0 & 0.817 $\pm$ 0.024 & 0.539 $\pm$ 0.044 & 0.855 $\pm$ 0.030 & 0.877 $\pm$ 0.013 & 0.830 $\pm$ 0.023 \\
AG-News-1 & 0.904 $\pm$ 0.007 & 0.561 $\pm$ 0.048 & 0.923 $\pm$ 0.027 & 0.939 $\pm$ 0.003 & 0.922 $\pm$ 0.010 \\
AG-News-2 & 0.790 $\pm$ 0.026 & 0.576 $\pm$ 0.022 & 0.814 $\pm$ 0.009 & 0.839 $\pm$ 0.009 & 0.816 $\pm$ 0.014 \\
AG-News-3 & 0.833 $\pm$ 0.015 & 0.550 $\pm$ 0.035 & 0.803 $\pm$ 0.017 & 0.852 $\pm$ 0.013 & 0.842 $\pm$ 0.018 \\
Civil Comments & 0.581 $\pm$ 0.039 & 0.499 $\pm$ 0.002 & 0.587 $\pm$ 0.018 & 0.576 $\pm$ 0.039 & 0.605 $\pm$ 0.036 \\
DBPedia-0 & 0.915 $\pm$ 0.005 & 0.652 $\pm$ 0.095 & 0.917 $\pm$ 0.011 & 0.908 $\pm$ 0.019 & 0.909 $\pm$ 0.025 \\
DBPedia-1 & 0.825 $\pm$ 0.037 & 0.633 $\pm$ 0.104 & 0.838 $\pm$ 0.032 & 0.829 $\pm$ 0.012 & 0.852 $\pm$ 0.011 \\
DBPedia-2 & 0.887 $\pm$ 0.020 & 0.606 $\pm$ 0.088 & 0.950 $\pm$ 0.007 & 0.952 $\pm$ 0.011 & 0.916 $\pm$ 0.014 \\
DBPedia-3 & 0.873 $\pm$ 0.033 & 0.611 $\pm$ 0.136 & 0.898 $\pm$ 0.034 & 0.927 $\pm$ 0.007 & 0.909 $\pm$ 0.008 \\
IMDB & 0.558 $\pm$ 0.029 & 0.515 $\pm$ 0.024 & 0.646 $\pm$ 0.010 & 0.563 $\pm$ 0.021 & 0.566 $\pm$ 0.035 \\
Rotten Tomatoes & 0.601 $\pm$ 0.020 & 0.533 $\pm$ 0.053 & 0.708 $\pm$ 0.015 & 0.579 $\pm$ 0.049 & 0.556 $\pm$ 0.033 \\
SMS Spam & 0.869 $\pm$ 0.023 & 0.825 $\pm$ 0.074 & 0.934 $\pm$ 0.016 & 0.898 $\pm$ 0.008 & 0.927 $\pm$ 0.004 \\
SST & 0.570 $\pm$ 0.038 & 0.554 $\pm$ 0.041 & 0.712 $\pm$ 0.033 & 0.609 $\pm$ 0.021 & 0.567 $\pm$ 0.016 \\
Youtube & 0.790 $\pm$ 0.050 & 0.819 $\pm$ 0.050 & 0.926 $\pm$ 0.014 & 0.891 $\pm$ 0.031 & 0.837 $\pm$ 0.055 \\
\bottomrule
\end{tabular}
\vspace{2mm}
\caption{\textbf{Standard deviation of accuracy, number of examples = 64, \gptsmall }}
\label{tab:gpt-std-macro-64}
\end{table}

\begin{table}
\centering

\begin{tabular}{lccccc}
\toprule
Dataset & \alg & In-context learning & Fine-tuning full & Fine-tuning layer & Adapters\\
\midrule
AG-News-0 & 0.851 $\pm$ 0.024 & 0.502 $\pm$ 0.003 & 0.858 $\pm$ 0.009 & 0.869 $\pm$ 0.010 & 0.764 $\pm$ 0.030 \\
AG-News-1 & 0.925 $\pm$ 0.002 & 0.529 $\pm$ 0.030 & 0.937 $\pm$ 0.011 & 0.938 $\pm$ 0.004 & 0.820 $\pm$ 0.021 \\
AG-News-2 & 0.812 $\pm$ 0.007 & 0.513 $\pm$ 0.013 & 0.752 $\pm$ 0.051 & 0.795 $\pm$ 0.010 & 0.738 $\pm$ 0.008 \\
AG-News-3 & 0.851 $\pm$ 0.007 & 0.503 $\pm$ 0.004 & 0.820 $\pm$ 0.020 & 0.855 $\pm$ 0.009 & 0.741 $\pm$ 0.019 \\
Civil Comments & 0.606 $\pm$ 0.029 & 0.500 $\pm$ 0.001 & 0.659 $\pm$ 0.032 & 0.566 $\pm$ 0.023 & 0.566 $\pm$ 0.018 \\
DBPedia-0 & 0.910 $\pm$ 0.013 & 0.518 $\pm$ 0.020 & 0.912 $\pm$ 0.027 & 0.858 $\pm$ 0.019 & 0.825 $\pm$ 0.015 \\
DBPedia-1 & 0.839 $\pm$ 0.027 & 0.542 $\pm$ 0.028 & 0.897 $\pm$ 0.016 & 0.824 $\pm$ 0.031 & 0.788 $\pm$ 0.018 \\
DBPedia-2 & 0.916 $\pm$ 0.011 & 0.609 $\pm$ 0.106 & 0.953 $\pm$ 0.013 & 0.922 $\pm$ 0.012 & 0.882 $\pm$ 0.019 \\
DBPedia-3 & 0.887 $\pm$ 0.028 & 0.527 $\pm$ 0.022 & 0.940 $\pm$ 0.013 & 0.904 $\pm$ 0.007 & 0.857 $\pm$ 0.020 \\
IMDB & 0.556 $\pm$ 0.024 & 0.506 $\pm$ 0.005 & 0.619 $\pm$ 0.030 & 0.574 $\pm$ 0.017 & 0.552 $\pm$ 0.009 \\
Rotten Tomatoes & 0.624 $\pm$ 0.024 & 0.485 $\pm$ 0.020 & 0.686 $\pm$ 0.040 & 0.577 $\pm$ 0.033 & 0.568 $\pm$ 0.019 \\
SMS Spam & 0.937 $\pm$ 0.018 & 0.905 $\pm$ 0.017 & 0.960 $\pm$ 0.021 & 0.961 $\pm$ 0.006 & 0.899 $\pm$ 0.014 \\
SST & 0.606 $\pm$ 0.022 & 0.508 $\pm$ 0.022 & 0.688 $\pm$ 0.047 & 0.602 $\pm$ 0.036 & 0.567 $\pm$ 0.017 \\
Youtube & 0.888 $\pm$ 0.028 & 0.715 $\pm$ 0.097 & 0.897 $\pm$ 0.046 & 0.878 $\pm$ 0.050 & 0.675 $\pm$ 0.019 \\
\bottomrule
\end{tabular}
\vspace{2mm}
\caption{\textbf{Standard deviation of accuracy, number of examples = 64, \pythiasmall }}
\label{tab:pythia-std-macro-64}
\end{table}

\begin{table}
\centering
\begin{tabular}{lccccc}
\toprule
Dataset & \alg & In-context learning & Fine-tuning full & Fine-tuning layer & Adapters\\
\midrule
AG-News-0 & 0.836 $\pm$ 0.018 & 0.509 $\pm$ 0.008 & 0.850 $\pm$ 0.027 & 0.856 $\pm$ 0.008 & 0.799 $\pm$ 0.029 \\
AG-News-1 & 0.918 $\pm$ 0.007 & 0.543 $\pm$ 0.033 & 0.900 $\pm$ 0.024 & 0.933 $\pm$ 0.008 & 0.875 $\pm$ 0.015 \\
AG-News-2 & 0.799 $\pm$ 0.012 & 0.515 $\pm$ 0.019 & 0.784 $\pm$ 0.018 & 0.831 $\pm$ 0.022 & 0.732 $\pm$ 0.017 \\
AG-News-3 & 0.836 $\pm$ 0.011 & 0.504 $\pm$ 0.003 & 0.811 $\pm$ 0.034 & 0.853 $\pm$ 0.011 & 0.784 $\pm$ 0.022 \\
Civil Comments & 0.602 $\pm$ 0.030 & 0.510 $\pm$ 0.012 & 0.573 $\pm$ 0.035 & 0.611 $\pm$ 0.025 & 0.572 $\pm$ 0.013 \\
DBPedia-0 & 0.905 $\pm$ 0.015 & 0.667 $\pm$ 0.052 & 0.936 $\pm$ 0.018 & 0.902 $\pm$ 0.022 & 0.882 $\pm$ 0.020 \\
DBPedia-1 & 0.809 $\pm$ 0.022 & 0.687 $\pm$ 0.117 & 0.887 $\pm$ 0.039 & 0.852 $\pm$ 0.032 & 0.853 $\pm$ 0.008 \\
DBPedia-2 & 0.881 $\pm$ 0.026 & 0.799 $\pm$ 0.075 & 0.955 $\pm$ 0.022 & 0.947 $\pm$ 0.009 & 0.922 $\pm$ 0.014 \\
DBPedia-3 & 0.877 $\pm$ 0.014 & 0.793 $\pm$ 0.106 & 0.906 $\pm$ 0.027 & 0.921 $\pm$ 0.017 & 0.899 $\pm$ 0.024 \\
IMDB & 0.571 $\pm$ 0.033 & 0.539 $\pm$ 0.020 & 0.621 $\pm$ 0.039 & 0.618 $\pm$ 0.013 & 0.542 $\pm$ 0.012 \\
Rotten Tomatoes & 0.597 $\pm$ 0.025 & 0.536 $\pm$ 0.047 & 0.684 $\pm$ 0.053 & 0.672 $\pm$ 0.041 & 0.543 $\pm$ 0.024 \\
SMS Spam & 0.907 $\pm$ 0.045 & 0.659 $\pm$ 0.133 & 0.942 $\pm$ 0.024 & 0.931 $\pm$ 0.030 & 0.909 $\pm$ 0.033 \\
SST & 0.620 $\pm$ 0.039 & 0.495 $\pm$ 0.015 & 0.672 $\pm$ 0.039 & 0.716 $\pm$ 0.032 & 0.554 $\pm$ 0.020 \\
Youtube & 0.844 $\pm$ 0.059 & 0.765 $\pm$ 0.137 & 0.879 $\pm$ 0.058 & 0.865 $\pm$ 0.033 & 0.883 $\pm$ 0.016 \\
\bottomrule
\end{tabular}
\vspace{2mm}
\caption{\textbf{Standard deviation of accuracy, number of examples = 64, \bloomsmall }}
\label{tab:bloom-std-macro-64}
\end{table}

\end{document}